%% file: main.tex
\documentclass[Afour,sageh,times]{sagej}

\input{custom/packages}

\input{custom/macros}

\runninghead{Liang et al.}

\title{Think Fast and Far: Long-Horizon Online POMDP Planning via Rapid State Sampling}

\author{Yuanchu Liang\affilnum{1}, Edward Kim\affilnum{1},  J. Arden Knoll\affilnum{2}, Wil Thomason\affilnum{2}, Zachary Kingston\affilnum{2}, Lydia E. Kavraki\affilnum{2} and Hanna Kurniawati\affilnum{1} }

\affiliation{\affilnum{1}Australian National University, Canberra ACT 2601, Australia. E-mails: \texttt{\{Yuanchu.Liang, Edward.Kim, hanna.kurniawati\}@anu.edu.au}
\affilnum{2}Rice University, Houston TX 77005, USA. E-mails: \texttt{ak228@rice.edu, wil.thomason@gmail.com, zkingston@purdue.edu, kavraki@rice.edu}}

\begin{abstract}
Partially Observable Markov Decision Processes (\pomdps) are a general and principled framework for motion planning under uncertainty. Despite tremendous improvement in the scalability of \pomdp solvers, long-horizon \pomdps remain difficult to solve. 
To alleviate the difficulty, this paper proposes a new approximate online \pomdp solver, called \nopLong (\nop). 
\nop uses novel extremely fast sampling-based motion planning techniques to sample the state space and generate a diverse set of macro actions online, which are then used to bias belief-space sampling and infer high-quality policies \emph{without} requiring exhaustive enumeration of the action space---a fundamental constraint for modern online \pomdp solvers. 
\nop converges to a near-optimal reference-based solution at a rate that depends on the number of sampled actions, rather than the size of the action space.
\nop is evaluated on various long-horizon \pomdps with up to 3000 lookahead steps and 35-dimensional state spaces, where the state, action and observation spaces can be continuous, discrete, or a hybrid of discrete and continuous. Although the reference-based optimal solution may not be the same as the optimal \pomdp solution, empirical results indicate that in all of these problems,  in terms of success rate, \nop \emph{outperforms} other state-of-the-art methods by up to \emph{multiple folds}. We also demonstrate the capability of our approach on a physical robot demonstration. This work extends the theory and empirical results of our ISRR24 paper. Code can be found at \texttt{https://github.com/RDLLab/ROPRAS3}.

\end{abstract}

\begin{document}
\maketitle
%%%%%%%%%%%%%%%%%%%%%%%%%%%%
\section{Introduction}
%%%%%%%%%%%%%%%%%%%%%%%%%%%%
\input{sections/introduction}

%%%%%%%%%%%%%%%%%%%%%%%%%%%%
\section{Background and Related Work}
%%%%%%%%%%%%%%%%%%%%%%%%%%%%

\subsection{Sampling-Based Motion Planners}

%\textcolor{blue}{TODO: WT or ZK: Could you add a section on \sbmps in general.}
\input{sections/sbmp}

\subsection{POMDP Background}
\input{sections/pomdp}

\subsection{Long-Horizon POMDPs}\label{sec: long horizon POMDPs}
\input{sections/largepomdps}

\subsection{Convergence of Online Sampling-based POMDP Planners}
\input{sections/convergence_background}

\section{Continuous Reference-Based POMDPs over Stochastic Actions}\label{sec.rbpomdp}
\input{sections/refPomdp}

\section{Algorithm}\label{sec.nop}
\input{sections/algorithm-sbmp}
 We introduce \nopLong (\nop), our online anytime planner that uses sampling techniques to rapidly estimate $\refVal$-Values and $\refQVal$-Values for inferring good actions to execute. We motivate its design by briefly outlining \vamp's capability to induce high-quality reference policies. A general algorithm is proposed to utilize \vamp to infer actions for online \pomdp planning. Then we elaborate \nop and analyze its online convergence rate.

\subsection{Vector Accelerated Motion Planning (VAMP)}\label{subsec.vamp}
\input{sections/vamp}

\subsection{POMDP Planning with SBMP-Generated Trajectories}

\input{sections/integration}

\subsection{\nopTitle}
\input{sections/algorithm}

\section{Experiments}\label{sec:experiments}
\input{sections/experiments}

\section{Summary}\label{sec: Discussion}
\input{sections/discussion}

\section{Acknowledgements}
YL, EK, and HK have been partially supported by the ANU Futures Scheme. YL has been partially supported by Australia RTP scholarship. EK has been partially supported by LP200301612. HK has been partially supported by the SmartSat CRC. JAK, ZK and LEK have been supported in part by NSF 2008720, 2336612, and Rice University Funds. WT has been supported by NSF ITR 2127309---CRA CIFellows Project.

%
% ---- Bibliography ----

\bibliographystyle{sageH}
\bibliography{references}

\setcounter{section}{0}
\renewcommand{\thesection}{\Alph{section}}
\renewcommand{\thesubsection}{\Alph{section}.\arabic{subsection}}

\section{Appendix}
Details of mathematical derivations, proofs and experiments are reported here. We begin with theoretical portion of this work. For convenience, we repeat the theorems again in this section. Notations and assumptions are introduced as they appear. Experimental details are given in the second part of this section.   
\subsection{Theoretical Analysis}%\label{subsec: theory appendix}
\input{sections/appendix_theory}

\input{sections/appendix_exp}

\end{document}

%% file: custom/packages.tex
\usepackage[T1]{fontenc}%
\usepackage[utf8]{inputenc}%
\usepackage{xcolor}
\usepackage{amsmath} % standard math package
\usepackage{mathrsfs} % includes mathscr
\usepackage{amssymb} % includes mathbb
\usepackage{amsthm}
\usepackage{algorithm} % for algorithm environment
\usepackage{algpseudocode} % for algorithm body
\usepackage{subcaption}
\usepackage{mwe}
\usepackage{graphicx}
\usepackage{multirow}
\usepackage{multicol}
\usepackage{tablefootnote}
\usepackage{moreverb}
\usepackage{mathtools}
\mathtoolsset{showonlyrefs}

\usepackage{xurl}
\usepackage{type1cm}
\usepackage{subeqnarray}
\usepackage{xspace}
\usepackage{enumitem}
\definecolor{mycolor}{HTML}{BE830E} % ANU Gold

\usepackage[english]{babel}

\definecolor{anugold}{HTML}{BE830E}%\renewcommand{\arraystretch}{1.2}
\usepackage[breaklinks,colorlinks,bookmarksopen,bookmarksnumbered,allcolors=anugold]{hyperref}

%% file: custom/macros.tex
\newenvironment{myDescription}
{\begin{description}[leftmargin=0pt, labelindent=-5pt]
  \setlength{\leftskip}{0pt}
  \setlength{\labelsep}{0pt}
  \setlength{\itemsep}{0pt}
  \setlength{\parskip}{3pt}
  \setlength{\parsep}{0pt}
  \setlength{\partopsep}{0pt}}
{\end{description}}

\newcommand{\fref}[1]{Figure~\ref{#1}}
\newcommand{\tref}[1]{Table~\ref{#1}}

\newcommand{\pomcp}{\textsc{pomcp}\xspace}
\newcommand{\despot}{\textsc{despot}\xspace}
\newcommand{\magic}{\textsc{magic}\xspace}
\newcommand{\rmag}{\textsc{rmag}\xspace}
\newcommand{\mdp}{\textsc{mdp}\xspace}

\newcommand{\pomdp}{\textsc{pomdp}\xspace}
\newcommand{\pomdps}{\textsc{pomdp}s\xspace}

\newcommand{\kl}{\textsc{kl}}
\newcommand{\samplingHeuristic}{\textsc{SampleHeuristics}\xspace}

\newcommand{\nopLong}{Reference-Based Online \pomdp Planning via Rapid State Space Sampling\xspace} 
\newcommand{\nop}{\textsc{rop-r\scalebox{.8}{a}s\scalebox{.9}{3}}\xspace} 
\newcommand{\nopTitle}{ROP-RAS3\xspace} 

\newcommand{\Reals}{\mathbb{R}} % real numbers
\newcommand{\Exp}{\mathbb{E}} % expectation
\newcommand{\Prob}{P} % standard probability measure
\newcommand{\KL}{\operatornamewithlimits{KL}} % KL-divergence
\newcommand{\bigO}{\mathcal{O}} % big O
\newcommand{\simplex}{\varDelta} % probability distributions
\newcommand{\almosteverywhere}{\text{a}.\text{e}.}

\newcommand{\States}{\mathcal{S}} % states
\newcommand{\Actions}{\mathcal{A}} % actions
\newcommand{\Observations}{\mathcal{O}} % observations
\newcommand{\TP}{\mathcal{T}} % transition density
\newcommand{\OP}{\mathcal{Z}} % observation density
\newcommand{\Reward}{R} % reward functional
\newcommand{\st}{\ensuremath{s}} % a state
\newcommand{\nst}{\ensuremath{{s'}}} % a next state
\newcommand{\act}{\ensuremath{a}} % action
\newcommand{\mact}{\ensuremath{\vec{\act}}} % macro action
\newcommand{\obs}{\ensuremath{o}} % observation
\newcommand{\mobs}{\ensuremath{\vec{\obs}}} % macro observation
\newcommand{\discount}{{\gamma}} % discount factor
\newcommand{\bel}{\ensuremath{b}} % belief
\newcommand{\belp}{\ensuremath{b'}} % belief
\newcommand{\hist}{\ensuremath{h}} % history
\newcommand{\belSpace}{\ensuremath{\mathcal{B}}} % belief space
\newcommand{\policies}{\ensuremath{\Pi}} % stochastic policies
\newcommand{\Val}{\ensuremath{V}} % value
\newcommand{\pol}{{\pi}} % a policy
\newcommand{\optPol}{{\pol}^*} % an optimal policy
\newcommand{\belInit}{\ensuremath{b_0}\xspace} % initial belief

\newcommand{{\QF}}{\mathscr{Q}} % value functions
\newcommand{\refPol}{{\bar{\pol}}} % reference policy
\newcommand{\refVal}{\mathcal{V}} % value
\newcommand{\refQVal}{\mathcal{Q}}
\newcommand{\temp}{\eta} % temperature parameter
\newcommand{\optRefPol}{{\pol}^*} % an optimal policy
\newcommand{\lagrangian}{\mathcal{L}} % the lagarangian of the constrained opitmisation

\newcommand{\depth}{\text{depth}} % depth of history
\newcommand{\GenModel}{\mathscr{G}} % generative model
\newcommand{\rV}{{{\mathcal{V}}}} % algorithm's approximation of \refVal
\newcommand{\treeDepth}{\ensuremath{D}}

\newcommand{\simd}{\textsc{simd}\xspace}
\newcommand{\simt}{\textsc{simt}\xspace}
\newcommand{\cpu}{\textsc{cpu}\xspace}
\newcommand{\gpu}{\textsc{gpu}\xspace}

\newcommand{\sbmp}{\textsc{sbmp}\xspace}

\newcommand{\sbmps}{\textsc{sbmp}s\xspace}

\newcommand{\rrtc}{\textsc{rrt}-Connect\xspace}

\newcommand{\eg}{\emph{e.g.},\xspace}
\newcommand{\ie}{\emph{i.e.},\xspace}
\newcommand{\dof}{\textsc{d\scalebox{.8}{o}f}\xspace}

\newcommand{\vamp}{\textsc{vamp}\xspace}

\newcommand{\X}{Q}
\newcommand{\Xg}{\X_G}
\newcommand{\Xfree}{\X_{\text{free}}}

\newcommand{\x}{q}

\renewcommand{\xi}{\x_{I}}
\renewcommand{\path}{p}

\newcommand{\weight}{\ensuremath{\omega}}
\newcommand{\snWeight}{\Tilde{\ensuremath{\omega}}}
\newcommand{\snEstimator}{\Tilde{\ensuremath{\mu}}}
\newcommand{\rbssFull}{Reference-based Sparse Sampling\xspace}
\newcommand{\rbss}{\textsc{RBSS}\xspace}
\newcommand{\particlesNum}{\ensuremath{C_\States}}
\newcommand{\actionsNum}{\ensuremath{C_\Actions}}
\newcommand{\particleBeliefs}{\ensuremath{\bar{\bel}}}
\newcommand{\fullyObsPol}{\ensuremath{\pi^{\text{fo}}}}

\newtheorem{theorem}{Theorem}
\newtheorem{lemma}{Lemma}

\def\review#1{{\color{blue}{}{\it{#1}}{}}}

%% file: sections/introduction.tex
Motion planning in partially observed and non-deterministic environments is a critical component of reliable and robust robot operation. Partially Observable Markov Decision Processes (\pomdps)~\citep{SS1973, klc1998} are a natural way to formulate such problems. The key insight of the \pomdp framework is to represent uncertainty on the effects of actions, perceptions and initial states as probability distributions, and then reason about the best strategy to perform with respect to distributions over the problem's state space, called \emph{beliefs}, rather than the state space itself. 

Although \pomdp{}s' methodical reasoning about uncertainty comes at the cost of high computational complexity~\citep{PT1987}, the \pomdp framework is practical for many robotics problems, thanks in large part to sampling-based approaches. These approaches relax the problem of finding an optimal solution to an approximate one by sampling states from the belief space and computing the best action from only the samples. 
Scalable anytime methods under this approach (surveyed by~\cite{kurniawatiSurvey}) have been proposed for solving large \pomdp problems. 
However, computing good solutions to long-horizon (\eg{} $\geq 15$ look-ahead steps) \pomdps remains difficult. Here, look-ahead steps refer to the minimum number of actions needed for the agent to identify which action is most beneficial to reach the goal under the partially observable characteristics of the problem. \looseness=-1 

Early results from the literature~\citep{KDHL2011} indicate that Sampling-Based Motion Planning (\sbmp)---sampling-based approaches designed for deterministic motion planning---can be used to tackle the challenges of long-horizon problems in offline \pomdp planning. Specifically, \sbmps can be used to generate suitable macro-actions (\ie{} sequences of actions) to reduce the effective planning horizon for a \pomdp solver. Macro-actions generated via \sbmp automatically adapt to geometric features of the valid region and tend to cover diverse paths in the state space.\looseness=-1 

Although this approach performs well for offline \pomdp planning, it is often impractical for online planning for two reasons. First is the speed of \sbmps, which historically required hundreds of milliseconds to tens of seconds to find a single motion plan. Second, most online \pomdp planners~\citep{sv10, kurniawati2016online,despot17} exhaustively enumerate each action at each sampled belief in computing the best action to perform. Such enumerations limit the number of actions sampled at run time, hence restricting online planners to quickly cover a good reachable belief space, from which the optimal solution can be computed efficiently \citep{easypomdp}. However, the recently proposed Vector-Accelerated Motion Planning (\vamp) framework~\citep{tkk23} enables \sbmps to find solutions on microsecond timescales, multiple orders of magnitude faster than prior approaches. 
Concurrently, recently proposed \emph{reference-based \pomdp planners}~\citep{kkk23} remove the requirement for exhaustive enumeration of the entire action space to compute approximately optimal POMDP solutions. 

Leveraging the above two advances, we propose an online \pomdp solver, called \nopLong (\nop), which is a reference-based \pomdp planner that employs a \vamp-enhanced macro-action sampler as its underlying reference policy. We modify the formulation from~\cite{kkk23} so that the reference is defined as a policy instead of a belief-to-belief transition, which makes it easier to generate suitable macro-actions and apply them to the reference-based \pomdp solving approach. We show that the modified reference-based Bellman backup naturally allows processing continuous action spaces by replacing the optimization procedures in the \pomdp Bellman backup with expectations. By adopting the analysis strategy used in~\cite{Lim20:IJCAI, lim2023optimality}, we show a convergence rate that depends on $\actionsNum$, the number of actions a planner samples at each belief node instead of $|\Actions|$, the size of the action space reported in \cite{Lim20:IJCAI, lim2023optimality}. Although an optimal solution of a reference-based \pomdp can be different from the original \pomdp, the empirical section of our work demonstrates how this formulation enables us to solve challenging robotics tasks. We evaluate \nop on multiple long-horizon \pomdps, including 4 navigation tasks and 3 manipulation tasks, which require planning horizons of hundreds to thousands of steps. Comparisons with state-of-the-art online \pomdp planners---including \pomcp~\citep{sv10}, a \pomcp modification that uses \vamp to generate macro-actions, and \despot~\citep{despot17} with learned macro-actions~\citep{lee.rss,KL2023}---indicate that \nop substantially outperforms state-of-the-art methods in all evaluation scenarios. Learning-based methods like \magic~\citep{lee.rss} do not naturally extend to high-dimensional motion planning domains, but \nop is able to solve problems with dimensionality up to 35. \nop is also deployed to a Hello-Robot Stretch 3 mobile base manipulator. In our pedestrian crossing scenario, \nop is the only method that demonstrates smart maneuvers to dodge a moving pedestrian and navigate to the goal. Our code will be released after publication. 

\textbf{Remark}. This work extends our ISRR24 paper \cite{Liang2024Scaling}. We modified the original algorithm to deal with continuous action and observation spaces and provided a neater way to do backups in the belief tree. The convergence analysis of \nop is added together with 3 new manipulation simulations and 1 physical robot demonstrations.

%% file: sections/sbmp.tex
Sampling-based motion planning (\sbmp) is a common, effective family of algorithms (\eg~\cite{Kavraki1996,Kuffner2000}) for solving \emph{deterministic} motion planning problems~\citep{LaValle2006}.
They are able to find collision-free motions for high degree-of-freedom (\dof) robots in environments containing many obstacles by drawing samples from a robot's \emph{configuration space}, the set of all possible robot configurations, $q \in \X$. Collisions between the robot and the environment or with itself partition the configuration space into \emph{valid} ($\Xfree$) and \emph{invalid} ($\X \setminus \Xfree$) configurations.
A deterministic \emph{motion planning problem} is then a tuple $(\Xfree, \xi, \Xg)$ representing the task of finding a continuous path, $\path: [0, 1] \rightarrow \Xfree$, from an initial configuration, $\xi$, to a goal region, $\Xg \subseteq \Xfree$ (\ie{} $\path(0) = \xi$ and $\path(1) \in \Xg$). \sbmps{} solve such problems by building a discrete approximation of $\Xfree$ as a graph or tree connecting sampled configurations in $\Xfree$ by short, local motions.
Once both $\xi$ and $\Xg$ are connected by this structure, a valid robot motion plan can be found with graph search methods. The solutions are deterministic, open-loop plans, and may not be robust against uncertainties or changes in configuration spaces.

%% file: sections/pomdp.tex
An infinite-horizon \pomdp is defined as the tuple
$
    \langle \States, \Actions, \Observations, \TP, \OP, \Reward, \discount, \belInit\rangle
$
where $\States$ denotes the set of all possible states of the system being considered, $\Actions$ denotes the set of all possible actions, and $\Observations$ denotes the set of all possible observations. We assume measures are properly defined for these spaces.
In our context, the state space encompasses the robot's configuration space ($\States \equiv \X $). The transition function $\TP(\nst \, | \, \st, \act)$ is the conditional probability that the robot will be in state $\nst \in \States$ after performing action $\act \in \Actions$ from state $\st \in \States$. The observation function $\OP( \obs \, | \, \nst, \act)$ is the conditional probability that the agent perceives $\obs \in \Observations$ when it is in state $\nst \in \States$ after performing action $\act \in \Actions$. The reward is a bounded real-valued function $\Reward: \States \times \Actions \rightarrow \Reals$. The parameter $\discount \in (0,1)$ is the discount factor which ensures the objective function is well-defined. 

In general, the robot does not know the true state. At each time-step, the robot maintains a \emph{belief} about its state, which is a probability distribution over the state space. The space of all possible beliefs is denoted by $\belSpace:= \simplex(\States)$. The agent starts from a given initial belief, denoted as \belInit. 

If $\belp$ denotes the agent's next belief after taking an action $\act \in \Actions$ and receiving the corresponding observation $\obs \in \Observations$, then the updated belief is given by
\begin{equation} \label{eq.bel.update}
\belp(\nst) =\tau(\bel, \act, \obs) \propto \OP(\obs \, | \, \nst, \act) \int_{\States} \TP(\nst \, | \, \st, \act) \, \bel(\st) \textup{d}\st
\end{equation}
For any given belief $\bel$ and action $\act$, the expected reward is given by
$\Reward(\bel, \act) := \int_{\st \in \States} \Reward(\st, \act) \bel(\st)\textup{d}\st$.
A (stochastic) \emph{policy} is a mapping $\pol: \belSpace \rightarrow \simplex(\Actions)$. We denote its distribution for any given input $\bel \in \belSpace$ by $\pol(\cdot \, | \, \bel)$.
A policy is \emph{deterministic} if it only has support at a single point $\act \in \Actions$.
Let $\policies$ be the class of all policies.

Given a policy $\pol \in \policies$, we define the \emph{value function} $\Val^\pol: \belSpace \rightarrow \Reals$ to be the expected total discounted reward, 
$\Val^\pol(\bel) := \Exp [ \sum_{t = 0}^\infty \discount^t \Reward\big(\bel_t, a_t)\big)]$, where the expectation is taken with respect to the policy and the transition probability.

In this paper, a \emph{solution} to the \pomdp is a deterministic policy $\optPol \in \policies$ satisfying
\begin{equation} \label{eq.pomdp_obj}
\Val^{\optPol}(\bel) = \sup_{\pol \in \policies} \Val^\pol(\bel)
\end{equation}
on $\belSpace$. The Bellman equation
\begin{equation} \label{eq.bellman}
\Val(\bel) = \max_{\act \in \Actions} \Big[ \Reward(\bel, \act) + \discount \int_{\Observations} \Prob(\obs \, | \, \act, \bel) \Val\big( \tau(\bel, \act, \obs) \big)\textup{d}o\Big]
\end{equation}
is satisfied by the optimal value function $\Val^{\optPol}$. The intrinsic conditional probability 
\begin{equation}
\Prob( \obs \, | \, \act, \bel)
 := \int_{\States} \OP(\obs \, | \, \nst, \act) \left( \int_{\States} \TP( \nst \, | \, \st, \act) \bel(\st)\textup{d}\st\right)\textup{d}\nst
\end{equation}
is the probability that the agent perceives $\obs$, having performed the action $\act$, under the belief $\bel$.

%% file: sections/largePOMDPs.tex
Solving long-horizon \pomdps{} is a significant challenge: the branching factor as a result of actions and observation grows exponentially with respect to the planning horizon. Planning over \emph{macro-actions} (i.e., a set of action sequences) is a common approach to deal with long-horizon \pomdps{} \citep{theo2003, he10:puma, KDHL2011, flaspohler,lee.rss,KL2023}.
Although this reduces the effective planning horizon, choosing a set of good macro-actions for planning is critical and automatic construction of macro-actions can require significant effort; \eg{} the learning time for \magic~\citep{lee.rss} can be on the order of hours.

Irrespective of macro-action use, most online \pomdp planners~\citep{sv10, kurniawati2016online,despot17} exhaustively enumerate all actions from each sampled belief in order to estimate the optimal action. Many efficient optimization methods, which are generally based on gradient ascent, cannot be used effectively because computing the gradient of the \pomdp{} value function is very expensive. Therefore, existing planners seldom explore long-term information and tend to perform poorly for horizons greater than 15 steps.

Recently,~\cite{kkk23} have mitigated this limitation by solving a \emph{reference-based \pomdp}: a reformulated \pomdp whose reward objective incurs a \kl-penalty for deviating too far from a given stochastic \emph{reference belief-to-belief transition}: $\bar{U}(\act, \obs \mid \bel):=\Prob(\obs \mid \act,\bel)\bar{\pol}(\act \, | \, \bel)$ where $\bar{\pol}$ is a given stochastic reference policy. The reward with respect to $U$ is defined as $\Reward(\bel, U):=\int_{\Actions,\Observations}\Reward(\bel,\act) \allowbreak U(\act,\obs\mid \bel)\textup{d}a\textup{d}o$.
The value of a belief is given by
\begin{multline}\label{eq: ref-nips-bellman}
    \rV^*(\bel) = \sup_{U\in\mathcal{U}(b)}\biggl(\Reward(\bel, U)-\KL(U\,\|\, \bar{U})+ \\ \gamma\int_{\Actions,\Observations}U(a,o\mid b)\rV^*(\tau(b,a,o))\textup{d}a\textup{d}o\biggr)
\end{multline}
where $\mathcal{U}(b)\subseteq \Delta(\Actions\times \Observations)$ is the set of admissible transitions at belief $b$.
The \textsc{rhs} can be optimized analytically, effectively removing the need for enumerative optimization in online \pomdp solving and thereby reducing the branching factor of the belief tree. Preliminary results from~\cite{kkk23} indicate that policies generated by this procedure can outperform state-of-the-art \pomdp benchmarks on long-horizon \pomdps{}. However, the assumption of having access to a belief-to-belief transition is very strong and cumbersome to work with in practice. \cite{kkk23} also employs relatively crude reference policies and does not incorporate macro-actions into planning.

%% file: sections/convergence_background.tex
Despite many empirical advances for online \pomdp planning, little theoretical justification for the convergence rates of these planners has been provided. Recently, \cite{Lim20:IJCAI, lim2023optimality} formally showed that a planner that uses observation likelihood weightings (a technique adopted by many online \pomdp planners) can estimate Q-values accurately with a convergence rate $\mathcal{O}(|\Actions|(|\Actions|\particlesNum)^D\exp(-t_{\max} \particlesNum))$, where $|\Actions|$ is the size of the action space, $\particlesNum$ is the number of state samples in a belief node, $D$ is the depth of the belief tree and $t_{\max}$ is a problem-specific bounding constant. The key of their analysis is to treat observation likelihood weighting as a self-normalized (SN) importance sampling process \citep{shachter1990simulation}. The latter aims to estimate the expectation of an arbitrary function $f(x)$ where $x$ is drawn from distribution $\mathcal{P}$, while the estimator only has access to another distribution $\mathcal{Q}$ along with the importance weights $\weight_{\mathcal{P}/\mathcal{Q}} \propto \mathcal{P}/\mathcal{Q}$. The following three quantities related to importance sampling are frequently used, 
\begin{equation}\label{eq: SN importance weight}
    \snWeight_{\mathcal{P}/\mathcal{Q}}(x):=\frac{\weight_{\mathcal{P}/\mathcal{Q}}(x)}{\sum_{i=1}^N \weight_{\mathcal{P}/\mathcal{Q}}(x_i)} 
\end{equation}
\begin{equation}\label{eq: Renyi divergence}
    d_\alpha(\mathcal{P}\| \mathcal{Q}) := \mathbb{E}_{x\sim \mathcal{Q}}[\weight_{\mathcal{P}/\mathcal{Q}}(x)^\alpha]
\end{equation}
\begin{equation}\label{eq: SN Estimator}
    \snEstimator_{\mathcal{P}/\mathcal{Q}} := \sum^N_{i=1}\snWeight_{\mathcal{P}/\mathcal{Q}}(x_i)f(x_i)
\end{equation}
Here, \eqref{eq: SN importance weight} is the SN importance weight, \eqref{eq: Renyi divergence} is the R\'enyi divergence and \eqref{eq: SN Estimator} is the SN estimator. Central to their analysis is the following self-normalized $d_\infty$-concentration bound, where $d_\infty$ is the infinite R\'enyi divergence.
\begin{theorem}\label{thm: SN concentration}[SN $d_\infty$-Concentration Bound \citep{Lim20:IJCAI}].
Let $\mathcal{P}$ and $\mathcal{Q}$ be two probability measures on the measurable space $(\mathcal{X}, \mathcal{F})$ with $\mathcal{P} \ll \mathcal{Q}$ and $d_\infty(\mathcal{P}||\mathcal{Q}) < +\infty$. Let $x_1,\dots,x_N$ be i.i.d. random variables sampled from $\mathcal{Q}$ and $f:\mathcal{X}\rightarrow \Reals$ be a bounded Borel function ($||f||_\infty < +\infty$). Then for any $\lambda > 0$ and $N$ large enough such that $\lambda > ||f||_\infty d_\infty(\mathcal{P}||\mathcal{Q})/\sqrt{N}$, the following bound holds with probability at least $1-3\exp(-N\cdot t^2(\lambda, N))$:
\begin{equation}
    |\Exp_{x\sim \mathcal{P}}[f(x)]-\Tilde{\mu}_{\mathcal{P}/\mathcal{Q}}| \leq \lambda 
\end{equation}
where $t(\lambda, N)$ is defined as,
\begin{equation}
    t(\lambda, N) \equiv \frac{\lambda}{||f||_\infty d_\infty(\mathcal{P}||\mathcal{Q})} - \frac{1}{\sqrt{N}}
\end{equation}
\end{theorem}
However, their analysis does not cover the case of continuous action spaces as exhaustively enumerating all actions for maximization is impossible. A subsequent work \citep{lim2021voronoi} proposed a Voronoi progressive widening technique to extend the convergence analysis to continuous action spaces. Instead of relying on Voronoi optimistic optimization and its convergence, we show that the reference policy acts as a natural action sampler that removes the need of online Q-value optimization, hence replacing the $|\Actions|$ in the convergence rate with $\actionsNum$, the number of actions we sample.

%% file: sections/refPomdp.tex
The concept of a reference-based \pomdp was introduced in~\cite{kkk23} as a generalization to \pomdp{}s of the \mdp formulations using \kl-penalization in~\cite{azar12,todorov}.
One limitation is that the formulation given by \eqref{eq: ref-nips-bellman} is somewhat artificial in that reference policies are stated with respect to \emph{belief-to-belief} transitions (see discussion in \S8 of~\cite{kkk23}).
In this paper, we will use a more natural formulation of a reference-based \pomdp over \emph{stochastic policies} rather than \emph{belief-to-belief transitions} as in~\cite{kkk23}. This allows us to directly work with reference policies which are easy to define and handcraft compared to belief-to-belief transitions. Further, the latter is less realistic as one does not have the choice of picking specific observations that resulted in the desired transitions, but is allowed to choose heuristic reference policies. Note that in doing so, all the benefits of the formulation in~\cite{kkk23} are still preserved.

Specifically, a reference-based \pomdp over stochastic actions is specified by the tuple
$\langle \States, \Actions, \Observations, \OP, \TP, \Reward, \discount, \temp, \refPol \rangle$. The presentation here treats the state, action and observation spaces as continuous spaces and assumes that the associated Borel $\sigma$-algebra exists and probability measures are properly defined with respect to the underlying $\sigma$-algebra. In addition to the standard parameters, we have a \emph{temperature} parameter $\temp > 0$ and a given (stochastic) \emph{reference policy} 
$\refPol(\cdot \, | \, \bel)$. We assume that $\refPol(\cdot\mid \bel) > 0, \forall \bel\in\belSpace$. 
The value $\refVal$ (distinguished from the \pomdp $V$-value) of a reference-based \pomdp for a given $\bel \in \belSpace$ satisfies the reference-based Bellman equation,
\begin{multline} \label{eq.ref.bellman}
        \refVal(\bel) = \sup_{\pol \in \policies} \Big[ \Reward(\bel, \pol) - \frac{1}{\temp} \KL( \pol \, \| \, \refPol ) \\ 
        + \discount \int_{\Actions, \Observations} \Prob(\obs \, | \, \act, \bel) \pol(\act \, | \, \bel) \refVal\big( \tau( \bel, \act, \obs) \big)\, \textup{d}\act \, \textup{d}\obs \Big]
\end{multline}
where $\Reward(\bel, \pol) := \int_{\Actions, \States} \Reward(\st, \act) \pol(\act \, | \, \bel) \bel(\st)\;\textup{d}\act\;\textup{d}\st $ is the reward estimate.
A \emph{solution} is a stochastic policy $\pol \in \policies$ that maximizes $\refVal$.
The problem can therefore be viewed as a \kl-penalized \pomdp whose objective is modified to trade off two (potentially competing) objectives: (1) abide by the reference policy, and (2) maximize reward.
The trade-off is balanced by $\temp$ and the \emph{quality} of the reference policy---higher-quality reference policies are those that reduce the \kl-divergence between the solution of the unpenalized \pomdp \eqref{eq.bellman} and the reference policy. The \kl-penalty implies that $\pol(\cdot\mid \bel) \ll \refPol(\cdot \mid \bel)$ for any $\bel \in \belSpace$.  When the optimal policy is used as the reference policy, we trivially obtain the optimal policy as the solution, since the optimal policy is deterministic and the only solution that maximizes rewards and minimizes KL divergence is itself. 

For an arbitrary $\refPol$, the supremum in \eqref{eq.ref.bellman} can be attained analytically by extending an argument of~\cite{azar11,azar12} to \pomdps,

\begin{theorem}\label{thm: analytical solution} [Analytical Solution of Reference-based \pomdp{}].
The exact solution of \eqref{eq.ref.bellman} is given by,
\begin{equation}\label{eq.maximized}
    \refVal(\bel) = \frac{1}{\temp} \log \Big[\int_\Actions \refPol(\act \, | \, \bel) \exp \Big\{ \temp \refQVal(\bel, \act) \Big\} \, \textup{d} \act\Big].
\end{equation}
Moreover, the \emph{exact} solution of the Reference-based \pomdp is given by,
\begin{equation}\label{eq.rbpomdp.opt.pol}
\optRefPol(\act \, | \, \bel)
\propto \refPol(\act \mid \bel)\exp(\temp \refQVal(\bel, \act))
\end{equation}
\end{theorem}
In above theorem, we have defined the reference-based Q-value,
\begin{equation}\label{eq: refQVal}
     \refQVal(\bel, \act) = \Reward(\bel,\act) + \int_\Observations \Prob(\obs \mid \act, \bel)\refVal(\tau(\bel, \act, \obs)) \, \textup{d}\obs
 \end{equation}
Derivation details can be found in the Appendix \ref{sec: appendix analytical solution}. The iterative backup procedure is guaranteed to converge to a unique solution $\refVal^*$ from which the policy can be read off using \eqref{eq.rbpomdp.opt.pol}~\cite{kkk23}. The main point is that enumerative maximization in \eqref{eq.bellman} is replaced with expectation. 

This new backup brings multiple advantages. First, gradients of the Bellman equation in \pomdps are generally hard to compute because it requires an estimate of the expected total future reward. This difficulty in computing gradients makes numerical optimization in \pomdp solving to generally be very computationally expensive. 
Our method removes this barrier and opens new ways to solve \pomdp{}s. Specifically, it partially solves  the optimization analytically, leaving numerical computation to only estimation of expectation, which can often be computed quickly via the Monte Carlo approach. Of course, if the probability distribution function is one that the Monte Carlo approach struggles with, such as those with rapidly oscillating functions or has very high variance, for instance, \nop requires much more careful consideration in its action selection (\ie, sampling strategy and selection of the reference policy). Second, the effective action space is reduced to the support of the reference policy, which can be much smaller than the entire action space, making \nop scalable to high-dimensional continuous action space.

These ideas are elaborated upon in the later sections, where we show that replacing exhaustive enumeration with Monte Carlo estimation provides a convergence rate that depends on $C_\Actions$, the number of actions sampled, rather than $|A|$, the size of the action space. Meanwhile, our algorithm \nop, using the new backup, outperforms many baselines across diverse long horizon tasks.

\textbf{Remark} The optimal solution of a reference-based \pomdp considered here is not the same as the optimal solution to the standard \pomdp, as the objective functions considered are different. It is natural to ask whether it is possible to improve the reference policy over time by treating the newly obtained solution as the next step reference policy, leading to convergence to the optimal solution of the standard POMDP. The answer is positive, as presented in~\cite{ijcai2025p949}. However, such a convergence requires higher computational resources and is beyond the scope of this paper. In this work, we demonstrate how reference policies constructed based on fast motion plans in the state space can lead to scalable robust solutions for robots operating in a non-deterministic and partially observable world.

%% file: sections/algorithm-sbmp.tex
\begin{algorithm*}[!th]
\caption{Online \pomdp Planner Tree Expansions with \sbmp-Generated Macro Actions}\label{alg: tree expand}
\label{alg:integration}
\textsc{SBMPExpand}($\Bar{\bel}, \actionsNum$)
\begin{algorithmic}[1]
    \State \textbf{Initialize} $l = \{\}$
    \For {$i = 0, \dots, \actionsNum - 1$}
        \State sample source state, $\st \sim \Bar{\bel}$
        \State sample target state, $\st' \sim \mathcal{J}(\cdot \mid \Bar{\bel})$
        \State construct a motion plan, $p \leftarrow \text{\sbmp}(\st, \st')$
        \State convert a motion plan to a macro-action, $\mact = \textsc{PathToMacroAction}(p)$
        \State append the new action, $l = l \cup \{\mact\}$
    \EndFor
    \State \textbf{Return} $l$
\end{algorithmic}
\end{algorithm*}

%% file: sections/vamp.tex
Towards fast deterministic motion planning, recent works have introduced new perspectives on \emph{hardware-accelerated} sampling-based motion planning (\sbmps), using either \cpu{} single-instruction, multiple-data (\simd)~\citep{tkk23} or \gpu{} single-instruction, multiple-thread (\simt)~\citep{sundaralingam2023curobo} parallelism to find complete motion plans in tens of microseconds to tens of milliseconds.
In particular, the authors of \vamp~\citep{tkk23} proposed a \simd-vectorized approach to computing \sbmp{} primitives (\ie{} local motion validation) that applies to all \sbmps{} thus, it is now possible to generate probabilistically-complete, global, collision-free trajectories for high-\dof systems at kilohertz rates---on the scale of tens of thousands of plans per second.
The key insight of this work is to lift the ``primitive'' operations of the \sbmp (e.g. forward kinematics and collision checking) to operate over \emph{vectors} of configurations in parallel.
Robot-specific code that uses these vector primitives is generated from a URDF using a tracing compiler---this pre-processing step is general to any robot. Functionally, this enables checking validity of a spatially distributed set of configurations over a candidate motion in parallel for the cost of a single collision check, massively lowering the expected time it takes to find colliding configurations along said motion. This development has called into question previously held perceptions that \sbmps{} are relatively time-expensive subroutines and---in the context of planning under uncertainty---opens the door to using \sbmps{} to guide \pomdp planning on the fly.

%% file: sections/integration.tex
A stochastic reference policy can be constructed from a deterministic policy. In robotics, such a representation can be naturally obtained by integrating deterministic motion plans with belief space sampling. We begin by presenting a general tree node expansion mechanism in Algorithm \ref{alg:integration} to integrate belief-space sampling of a canonical online \pomdp planner with state space sampling to generate macro-actions using \sbmp. Since an expansion mechanism to get new actions for a new belief node is universal in online \pomdp planning, Algorithm \ref{alg:integration} can use any \sbmp planner and can generalize to most existing online particle-based \pomdp planners, such as \pomcp \citep{sv10}, \despot \citep{despot17} and their derivatives \citep{kurniawati2016online,hypdespot21,lee.rss,KL2023}.

At a newly encountered belief node $\Bar{\bel}$, \textsc{SBMPExpand} in Algorithm \ref{alg:integration} aims to expand $\actionsNum$ of actions by firstly sampling a source state from the belief particles. In line 4, a target state is drawn from a distribution $\mathcal{J}(\cdot \mid \Bar{\bel})$ , a distribution over the state space $\States$ conditioned on the current belief particles. In robotics, target states can be states that reveal a certain amount of information to the robot, including states with high reward or penalty and states with observations. A deterministic motion plan is then constructed using any existing \sbmp in line 6. The path is converted to a macro-action via \textsc{PathToMacroAction} and added to the return list.

The choice of \sbmp is an essential component for achieving high-quality policies online. \vamp enables rapid generation of collision-free state space paths to goals or to highly informative states, which in turn can be used as macro-actions for the \pomdp planner. Thus, promising macro-actions can be dynamically created as a subroutine (line 5 in \textsc{SBMPExpandTree}) \emph{within} a \pomdp planner itself in fractions of a second.
In contrast, the state-of-the-art planner \magic~\citep{lee.rss} takes on the order of hours to learn macro-actions.

However, fast macro-action sampling alone is not sufficient, as current online \pomdp planners perform numerical optimization via enumeration of \emph{all} (macro) actions at each sampled belief, which results in a time and space complexity of  $\bigO(|\Vec{\Actions}|^{h})$, where $\Vec{\Actions}$ denotes the space of the macro-actions, $h$ is the planning horizon. This complexity significantly limits the number of macro-actions and the depth of the tree it can construct, thereby significantly limiting the benefit of \sbmp for \pomdp{} solving.
To tackle this problem, we propose \nop next that draws on insights from Section \ref{sec.nop} while, in tandem, exploiting the speed of \vamp~\citep{tkk23}---an implementation of \simd-vectorized \sbmp---to create a rich set of diverse macro-actions, which the planner uses to efficiently explore relevant parts of the belief space.

%% file: sections/algorithm.tex
\begin{algorithm*}[h]
\caption{\textsc{\nop}}
\label{alg: ref-prm}
\begin{multicols}{2}
\begin{algorithmic}[1]
\State Initialize tree $T$ rooted at $\particleBeliefs$
\While {time permitting}
    \State \Call{Simulate}{$\particleBeliefs$}
\EndWhile
\State \Return $T$
\end{algorithmic}

\vspace{0.1cm}
\textsc{Simulate}($\particleBeliefs$)
\begin{algorithmic}[1]
\State Sample $\st$ from $\particleBeliefs$
\If {$\depth(\hist) > \treeDepth$}
    \State \Return \Call{Rollout}{$\particleBeliefs, \st$}
\EndIf
\State $\mact \gets \Call{ActionProgressiveWiden}{\particleBeliefs, \st}$
% \State Resample $\st$ from belief particles of $\hist$
\State Sample $(\nst, \mobs, r(\st, \mact; \discount))$ from generative model $\GenModel(\st, \mact)$
\If {$|\mathcal{C}(\particleBeliefs\mact\mobs)| \leq \beta_oN(\particleBeliefs\mact)^{\alpha_o}$}
    \State $M(\particleBeliefs\mact\mobs) \gets M(\particleBeliefs\mact\mobs) + 1$
\Else 
    \State sample $\mobs$ from $\mathcal{C}(\particleBeliefs\act)$ w.p $\frac{M(\particleBeliefs\mact\mobs)}{\sum_{\mobs} M(\particleBeliefs\mact\mobs)}$
\EndIf
\State Create nodes for $\particleBeliefs \mact$ and $\particleBeliefs \mact \mobs$ if not created already
\State Add $\nst$ to belief particles of $\particleBeliefs \mact \mobs$
\State \Return $\rV(\particleBeliefs)\gets\Call{Backup}{\particleBeliefs, \mact, \mobs, r}$
\end{algorithmic}

\vspace{0.1cm}
\textsc{ActionProgressiveWiden}($\particleBeliefs, \st$)
\begin{algorithmic}[1]
    \If {$|\mathcal{C}(\particleBeliefs)| \leq \beta_a (N(\particleBeliefs) ^ {\alpha_a})$}
        \State \Return $\Call{SampleMacroActionSBMP}{\particleBeliefs, \st}$
    \EndIf
    \State Sample $\mact$ from $\particleBeliefs.\text{children}$ uniformly
    \State \Return $\mact$
\end{algorithmic}

\vspace{0.1cm}
\textsc{SampleMacroActionSBMP}($\particleBeliefs, \st$)
\begin{algorithmic}[1]
    \State Get the current configuration $q_{\text{start}}$ from $\st$ (assert free)
    \State  $q_{\text{goal}} \gets \samplingHeuristic(\particleBeliefs)$ (assert free)
    \State Plan path $p$ from $q_{\text{start}}$ to $q_{\text{goal}}$ using fast \sbmp
    \State \Return $\mact = \Call{PathToMacroAction}{p} $ 
\end{algorithmic}

\vspace{0.1cm}
\textsc{Backup}($\particleBeliefs, \mact, \mobs, r$)
\begin{algorithmic}[1]
    \State $r \rightarrow r + \discount \Call{SIMULATE}{\bel\mact\mobs}$
    \State $N(\particleBeliefs) \leftarrow N(\particleBeliefs) + 1$
    \State $N(\particleBeliefs\mact) \leftarrow N(\particleBeliefs\mact) + 1$
    \State $\refQVal(\particleBeliefs\mact) \leftarrow \refQVal(\particleBeliefs\mact) + \frac{r - \refQVal(\particleBeliefs\mact)}{N(\particleBeliefs\mact)}$
    \State $\refVal(\particleBeliefs) \leftarrow \frac{1}{\eta}\log\Big[\exp(\eta\refVal(\particleBeliefs)) + \frac{\exp(\eta \refQVal(\particleBeliefs\mact)) - \exp(\eta\refVal(\particleBeliefs))}{N(\particleBeliefs)}\Big] $
    \State \Return $\refVal(\particleBeliefs)$
    % \State \Return $\rW(\hist)$
\end{algorithmic}

\end{multicols}
\end{algorithm*}

We propose \nop (Algorithm \ref{alg: ref-prm}), a practical and scalable reference-based \pomdp online planner for motion planning under uncertainty. We use arrows on top of letters to indicate macro elements associated to that letter (e.g. $\mact$ is a macro action). With a slight abuse of notation, $\particleBeliefs$ denotes the belief node in a tree, $\particleBeliefs\mact$ denotes the action node $\mact$ associated with $\particleBeliefs$ and $\particleBeliefs\mact\mobs$ is the next belief node given $\mact$ and $\mobs$. The curly $\mathcal{C}(\cdot)$ indicates the children of a belief node (e.g. $\mathcal{C}({\particleBeliefs})$ returns all actions associated with $\particleBeliefs$). And $N(\cdot)$ denotes the visitation count of a node. Values are initialized to zero by default.

\subsubsection{Numerical Backup.}
Recall from Theorem \ref{thm: analytical solution}, the reference-based Bellman backup can be broken down into two parts, the $\refQVal$-value defined in \eqref{eq: refQVal} and the $\refVal$-value defined in \eqref{eq.maximized}. The $\refQVal$-value can be viewed as an expected total future reward, $\refQVal(\particleBeliefs\mact) \approx \frac{1}{N(\particleBeliefs\mact)}\sum_{i=1}^N R_i$, where $R_i$ is the reward return from the $i$-th simulation at the node $\particleBeliefs\mact$. From \eqref{eq.maximized}, by realizing that $\exp(\temp \refVal(\bel)) = \Exp[\exp(\temp \refQVal(\bel, \act))]$, we can then estimate this term with $\frac{1}{N(\particleBeliefs\mact)}\sum_{i=1}^N \exp(\temp \refQVal(\bel, \act))$. The iterative versions of both estimators are given in line 4 and 5 of the \textsc{Backup} function in Algorithm \ref{alg: ref-prm}. Further details and convergence rates of such nested Monte Carlo Estimators can be found in \cite{rainforth2018nesting}.

Due to the closed-form solution found in Theorem \ref{thm: analytical solution}, we do not need to solve the bandit problem, which is typical of many online \pomdp planners, dropping the need for UCB-like techniques. Second, without the need to expand all actions encountered at a given new node and restarting search from the root node, we can search deeper earlier by continuing to simulate from the newly sampled action, prioritizing long-horizon search commonly found in many robotics problems. Such tree search strategy is tightly linked to the backup equation \eqref{eq.maximized}. 

\subsubsection{Planning Tree Construction.}
\nop uses unweighted belief particles. This has the speed advantage of only retaining particles that have been sampled and obtained from the generative model during simulations. In \textsc{simulation}, \nop aims to get an action and query the generative model to simulate for the next state, rewards and observations. Since both the action space and the observation space can be continuous, we apply the double progressive widening technique \citep{sunberg2018online} to ensure sampled elements can be revisited again in a new simulation. Our progressive widening technique is different to the standard implementation in the sense that we only need to uniformly sample from the existing actions and observations if the widenning condition is not met, thanks to Monte Carlo estimations in our backup procedures. If the condition of expanding a new action is met, then \nop proceeds to query a macro action $\mact$ from a reference policy induced from a \sbmp (see \textsc{SampleMacroActionSBMP}), instead of uniformly sampling from the free action space as done in \citep{sunberg2018online}. The generative model $\GenModel$ is used to simulate for the next state, reward and observations. The macro action, observation, and the next state is added to the tree if they are not created yet. This repeats up to a predefined depth $D$. When the required depth is reached, \nop obtains an estimate of the node's value by rollouts, a standard operation used by other online planners. The planner then approximates the exact backup (\textsc{Backup}) by carefully \emph{maintaining} an empirical expectation, repeating backups on the simulated belief-tree path up to the root node. The above procedure is repeated until exhausting a computation budget (e.g., time), at which point the estimate of the optimal policy is read off from the tree's root node via \eqref{eq.rbpomdp.opt.pol}.

\subsubsection{Reference Policy.}
The simplicity offered by the formulation comes at a cost: if the optimal policy of the \pomdp with an unmodified objective is too far from the reference policy (in the sense of the \kl-divergence), pure reward maximization can be compromised.
Of course, the optimal policy and hence this divergence are not known a priori. 
Instead \nop assumes that reference policies generated by accelerated \sbmps (\vamp) provide a reasonable starting point, leveraging them to rapidly sample high-quality deterministic policies to induce a reasonable partially observed reference policy which it deforms online. The subroutine \textsc{SampleMacroActionSBMP} acts as the foundation for our reference policies. It lifts \vamp to belief space by sampling a state from the input belief particles, pick a target state with information (observation or reward) and query \vamp to build a deterministic motion plan connecting the two states. This motion plan is then converted to macro actions for \nop. The target states can be either uniformly drawn from the free space or dynamically sampled based on the belief of the agent. Sampling these informative states ensure our planning tree only cover a compact reachable belief space, retaining optimal solution while reducing the search complexity as uninformative path occurs less frequently.

We emphasize that this modification is not negligible.
Indeed, our results show that the performance deteriorates with problem complexity if the agent only executes the \sbmp reference policy as an open-controller without planning for uncertainty. More implementation details are discussed in Section \ref{sec:experiments} and our code website.

\subsection{Convergence Analysis}
\begin{algorithm*}[!th]
\caption{\rbssFull}
\label{alg: rbss}
\begin{multicols}{2}
\textsc{\rbss}($\discount, \GenModel, \actionsNum, \particlesNum, \particleBeliefs_0$)
\begin{algorithmic}[1]
\For{$i = 1, \dots, \actionsNum$}
    \State sample $\act_i \sim \refPol(\cdot\mid \particleBeliefs)$
    \State $\hat{\refQVal}_0(\particleBeliefs, \act_i) = $ \textsc{Estimate}-$\refQVal(\particleBeliefs, \act_i, d)$
\EndFor
\State $\hat{\pol} \propto \refPol(\act \mid \bel)\exp(\eta \hat{\refQVal}_0)$
\State \Return $\act \sim \hat{\pol}$
\end{algorithmic}

\vspace{1mm}
\textsc{Estimate}-$\refVal$($\particleBeliefs, d$)
\begin{algorithmic}[1]
\If {$d \geq D$}
    \State \Return 0
\EndIf 
\For{$i = 1, \dots, \actionsNum$}
    \State sample $\act_i \sim \refPol(\cdot\mid \particleBeliefs)$
    \State $\nu_i = \sum_{j=1}^{\particlesNum} \mathbb{I}(\fullyObsPol(\st_j)=\act_i)\weight_j$
    \State $\hat{\refQVal}_d(\particleBeliefs, \act_i) = $ \textsc{Estimate}-$\refQVal(\particleBeliefs, \act_i, d)$
\EndFor
\State \Return $\hat{\refVal}_d(\particleBeliefs) = \frac{1}{\eta}\log\big[\frac{\sum_{i=1}^{\actionsNum}\nu_i \exp(\eta\hat{\refQVal}_d(\particleBeliefs, \act_i))}{\sum_{i=1}^{\actionsNum}\nu_i}\big]$
\end{algorithmic}

\vspace{1mm}
\textsc{Estimate}-$\refQVal$($\particleBeliefs, \act, d$)
\begin{algorithmic}[1]
\For{$\st_i \in \particleBeliefs$}
    \State $\st'_i, \obs_i, r_i = \GenModel(\st_i, \act)$
\EndFor
\For{$i = 1, \dots, \particlesNum$}
    \For{$j = 1, \dots, \particlesNum$}
        \State $\weight'_j = \weight_j\OP(\obs_i\mid \st'_j, \act)$ 
    \EndFor
    \State ${\particleBeliefs}'_{i} = (\st'_i, \weight_i)$
\EndFor
\State \Return $\hat{Q}_d(\particleBeliefs,\act) = \frac{\sum_{i=1}^{\particlesNum} \weight_i r_i + \discount \textsc{Estimate-}\refVal(\particleBeliefs', d+1)}{\sum_{i=1}^{\particlesNum} \weight_i}$
\end{algorithmic}
\end{multicols}
\end{algorithm*} 

With the maximization steps in \pomdp planning being replaced by an expectation estimation, \nop{}'s convergence can be analyzed by bounding $\refQVal$-value estimations and $\refVal$-Value estimations at all nodes of the tree. To make analysis tractable, we simplify \nop and distill its core computations into a theoretical algorithm, \rbssFull (Algorithm~\ref{alg: rbss}). At the core of \rbss is a recursive self-normalized importance sampling process to estimate the beliefs based on the observation weightings. At a belief node, \rbss expands all sampled actions and observations to obtain a full belief particle set in the next step. This does not have the same computational advantage of only updating those particles that are sampled for the next step simulation as done in \nop, but allows us to leverage Theorem \ref{thm: SN concentration} to provide anytime estimation error bounds at all nodes. Further, by noticing that the reference policy in \nop is created from deterministic motion plans by sampling the source state from the belief and the target state from a heuristic, we assume that the reference policy $\refPol$ used in \rbss is also created from an underlying fully observable policy $\fullyObsPol: \States \rightarrow \Delta\Actions$ as a pushforward measure of $\bel(s)$,
\begin{equation}\label{eq: ref pol struct}
    \refPol(\cdot\mid \bel) = \int_\States \delta_{\fullyObsPol(\st)}(\cdot)\bel(\st) \, \textup{d}\st
\end{equation}
where $\delta$ denotes the Dirac measure. Such representation allows us to compute the weights of an action based on the particle representation of the belief. More details about \rbss are discussed in the Appendix \ref{sec: appendix convergence}. Theorem~\ref{thm: q value estimate bound} shows \rbss has a convergence rate of $\mathcal{O}(\actionsNum(\actionsNum\particlesNum)^D\exp(-\min\{\actionsNum, \particlesNum\}t^2_{\max}))$.

\begin{theorem}\label{thm: q value estimate bound}[Accuracy of \rbss $\refQVal$-Value Estimate].
    For any $\epsilon > 0$, choose constants $\actionsNum$, $\particlesNum$, $\lambda$, $\delta$ that satisfy, 
    $$\lambda = \epsilon(1-\discount)^2/5,$$
    $$\delta = \lambda/(\refVal_{\max}D(1-\discount)^2),$$
    $$\delta \geq 3\actionsNum(3\actionsNum\particlesNum)^D\exp(-\min\{\actionsNum, \particlesNum\}t^2_{\max}),$$
    \begin{equation}
        t_{\max} = \frac{\lambda}{3\refVal_{\max}d^{\max}_\infty} - \frac{1}{\sqrt{\min\{\actionsNum, \particlesNum}\}} > 0,
    \end{equation}
    Then, the $\refQVal$-values estimates obtained for all depth D and sampled actions $a$ are near-optimal with probability at least $1-\delta$,
    \begin{equation}
        \big|\refQVal^*_d(\bel_d, \act) - \hat{\refQVal}^*_d(\particleBeliefs_d, a) \big| \leq \frac{2\lambda}{1-\discount}.
    \end{equation}
\end{theorem}

The proof is adapted from \cite{Lim20:IJCAI} and is detailed in Appendix \ref{sec: appendix convergence}. The key contribution in our proof is to view the \textsc{Estimate}-$\refVal$ as another SN importance sampling estimator. For this estimator to be properly defined, we need to rely on the representation of reference policy defined in \eqref{eq: ref pol struct}. Since \eqref{eq: ref pol struct} is a pushforward measure of $\bel(s)$ and $\bel(s)$ is estimated via observation weightings (see \textsc{Estimate}-$\refQVal$), the importance weightings of $\refPol$ are simply the observation weightings reweighted by the support of $\fullyObsPol$. Then, we can apply the SN concentration bound from Theorem \ref{thm: SN concentration} to bound the estimation errors incurred in the \textsc{Estimate}-$\refVal$ step. Further, the estimation errors from \textsc{Estimate}-$\refQVal$ can be calculated following the same steps as the presentation of \cite{Lim20:IJCAI}. Hence, the final result is a worst-case union bound that combines all estimation errors from all depths of the tree. Since our \textsc{Estimate}-$\refVal$ procedure does not exhaustively enumerate through the action space, we can directly work with continuous action spaces where the number of actions we sample is a key parameter that controls the convergence rate. 

%% file: sections/experiments.tex
\begin{table*}[!htbp]
\caption{\pomdp Summary of Simulation Scenarios. $H_{\textup{min}}$ denotes the minimum number of steps needed to complete the goal without any uncertainty. Since these scenarios are goal-reaching problems, the minimum planning horizon of the \pomdp problem must be higher than the respective $H_{\textup{min}}$. Discrete($x$) indicates the action space consists of $x$ discrete actions. None corresponds to no observation due to partial observability. Discretized motion path refers to snapping primitive actions to a motion path.}
\label{tab: pomdp size summary}
\centering
\resizebox{0.7\textwidth}{!}{
\begin{tabular}{l|l|l|l|l|l|l|}
\cline{2-7}
                                             & $\States$     & $\Actions$    & $\Observations$                       & $H_{\text{min}}$ & Uncertainties                                                                            & \textsc{PathToMacroActions}                                        \\ \hline
\multicolumn{1}{|l|}{\textbf{Light-Dark}}    & $\Reals^2$    & Discrete (4)  & $\Reals^2 \cup \{\textup{None}\}$     & 8                & Initial position                                                                         & \begin{tabular}[c]{@{}l@{}}Discretised motion path up to \\ a predefined length \end{tabular} \\ \hline
\multicolumn{1}{|l|}{\textbf{Maze2D}}        & $\Reals^2$    & Discrete (4)  & $\Reals^2 \cup \{\textup{None}\}$     & 100              & \begin{tabular}[c]{@{}l@{}}Initial position,\\ Drone transitions\\ Odometry\end{tabular} & \begin{tabular}[c]{@{}l@{}}Discretised motion path up to \\ a predefined length\end{tabular} \\ \hline
\multicolumn{1}{|l|}{\textbf{Random3D}}      & $\Reals^3$    & Discrete (6)  & $\Reals^3 \cup \{\textup{None}\}$     & 40               & \begin{tabular}[c]{@{}l@{}}Initial position\\ Drone transitions\\ Odometry\end{tabular}  & \begin{tabular}[c]{@{}l@{}}Discretised motion path up to \\ a predefined length \end{tabular}  \\ \hline
\multicolumn{1}{|l|}{\textbf{Multi-Drone}}   & $\Reals^{15}$ & Discrete (24) & $\Reals^3 \cup \{\textup{None}\}$     & 20               & Target locations                                                                         & \begin{tabular}[c]{@{}l@{}}Discretised motion path up to \\ a predefined length\end{tabular}  \\ \hline
\multicolumn{1}{|l|}{\textbf{Sphere-Search}} & $\Reals^{13}$ & $\Reals^7$    & $\Reals^{13} \cup \{\textup{None}\} $ & 50               & \begin{tabular}[c]{@{}l@{}}Target locations\\ Arm transitions\end{tabular}               & \begin{tabular}[c]{@{}l@{}}Motion path up to a \\ predefined length\end{tabular}    \\ \hline
\multicolumn{1}{|l|}{\textbf{Ray-Detect}}    & $\Reals^{31}$ & $\Reals^7$    & $\Reals^{31} \cup \{\textup{None}\} $ & 80               & \begin{tabular}[c]{@{}l@{}}Obstacle poses,\\ Arm Transitions.\end{tabular}               & \begin{tabular}[c]{@{}l@{}}Motion path up to a\\ predefined length\end{tabular}       \\ \hline
\multicolumn{1}{|l|}{\textbf{Shelf-Move}}    & $\Reals^{35}$ & $\Reals^7$    & $\Reals^{35}\cup \{\textup{None}\} $  & 300              & \begin{tabular}[c]{@{}l@{}}Obstacle poses,\\ Arm Transitions.\end{tabular}               & Entire motion path                                                                  \\ \hline
\end{tabular}}
\end{table*}

We evaluated \nop on seven different long-horizon simulated \pomdp problems, systematically analyzing the effects of its respective components on its performance. The simulated testing scenarios include 2D navigation, 3D navigation, multi-robot tag, and manipulation problems. We also demonstrate \nop's practicality by deploying it to a Hello-Robot Stretch 3 mobile manipulator. Finally, an ablation study is carried out to understand \nop's performance as the reference policy gradually deteriorates to a purely uniform policy. Empirically, we show that the success of \nop is a combination of the new tree search method and reference policies induced from \vamp.

\subsection{Simulation Scenarios and Benchmark Methods}
The scenarios are described next, while a summary of the corresponding \pomdp models is provided in \tref{tab: pomdp size summary}. More detailed \pomdp definitions of each scenario are described in Appendices \ref{sec: appendix simulation details}, \tref{tab: pomdp benchmark space summary} and \tref{tab: pomdp benchmark func summary}. 

\begin{figure*}[t]
        \centering
        \begin{subfigure}[b]{0.35\textwidth}
            \centering
            \includegraphics[width=\textwidth]{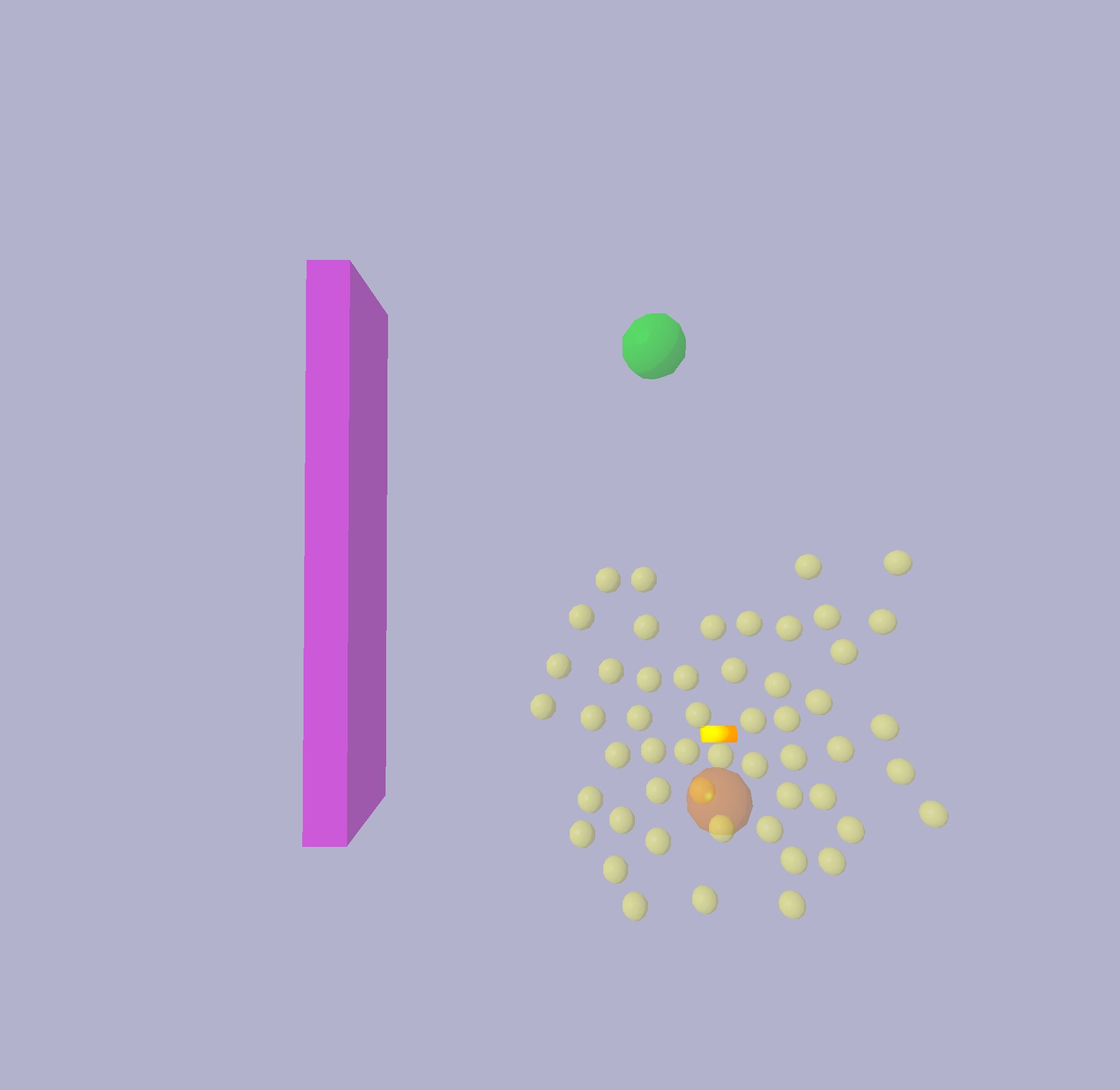}
            \caption[Light Dark]%
            {{\small Light Dark}}    
            \label{fig: visual light dark}
        \end{subfigure}
        \begin{subfigure}[b]{0.35\textwidth}  
            \centering 
            \includegraphics[width=\textwidth]{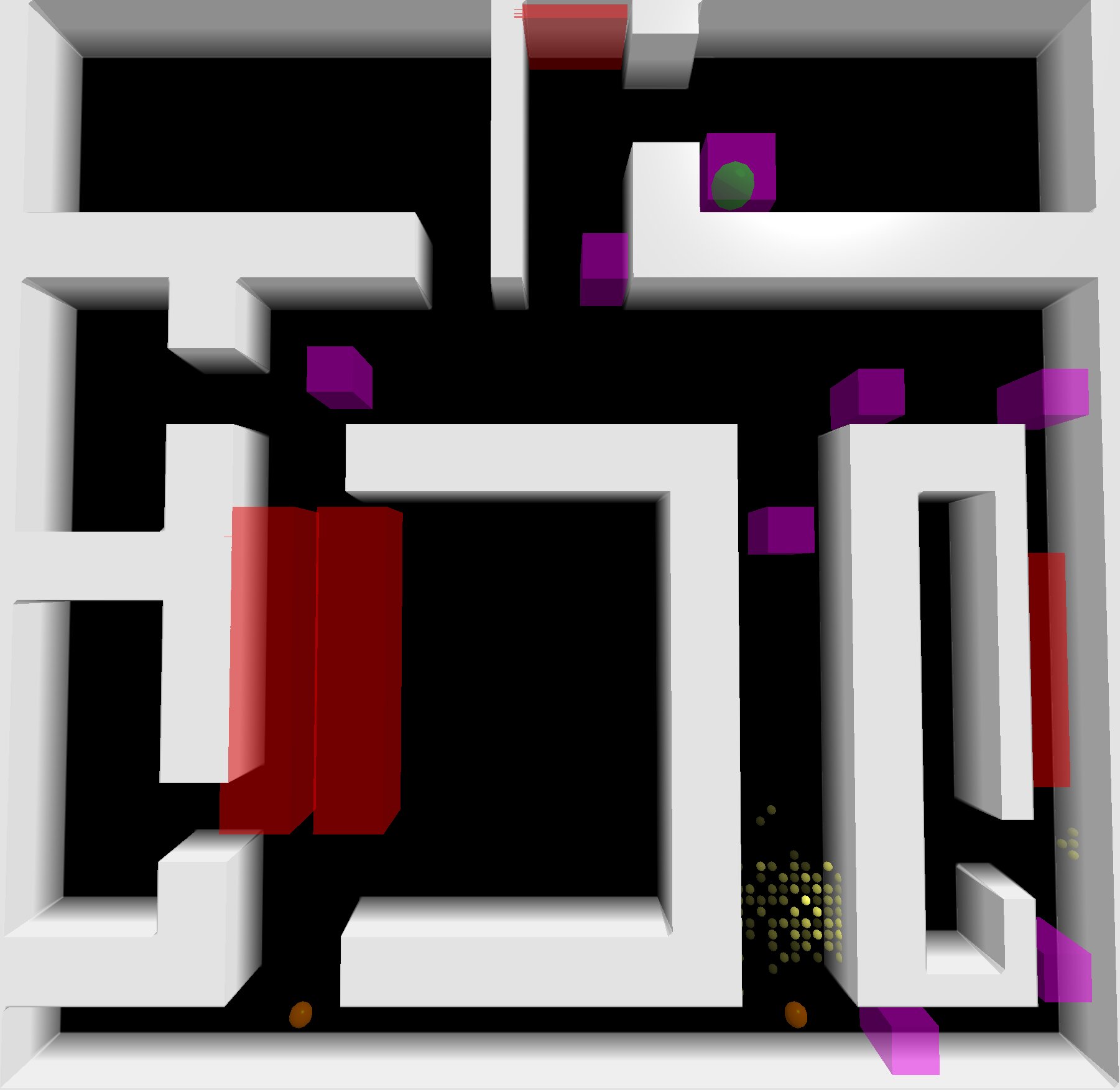}
            \caption[]%
            {{\small Maze2D}}    
            \label{fig: visual maze 2D}
        \end{subfigure}
        \begin{subfigure}[b]{0.35\textwidth}   
           \centering 
            \includegraphics[width=\textwidth]{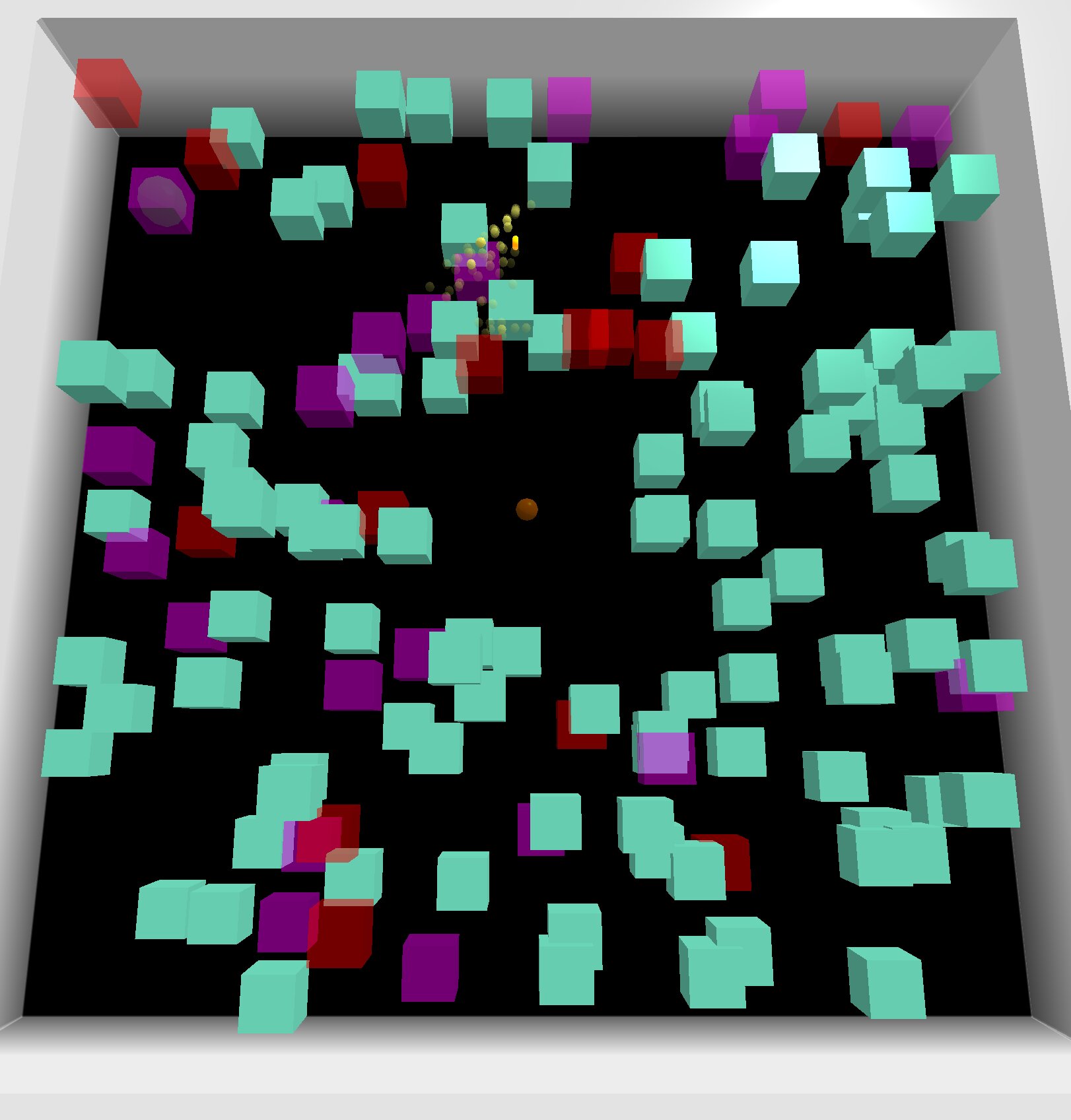}
            \caption[]%
            {{\small Random3D}}    
            \label{fig:visual Random 3D}
        \end{subfigure}
        \begin{subfigure}[b]{0.35\textwidth}   
            \centering 
            \includegraphics[width=\textwidth]{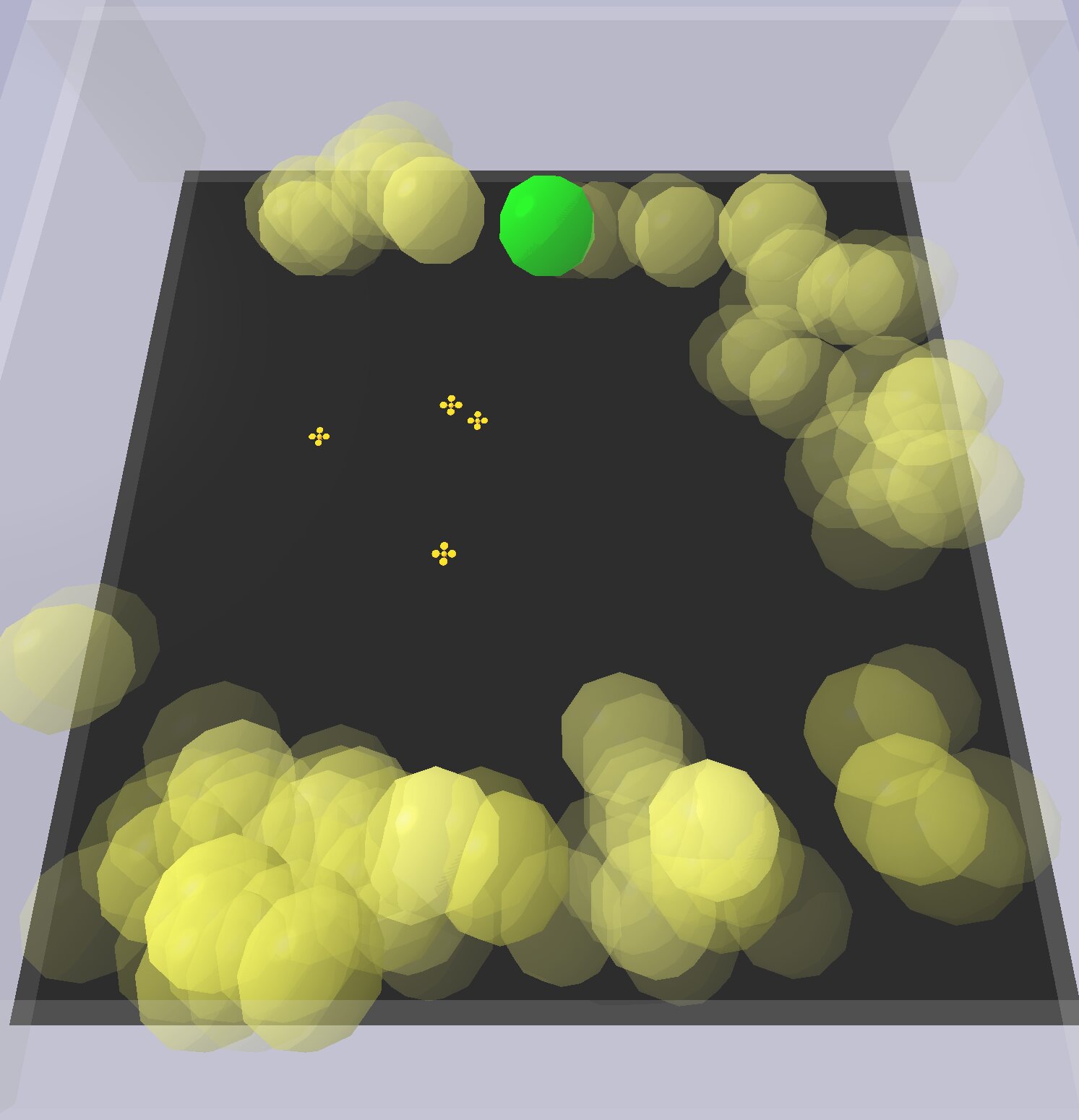}
            \caption[]%
            {{\footnotesize Multi-Drone Tag}}    
            \label{fig:visual capture}
        \end{subfigure}
        \caption[]
        {\small Navigation benchmark environments. All modes of starting locations, if any, are marked in \textcolor{orange}{\textbf{orange}}. Goals are marked in \textcolor{green}{\textbf{green}}. \textcolor{purple}{\textbf{Purple}} boxes, if any, indicate observation zones. Sampled belief particles are displayed in \textcolor{yellow}{\textbf{Yellow}} and the opacity denotes the weight of the particle. \textcolor{red}{\textbf{Red}} boxes are danger zones. Fixed walls are marked in \textcolor{gray}{\textbf{gray}} and randomly sampled obstacles are marked in \textcolor{cyan}{\textbf{cyan}}.
        }
        \label{fig: Navigation Visualisations}
\end{figure*}

% To achieve our aims,  we evaluated \nop on 4 different scenarios:
\begin{myDescription}
	\item[Light Dark (\fref{fig: visual light dark}).]~ A variation of the classical Light Dark problem. The agent needs to navigate to a goal region with an initially unknown location. It can only localize in the light stripe.     
	\item[Maze2D (\fref{fig: visual maze 2D}).]~This is a long-horizon problem, modified from the discrete 2D Navigation scenario in~\cite{KDHL2011}. A 2-\dof mobile robot must navigate from one of the two potential initial positions to a goal region without entering a danger zone and dodge obstacles. 
    The robot receives position readings with small Gaussian noise inside the landmarks and no observations otherwise. 
    This limited localization capability results in a minimum of 100 primitive actions to reach the goal, as the robot must take detours to localize and avoid danger zones.
	\item[Random3D (\fref{fig:visual Random 3D}).]~A 3D navigation problem with randomly placed obstacles, landmarks, danger zones and goals. The goals are initialized to be far away from the initial position of the drone.
    We systematically evaluate \nop on environments with progressively increasing obstacle density. The chance that the \sbmp fails to find a path due to narrow passages within a given time limit increases with the obstacle density and a \pomdp planner needs to consider more diverse macro-actions to find a good motion strategy.  

    \begin{figure*}[t]
        \centering
        \begin{subfigure}[b]{0.32\textwidth}
            \centering
            \includegraphics[width=\textwidth]{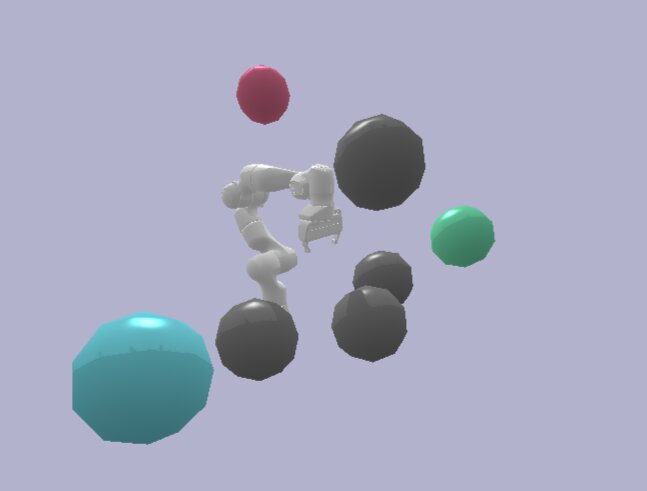}
            \caption[Sphere-Search]%
            {{\small Sphere-Search}}    
            \label{fig: visual sphere search}
        \end{subfigure}
        \begin{subfigure}[b]{0.32\textwidth}  
            \centering 
            \includegraphics[width=\textwidth]{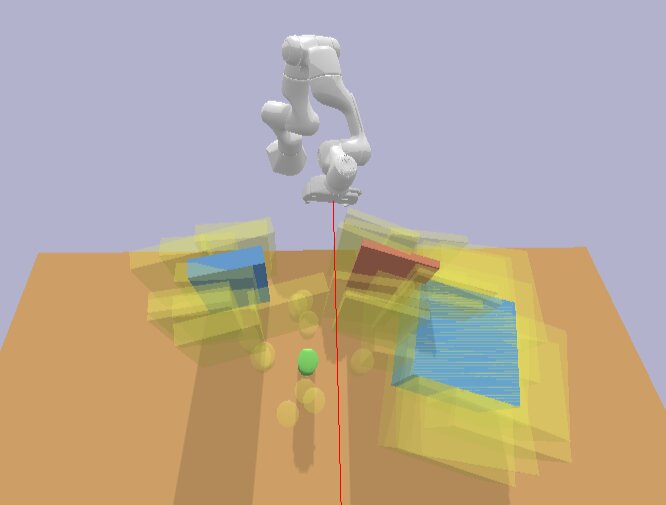}
            \caption[]%
            {{\small Ray-Detect}}    
            \label{fig: visual Ray-Detect}
        \end{subfigure}
        \begin{subfigure}[b]{0.32\textwidth}   
           \centering 
            \includegraphics[width=\textwidth]{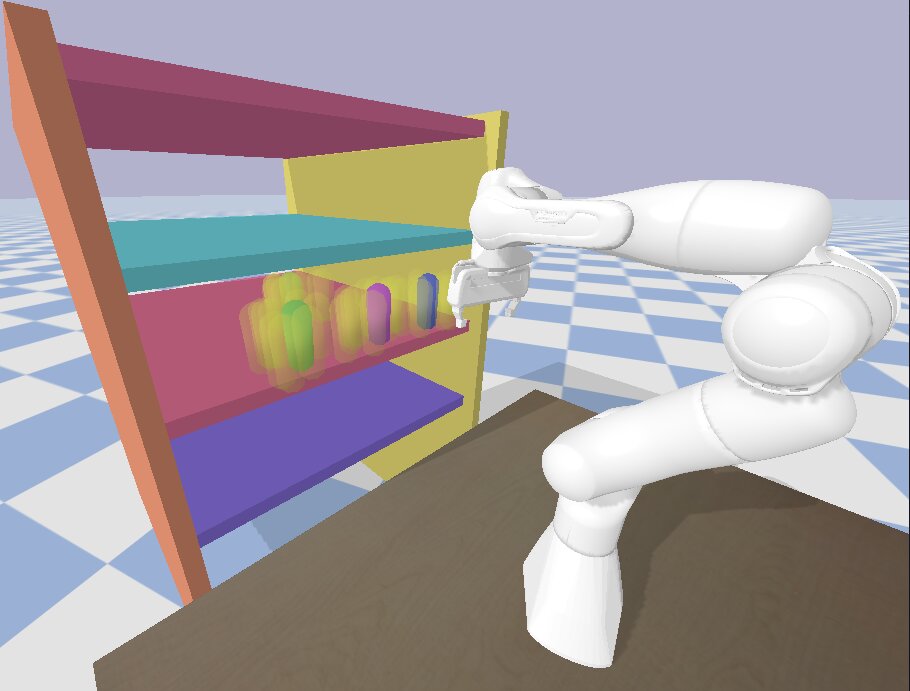}
            \caption[]%
            {{\small Shelf-Move}}    
            \label{fig:visual Shelf-Move}
        \end{subfigure}
        \caption[]
        {\small Manipulation benchmark environments. Belief particles of observed obstacles are displayed in \textcolor{yellow}{\textbf{yellow}}. In Sphere-Search, light is denoted as a \textcolor{purple}{\textbf{purple}} sphere. Goals are marked in \textcolor{green}{\textbf{green}}. Obstacles are marked in \textcolor{gray}{\textbf{gray}}. In Ray-Detect, obstacles are placed on a table. Rays from the arm is displayed as a \textcolor{red}{\textbf{red}} straight line. In Shelf-Move, movable obstacles are the cylinders inside the shelf.
        }
        \label{fig: Manipulation Visualisations}
\end{figure*}
     
	\item[Multi-Drone with A Teleporting Target (\fref{fig:visual capture}).]~ 
    In a 3D open area, four drones need to work together to capture a moving target (in green) whose initial position is not known to the drones. Moreover, the target can \textit{teleport} to the opposite side of the map once it collides with the map boundaries but drones cannot. The target is equipped with the strategy of moving away from the closest drone. This strategy is known to the drones. The drone can only detect the target if they are within a small detection range. And the target is captured as long as one drone is within the capture radius (smaller than detection range).
        \item[Sphere-Search (\fref{fig: visual sphere search})]. A 7-DOF manipulator needs to move from its initial configuration to an initially unknown goal location. The robot needs to take a detour to reach a light using its end-effector and observe the exact goal location. 

        \item[Ray-Detect (\fref{fig: visual Ray-Detect})]. A 7-DOF manipulator is tasked to use a fixed line of sight detector mounted at its end effector to scan environments, receives noisy readings of the obstacle positions and find a robust way to reach the target cylinder at the center of the table.
        
        \item[Shelf-Move (\fref{fig:visual Shelf-Move})]. This is the hardest problem that incorporates the difficulties of the other scenarios. A 7-DOF manipulator can move its end-effector close to obstacles to sense their positions. It is tasked to retrieve a target can, placed at the back of the shelf where the front is packed with obstacles, and put the target can at the intended location on the top shelf. The robot can remove the blocking obstacles, but needs to be careful not to place the obstacles at the location reserved for the target.
        
\end{myDescription}

The following baselines are used for comparisons,
\begin{myDescription}
    \item[Belief-VAMP (B-VAMP)]. The planner maintains the belief of the current state of the agent, but only takes actions returned by \vamp{}'s macro-action sampler \emph{without} \pomdp planning. It resembles the reference policy used by \nop.
    \item[Ref-Basic~\cite{kkk23}]. A reference-based \pomdp planner with a uniform sampling reference policy but no \vamp.
    \item[POMCP~\citep{sv10}]. A standard benchmark for online \pomdp planning.
    \item[R-POMCP]. An instantiation of Algorithm \ref{alg:integration} using \pomcp with macro-actions generated by the same reference policy as \nop, but a finite set of macro-actions is sampled for each belief node over which \pomcp optimizes. The sample size at each node is set to be roughly equal to that of \nop after progressive widening for fair comparisons.
    \item[MAGIC~\citep{lee.rss}]. A variation of \despot~\citep{despot17} that uses learnable macro-actions generated via an actor-critic approach.
    \item[RMAG~\citep{KL2023}]. \despot with macro-action learning boosted by recurrent neural networks.
\end{myDescription}

\subsection{Sampling Heuristics}\label{sec: sampling heuristics}
We detail the sampling heuristics (see \textsc{SampleMacroActionSBMP} in Algorithm \ref{alg: ref-prm}) used to create the reference policies for each problem. Biased sampling of informative states is not in itself a new idea, e.g., \citet{KDHL2011, flaspohler}; however, using such an idea to construct reference policies is fundamental to the success of \nop. Although it is generally impossible to encode (near) optimal solutions as the reference policy, it is much easier to create a reference policy that has compact support over actions that cover the optimal reachable belief space, which in turn makes \pomdp{}s efficient to solve \citet{easypomdp}. In our case, such a policy can be obtained by constructing deterministic motion plans to informative states in the world. Informative states refer to states with observations (e.g., landmarks) and rewards (e.g., goals). When these states are not known exactly by the agent, they are modeled as part of the state variable so we can sample these states from the belief particles.

We create two types of sampling heuristics for ablation purposes to understand the importance of the level of information revealed to the agent. The \textsc{Uniform} heuristic samples informative states in a uniform manner. It samples the goal state with probability 0.5 and samples the rest of informative states uniformly. The \textsc{Dynamic} heuristic varies its sampling strategy according to the belief such that when it is well informed, it is steered towards reaching the goal and when it is less informed, it inclines to gather information. Specifically, it samples the goal with probability $1 - \mathcal{H}(\bel_t)$, where $\mathcal{H}(\bel_t)$ is the normalized entropy of the current belief $\bel_t$. With probability $\mathcal{H}(\bel_t)$, it samples the rest of the informative states with probability inversely proportional to the distance of the agent to those states. 

An exception is the Multi-Drone problem due to its multi-agent and limited information nature. Since information is only revealed when the drones detect the target, previous heuristics would not be enough to cover the optimal solution as the target can teleport elsewhere but drones cannot. Instead, we create a sampling heuristic that commands the nearest drone to move towards a possible target location sampled from the belief and the rest of the drones spread out uniformly. To further understand how well our method performs, we also conduct ablation studies that gradually remove the information carried by the sampling heuristics in Section \ref{sec: ablations}.

\subsection{Experimental Setup}
\begin{table*}[!htbp]
\caption{Results on Light Dark and Maze2D (\textcolor{red}{Red} indicates the best result and \textcolor{blue}{blue} indicates the second best.)}
\label{tab: 2dexperimentsresults}
\centering
\resizebox{0.8\textwidth}{!}{
\begin{tabular}{cc|cccc|cccc|}
\cline{3-10}
                                       & \textbf{}                                                              & \multicolumn{4}{c|}{\textbf{Light Dark}}                                                                                                                                                                                                                                                   & \multicolumn{4}{c|}{\textbf{Maze 2D}}                                                                                                                                                                                                                                                      \\ \hline
\multicolumn{1}{|c|}{}                 & \textbf{\begin{tabular}[c]{@{}c@{}}Sampling \\ Heuristic\end{tabular}} & \multicolumn{1}{c|}{\textbf{\begin{tabular}[c]{@{}c@{}}Sims\\ \#\end{tabular}}} & \multicolumn{1}{c|}{\textbf{\begin{tabular}[c]{@{}c@{}}Succ. \\ \%\end{tabular}}} & \multicolumn{1}{c|}{\textbf{\begin{tabular}[c]{@{}c@{}}Rewards\\ (std err)\end{tabular}}} & \textbf{Steps}            & \multicolumn{1}{c|}{\textbf{\begin{tabular}[c]{@{}c@{}}Sims\\ \#\end{tabular}}} & \multicolumn{1}{c|}{\textbf{\begin{tabular}[c]{@{}c@{}}Succ.\\ \%\end{tabular}}} & \multicolumn{1}{c|}{\textbf{\begin{tabular}[c]{@{}c@{}}Rewards\\ (std err)\end{tabular}}} & \textbf{Steps}             \\ \hline
\multicolumn{1}{|c|}{B-\vamp}           & Uniform                                                                & \multicolumn{1}{c|}{N/A}                                                        & \multicolumn{1}{c|}{43.3}                                                         & \multicolumn{1}{c|}{38.4 (9.3)}                                                          & 50                        & \multicolumn{1}{c|}{N/A}                                                        & \multicolumn{1}{c|}{50}                                                          & \multicolumn{1}{c|}{-180 (211)}                                                          & 473                        \\ \hline
\multicolumn{1}{|c|}{B-\vamp}           & Dynamic                                                                & \multicolumn{1}{c|}{N/A}                                                        & \multicolumn{1}{c|}{76}                                                           & \multicolumn{1}{c|}{70.2 (7.9)}                                                          & 46                        & \multicolumn{1}{c|}{N/A}                                                        & \multicolumn{1}{c|}{43}                                                          & \multicolumn{1}{c|}{-495 (232)}                                                          & {\color[HTML]{3531FF} 430} \\ \hline
\multicolumn{1}{|c|}{\pomcp}            & N/A                                                                    & \multicolumn{1}{c|}{218}                                                        & \multicolumn{1}{c|}{56.7}                                                         & \multicolumn{1}{c|}{52.5 (9.3)}                                                          & 42                        & \multicolumn{1}{c|}{314}                                                        & \multicolumn{1}{c|}{0}                                                           & \multicolumn{1}{c|}{-80 (0)}                                                             & 800                        \\ \hline
\multicolumn{1}{|c|}{Ref-Basic}        & Dynamic                                                                & \multicolumn{1}{c|}{44}                                                         & \multicolumn{1}{c|}{50}                                                           & \multicolumn{1}{c|}{45.1 (9.3)}                                                          & 49                        & \multicolumn{1}{c|}{33}                                                         & \multicolumn{1}{c|}{0}                                                           & \multicolumn{1}{c|}{-851 (172)}                                                           & 509                        \\ \hline
\multicolumn{1}{|c|}{\magic}          & N/A                                                                    & \multicolumn{1}{c|}{N/A}                                                        & \multicolumn{1}{c|}{88}                                                           & \multicolumn{1}{c|}{85.0 (1.1)}                                                          & {\color[HTML]{FE0000} 29} & \multicolumn{1}{c|}{N/A}                                                        & \multicolumn{1}{c|}{0}                                                           & \multicolumn{1}{c|}{-80 (0)}                                                             & 800                        \\ \hline
\multicolumn{1}{|c|}{\rmag}           & N/A                                                                    & \multicolumn{1}{c|}{N/A}                                                        & \multicolumn{1}{c|}{88}                                                           & \multicolumn{1}{c|}{84.8 (1.1)}                                                          & {\color[HTML]{FE0000} 29} & \multicolumn{1}{c|}{N/A}                                                        & \multicolumn{1}{c|}{0}                                                           & \multicolumn{1}{c|}{-80 (0)}                                                             & 800                        \\ \hline
\multicolumn{1}{|c|}{R-\pomcp}           & Uniform                                                                & \multicolumn{1}{c|}{72}                                                         & \multicolumn{1}{c|}{80}                                                           & \multicolumn{1}{c|}{77.2 (8.1)}                                                          & {\color[HTML]{3531FF} 30} & \multicolumn{1}{c|}{92}                                                         & \multicolumn{1}{c|}{0}                                                           & \multicolumn{1}{c|}{-857 (174)}                                                          & 576                        \\ \hline
\multicolumn{1}{|c|}{R-\pomcp}           & Dynamic                                                                & \multicolumn{1}{c|}{73}                                                         & \multicolumn{1}{c|}{76.7}                                                         & \multicolumn{1}{c|}{73.9 (8.1)}                                                          & {\color[HTML]{3531FF} 30} & \multicolumn{1}{c|}{87}                                                         & \multicolumn{1}{c|}{0.0}                                                         & \multicolumn{1}{c|}{-403 (128)}                                                          & 699                        \\ \hline
\multicolumn{1}{|c|}{\textbf{\nopTitle}} & Uniform                                                                & \multicolumn{1}{c|}{21}                                                         & \multicolumn{1}{c|}{{\color[HTML]{FE0000} 96.7}}                                  & \multicolumn{1}{c|}{{\color[HTML]{FE0000} 94.0 (3.4)}}                                   & {\color[HTML]{FE0000} 29} & \multicolumn{1}{c|}{198}                                                        & \multicolumn{1}{c|}{{\color[HTML]{3531FF} 80}}                                   & \multicolumn{1}{c|}{{\color[HTML]{FE0000} 594 (62)}}                                     & {\color[HTML]{000000} 462} \\ \hline
\multicolumn{1}{|c|}{\textbf{\nopTitle}} & Dynamic                                                                & \multicolumn{1}{c|}{3}                                                          & \multicolumn{1}{c|}{{\color[HTML]{FE0000} 96.7}}                                  & \multicolumn{1}{c|}{{\color[HTML]{3531FF} 93.4 (3.4)}}                                   & 35                        & \multicolumn{1}{c|}{43.1}                                                        & \multicolumn{1}{c|}{{\color[HTML]{FE0000} 90.0}}                                 & \multicolumn{1}{c|}{{\color[HTML]{3531FF} 553 (129)}}                                     & {\color[HTML]{FE0000} 339} \\ \hline
\end{tabular}}
\end{table*}

For evaluation, all variants of \nop are implemented in Python following~\cite{zheng2020pomdp_py}. \vamp is implemented in C++ and is used via a Python \textsc{api}. We avoid implementing benchmark methods from scratch for a fairer comparison; the \pomcp implementation is from~\cite{zheng2020pomdp_py} and the implementations of \pomcp, \magic, \rmag and Ref-Basic all use the code implemented by their respective authors~\citep{lee.rss,KL2023}. PyBullet \cite{coumans2021} is used as a visualizer only. For the first four navigation problems, all methods are provided with the same planning time of \textbf{1s} in all scenarios with the exception of Light Dark where the planning time is \textbf{0.1s}. For manipulation problems, we fix the number of simulations per planning iteration and report the average planning time, as these problems have more computational overhead compared to navigation problems. The temperature parameter $\eta$ for \nop and Ref-Basic is $\eta=0.2$. The action progressive widening parameter is set to $\beta_a = 6$ and $\alpha_a = 0.05$. \magic and \rmag are trained on a 4070 GPU with data collected in 500,000 runs ($\sim$3 hours of training).  The parameters have been tuned to optimize performance for each problem. Other implementation details can be found in Appendix \ref{sec: appendix implementation details}.

\subsection{Results and Discussions}
\begin{table*}[!htbp]
\caption{Results on Random3D (\textcolor{red}{Red} indicates the best result and \textcolor{blue}{blue} indicates the second best.) \\ {\footnotesize(\#1: 100 obstacles, 15 danger zones. \#2: 200 obstacles, 10 danger zones. \\ \#3: 300 obstacles, 10 danger zones. \#4: 400 obstacles, 5 danger zones. \review{R.P stands for reference policy}.)}}

\label{tab:maze3dresults}
\centering
\resizebox{\textwidth}{!}{
\begin{tabular}{c|c|c|c|c|c|c|c|c|c|c|c|c|c|}
\cline{2-14}
                              & \textbf{\begin{tabular}[c]{@{}c@{}}Sampling\\ Heuristic\end{tabular}} & \textbf{\begin{tabular}[c]{@{}c@{}}Exp \\ \#\end{tabular}} & \textbf{\begin{tabular}[c]{@{}c@{}}Sims\\ \#\end{tabular}} & \textbf{\begin{tabular}[c]{@{}c@{}}Succ. \\ \%\end{tabular}} & \textbf{\begin{tabular}[c]{@{}c@{}}Rewards\\ (std err)\end{tabular}} & \textbf{Steps}             & \textbf{\begin{tabular}[c]{@{}c@{}}R.P\\  Fail \%\end{tabular}} & \textbf{\begin{tabular}[c]{@{}c@{}}Exp\\ \#\end{tabular}} & \textbf{\begin{tabular}[c]{@{}c@{}}Sims\\ \#\end{tabular}} & \textbf{\begin{tabular}[c]{@{}c@{}}Succ. \\ \%\end{tabular}} & \textbf{\begin{tabular}[c]{@{}c@{}}Rewards\\ (std err)\end{tabular}} & \textbf{Steps}             & \textbf{\begin{tabular}[c]{@{}c@{}}R.P\\  Fail \%\end{tabular}} \\ \hline
\multicolumn{1}{|c|}{B-\vamp}  & Uniform                                                               &                                                            & N/A                                                        & 30                                                           & -981 (223)                                                          & 354                        & N/A                                                             &                                                           & N/A                                                        & 20                                                           & -801(222)                                                           & 456                        & N/A                                                             \\ \cline{1-2} \cline{4-8} \cline{10-14}
\multicolumn{1}{|c|}{B-\vamp}  & Dynamic                                                               &                                                            & N/A                                                        & 6.67                                                         & -1237(190)                                                          & 521                        & N/A                                                             &                                                           & N/A                                                        & 6.67                                                         & -935(191)                                                           & 556                        & N/A                                                             \\ \cline{1-2} \cline{4-8} \cline{10-14}
\multicolumn{1}{|c|}{R-\pomcp}  & Uniform                                                               &                                                            & 18                                                         & 13.3                                                         & -500(179)                                                           & 588                        & N/A                                                             &                                                           & 10                                                         & 13.3                                                         & -147(140)                                                           & 668                        & N/A                                                             \\ \cline{1-2} \cline{4-8} \cline{10-14}
\multicolumn{1}{|c|}{R-\pomcp}  & Dynamic                                                               &                                                            & 19                                                         & 13.3                                                         & -553(184)                                                           & 612                        & N/A                                                             &                                                           & 8                                                          & 10.0                                                         & -263(140)                                                           & 697                        & N/A                                                             \\ \cline{1-2} \cline{4-8} \cline{10-14} 
\multicolumn{1}{|c|}{\textbf{\nopTitle}} & Uniform                                                               &                                                            & 29                                                         & {\color[HTML]{3531FF} 63.3}                                  & {\color[HTML]{FE0000} 10.6 (213)}                                   & {\color[HTML]{FE0000} 318} & 2.9                                                             &                                                           & 16                                                         & {\color[HTML]{FE0000} 63}                                    & {\color[HTML]{FE0000} -21.6(212)}                                   & {\color[HTML]{FE0000} 364} & 13                                                              \\ \cline{1-2} \cline{4-8} \cline{10-14} 
\multicolumn{1}{|c|}{\textbf{\nopTitle}} & Dynamic                                                               & \multirow{-6}{*}{1}                                        & 58                                                         & {\color[HTML]{FE0000} 66.7}                                  & {\color[HTML]{FE0000} 34 (212)}                                     & {\color[HTML]{3531FF} 336} & 2.8                                                             & \multirow{-6}{*}{2}                                       & 16                                                         & {\color[HTML]{3531FF} 53.0}                                  & {\color[HTML]{3531FF} -69.8(220)}                                   & {\color[HTML]{3531FF} 400} & 1                                                               \\ \hline
\multicolumn{1}{|c|}{B-\vamp}  & Uniform                                                               &                                                            & N/A                                                        & 23.3                                                         & -535(202)                                                           & 556                        & N/A                                                             &                                                           & N/A                                                        & 43.3                                                         & -171(198)                                                           & 520                        & N/A                                                             \\ \cline{1-2} \cline{4-8} \cline{10-14}
\multicolumn{1}{|c|}{B-\vamp}  & Dynamic                                                               &                                                            & N/A                                                        & 0                                                            & -1062(179)                                                          & 624                        & N/A                                                             &                                                           & N/A                                                        & 13.3                                                         & -233(140)                                                           & 730                        & N/A                                                             \\ \cline{1-2} \cline{4-8} \cline{10-14}
\multicolumn{1}{|c|}{R-\pomcp}  & Uniform                                                               &                                                            & 6                                                          & 10                                                           & {\color[HTML]{3531FF} -191(120)}                                    & 713                        & N/A                                                             &                                                           & 4                                                          & 10                                                           & -191(122)                                                           & 713                        & N/A                                                             \\ \cline{1-2} \cline{4-8} \cline{10-14}
\multicolumn{1}{|c|}{R-\pomcp}  & Dynamic                                                               &                                                            & 6                                                          & 10.0                                                         & -448(165)                                                           & 620                        & N/A                                                             &                                                           & 5                                                          & 6.67                                                         & -217(116)                                                           & 702                        & N/A                                                             \\ \cline{1-2} \cline{4-8} \cline{10-14} 
\multicolumn{1}{|c|}{\textbf{\nopTitle}} & Uniform                                                               &                                                            & 9                                                          & {\color[HTML]{FE0000} 56.7}                                  & {\color[HTML]{FE0000} -55.7(207)}                                   & {\color[HTML]{3531FF} 432} & 18                                                              &                                                           & 7                                                          & {\color[HTML]{FE0000} 56.7}                                  & {\color[HTML]{3531FF} 73.7(184)}                                    & {\color[HTML]{FE0000} 470} & 23                                                              \\ \cline{1-2} \cline{4-8} \cline{10-14} 
\multicolumn{1}{|c|}{\textbf{\nopTitle}} & Dynamic                                                               & \multirow{-6}{*}{3}                                        & 33                                                         & {\color[HTML]{3531FF} 50.0}                                  & -371(235)                                                           & {\color[HTML]{FE0000} 396} & 8                                                               & \multirow{-6}{*}{4}                                       & 15                                                         & {\color[HTML]{3531FF} 46.6}                                  & {\color[HTML]{FE0000} 183(130)}                                     & {\color[HTML]{3531FF} 576} & 9                                                               \\ \hline
\end{tabular}}
\end{table*}

\begin{table*}[!htbp]
\caption{Results on Multi Drones Tag and Sphere Search (\textcolor{red}{Red} indicates the best result and \textcolor{blue}{blue} indicates the second best.)}
\label{tab: Results on Multi-Drones and Sphere Search}
\centering
\resizebox{0.7\textwidth}{!}{
\begin{tabular}{c|cccc|cccc|}
\cline{2-9}
                                              & \multicolumn{4}{c|}{\textbf{Multi-Drone Tag}}                                                                                                                                                                                                 & \multicolumn{4}{c|}{\textbf{Sphere Search}}                                                                                                                                                                                                   \\ \hline 
                                              \multicolumn{1}{|c|}{} & \multicolumn{1}{c|}{\textbf{Sims \#}} & \multicolumn{1}{c|}{\textbf{\begin{tabular}[c]{@{}c@{}}Succ.\\ \%\end{tabular}}} & \multicolumn{1}{c|}{\textbf{\begin{tabular}[c]{@{}c@{}}Rewards \\ (std err)\end{tabular}}} & \textbf{Steps}             & \multicolumn{1}{c|}{\textbf{Sims \#}} & \multicolumn{1}{c|}{\textbf{\begin{tabular}[c]{@{}c@{}}Succ.\\ \%\end{tabular}}} & \multicolumn{1}{c|}{\textbf{\begin{tabular}[c]{@{}c@{}}Rewards \\ (std err)\end{tabular}}} & \textbf{Steps}             \\ \hline
\multicolumn{1}{|c|}{B-\vamp}  & \multicolumn{1}{c|}{N/A}               & \multicolumn{1}{c|}{26.7}                                                         & \multicolumn{1}{c|}{133 (51.3)}                                                      & 328                        & \multicolumn{1}{c|}{N/A}               & \multicolumn{1}{c|}{10}                                                          & \multicolumn{1}{c|}{65.4 (44.1)}                                                     & 145                        \\ \hline
\multicolumn{1}{|c|}{R-\pomcp} & \multicolumn{1}{c|}{34}                & \multicolumn{1}{c|}{{\color[HTML]{3531FF} 66.7}}                                   & \multicolumn{1}{c|}{{\color[HTML]{3531FF} 389 (52.6)}}                                  & {\color[HTML]{3531FF} 251} & \multicolumn{1}{c|}{150}               & \multicolumn{1}{c|}{{\color[HTML]{3531FF} 76.7}}                                 & \multicolumn{1}{c|}{{\color[HTML]{3531FF} 601 (61.9)}}                               & {\color[HTML]{3531FF} 124} \\ \hline
\multicolumn{1}{|c|}{\textbf{\nopTitle}}     & \multicolumn{1}{c|}{47}               & \multicolumn{1}{c|}{{\color[HTML]{FE0000} 90}}                                   & \multicolumn{1}{c|}{{\color[HTML]{FE0000} 524 (43)}}                                 & {\color[HTML]{FE0000} 157} & \multicolumn{1}{c|}{150}               & \multicolumn{1}{c|}{{\color[HTML]{FE0000} 100}}                                  & \multicolumn{1}{c|}{{\color[HTML]{FE0000} 790 (0.31)}}                               & {\color[HTML]{FE0000} 104} \\ \hline
\end{tabular}}
\end{table*}

\begin{table*}[!htbp]
\caption{Results on Ray-Detect and Shelf-Move (\textcolor{red}{Red} indicates the best result and \textcolor{blue}{blue} indicates the second best.)}
\label{tab: Results on Ray-Detect and Shelf-Move}
\centering
\resizebox{\textwidth}{!}{
\begin{tabular}{cc|ccccc|ccccc|}
\cline{3-12}
                                          &                                                                        & \multicolumn{5}{c|}{\textbf{Ray-Detect (Planning Horizon: 500)}}                                                                                                                                                                                                                                                                     & \multicolumn{5}{c|}{\textbf{Shelf-Move (Planning Horizon: 1500)}}                                                                                                                                                                                                                                                                     \\ \hline
\multicolumn{1}{|c|}{}                    & \textbf{\begin{tabular}[c]{@{}c@{}}Sampling\\ Heuristics\end{tabular}} & \multicolumn{1}{c|}{\textbf{Sims \#}} & \multicolumn{1}{c|}{\textbf{\begin{tabular}[c]{@{}c@{}}Succ.\\ \%\end{tabular}}} & \multicolumn{1}{c|}{\textbf{\begin{tabular}[c]{@{}c@{}}Rewards \\ (std err)\end{tabular}}} & \multicolumn{1}{c|}{\textbf{Steps}}             & \textbf{\begin{tabular}[c]{@{}c@{}}Plan\\ Time (s)\end{tabular}} & \multicolumn{1}{c|}{\textbf{Sims \#}} & \multicolumn{1}{c|}{\textbf{\begin{tabular}[c]{@{}c@{}}Succ.\\ \%\end{tabular}}} & \multicolumn{1}{c|}{\textbf{\begin{tabular}[c]{@{}c@{}}Rewards \\ (std err)\end{tabular}}} & \multicolumn{1}{c|}{\textbf{Steps}}              & \textbf{\begin{tabular}[c]{@{}c@{}}Plan\\ Time (s)\end{tabular}} \\ \hline
\multicolumn{1}{|c|}{B-\vamp}              & Uniform                                                                   & \multicolumn{1}{c|}{N/A}               & \multicolumn{1}{c|}{0}                                                           & \multicolumn{1}{c|}{-160}                                                           & \multicolumn{1}{c|}{800}                        & N/A                                                              & \multicolumn{1}{c|}{N/A}               & \multicolumn{1}{c|}{6.67}                                                        & \multicolumn{1}{c|}{-243 (38.9)}                                                    & \multicolumn{1}{c|}{2968}                        & N/A                                                              \\ \hline
\multicolumn{1}{|c|}{B-\vamp}              & Dynamic                                                                 & \multicolumn{1}{c|}{N/A}               & \multicolumn{1}{c|}{0}                                                           & \multicolumn{1}{c|}{-160}                                                           & \multicolumn{1}{c|}{800}                        & N/A                                                              & \multicolumn{1}{c|}{N/A}               & \multicolumn{1}{c|}{10}                                                          & \multicolumn{1}{c|}{-213 (47.7)}                                                    & \multicolumn{1}{c|}{2933}                        & N/A                                                              \\ \hline
\multicolumn{1}{|c|}{R-\pomcp}             & Uniform                                                                   & \multicolumn{1}{c|}{50}                & \multicolumn{1}{c|}{26.7}                                                        & \multicolumn{1}{c|}{-66 (28.6)}                                                     & \multicolumn{1}{c|}{733}                        & 4.4                                                              & \multicolumn{1}{c|}{300}               & \multicolumn{1}{c|}{10}                                                          & \multicolumn{1}{c|}{-213 (47.7)}                                                    & \multicolumn{1}{c|}{2934}                        & 35.6                                                             \\ \hline
\multicolumn{1}{|c|}{R-\pomcp}             & Dynamic                                                                 & \multicolumn{1}{c|}{50}                & \multicolumn{1}{c|}{20.0}                                                        & \multicolumn{1}{c|}{-93 (24.4)}                                                     & \multicolumn{1}{c|}{769}                        & 4.5                                                              & \multicolumn{1}{c|}{100}               & \multicolumn{1}{c|}{10.0}                                                        & \multicolumn{1}{c|}{-214 (47.2)}                                                    & \multicolumn{1}{c|}{2939}                        & 12.7                                                             \\ \hline
\multicolumn{1}{|c|}{\textbf{\nopTitle}} & Uniform                                                                   & \multicolumn{1}{c|}{50}                & \multicolumn{1}{c|}{{\color[HTML]{3531FF} 63.3}}                                 & \multicolumn{1}{c|}{{\color[HTML]{3531FF} 71.4 (32.3)}}                             & \multicolumn{1}{c|}{{\color[HTML]{3531FF} 597}} & 5.3                                                              & \multicolumn{1}{c|}{300}               & \multicolumn{1}{c|}{{\color[HTML]{3531FF} 23.3}}                                 & \multicolumn{1}{c|}{{\color[HTML]{3531FF} -96.7 (67.5)}}                            & \multicolumn{1}{c|}{{\color[HTML]{3531FF} 2835}} & 35.3                                                             \\ \hline
\multicolumn{1}{|c|}{\textbf{\nopTitle}} & Dynamic                                                                 & \multicolumn{1}{c|}{50}                & \multicolumn{1}{c|}{{\color[HTML]{FE0000} 80.0}}                                 & \multicolumn{1}{c|}{{\color[HTML]{FE0000} 137 (27.5)}}                              & \multicolumn{1}{c|}{{\color[HTML]{FE0000} 519}} & 4.6                                                              & \multicolumn{1}{c|}{100}               & \multicolumn{1}{c|}{{\color[HTML]{FE0000} 70.0}}                                 & \multicolumn{1}{c|}{{\color[HTML]{FE0000} 313 (57.1)}}                              & \multicolumn{1}{c|}{{\color[HTML]{FE0000} 2470}} & 8.9                                                              \\ \hline
\end{tabular}}
\end{table*}
We ran all methods 30$\times$ for each scenario and method, and the results are summarized in Tables~\ref{tab: 2dexperimentsresults},~\ref{tab:maze3dresults}, ~\ref{tab: Results on Multi-Drones and Sphere Search} and ~\ref{tab: Results on Ray-Detect and Shelf-Move}. Regardless of success or failure, the "Sims. \#"  and "Planning Times" columns present the average number of episodes simulated or the time spent in seconds in one planning call, whereas the steps column records the average execution steps for each run. Reward columns indicate the average total reward, and the standard errors are put in brackets. In ~\tref{tab:maze3dresults}, "R.P. Fail \%" column indicates the percentage of reference policy failures across 30 runs. Variants of \magic and \rmag were run only on Light Dark and Maze2D because they do not naturally extend to high-dimensional motion planning problems (e.g. manipulation and multi-agent tasks). \pomcp and Ref-Basic failed on all runs of Random3D and Multi Drones Tag and are therefore excluded from the corresponding tables. Subsequently, we do not expect that these two methods will perform well on manipulation tasks with even longer horizons and higher state dimensions, and hence they are excluded from those experiments. Refer to Figure~\ref{fig: ROPRAS Light-Dark Behavior Visualisations} to Figure~\ref{fig: ROPRAS Manipulation Behavior Visualisations} in the Appendix for examples of critical time stamps of \nop in performing each tasks. 

The results indicate that all variants of \nop substantially outperform all baseline methods in all evaluation scenarios. The improvement provided by \nop is the smallest for Light Dark and Sphere Search (the simplest evaluation scenarios for navigation and manipulations respectively) where all variants of \nop achieved a success rate of up to 28\% higher than R-\pomcp, \magic, and \rmag. The reason is that both of these scenarios require a much lower planning horizon and have less uncertainties compared to other scenarios. In \textbf{Light Dark}, this simplicity is indicated by the good performance of methods that do not use macro-actions, such as \pomcp and Ref-Basic, and by the good performance of B-\vamp, which reasons with respect to only a sampled state of the belief. In \textbf{Sphere Search}, the bimodal uncertainties associated with the goal cause troubles to B-\vamp, however, since the partial observability is fully resolved by reaching the light once and the light is not placed too far away from the manipulator, this problem has a relatively short effective horizon compared to other problems and therefore both R-\pomcp and \nop perform well with \nop outperforming R-\pomcp by ~30\%.

In the other scenarios, all variants of \nop improved the success rate of the benchmark methods by many folds. When the scenario requires a much longer planning horizon (e.g., Maze2D, Shelf-Move) or has higher uncertainty and more complex geometric structures (e.g. Ray-Detect and Shelf-Move), the benefit of \nop increases. By using \vamp, \nop can quickly generate much more diverse macro-actions that capture the geometric features of the problem well. Simultaneously, the reference-based \pomdp objective \eqref{eq.ref.bellman} enables \nop to utilize the sampled macro-actions more efficiently than other benchmark methods. The result is that \nop can consistently find the most robust motion path (often requiring longer horizons but has overall greater accumulated rewards) to interact with the environment and reach the goal. Unlike \nop, R-\pomcp expands all reference policy actions at once when encountering a new belief node in the tree, which means the actions being expanded cannot be adaptive as new belief particles were introduced to the node during planning. By growing the tree in a depth-first-search manner, \nop expands a new action edge at a belief node every time a new particle is added; this way, it benefits the most from the underlying reference policy. As theoretically shown, \nop removes the cumbersome optimization procedures in traditional methods with Monte-Carlo estimations, significantly reducing the number of samples of simulations required to achieve a high success rate in all benchmarks. For the most sophisticated benchmark---Shelf-Move, \nop only requires 100 simulations with under 10s planning time to achieve high success rate (see \tref{tab: Results on Ray-Detect and Shelf-Move}). Learning-based methods perform poorly in \textbf{Maze2D} due to the difficulty in learning suitable macro-actions with fixed length. In contrast, by using \sbmp, \nop is able to better capture the fine motions the robot needs to navigate the geometric features of the environment. 

\begin{figure*}
    \centering
    \begin{subfigure}[b]{0.24\textwidth}
         \centering
         \includegraphics[width=\textwidth]{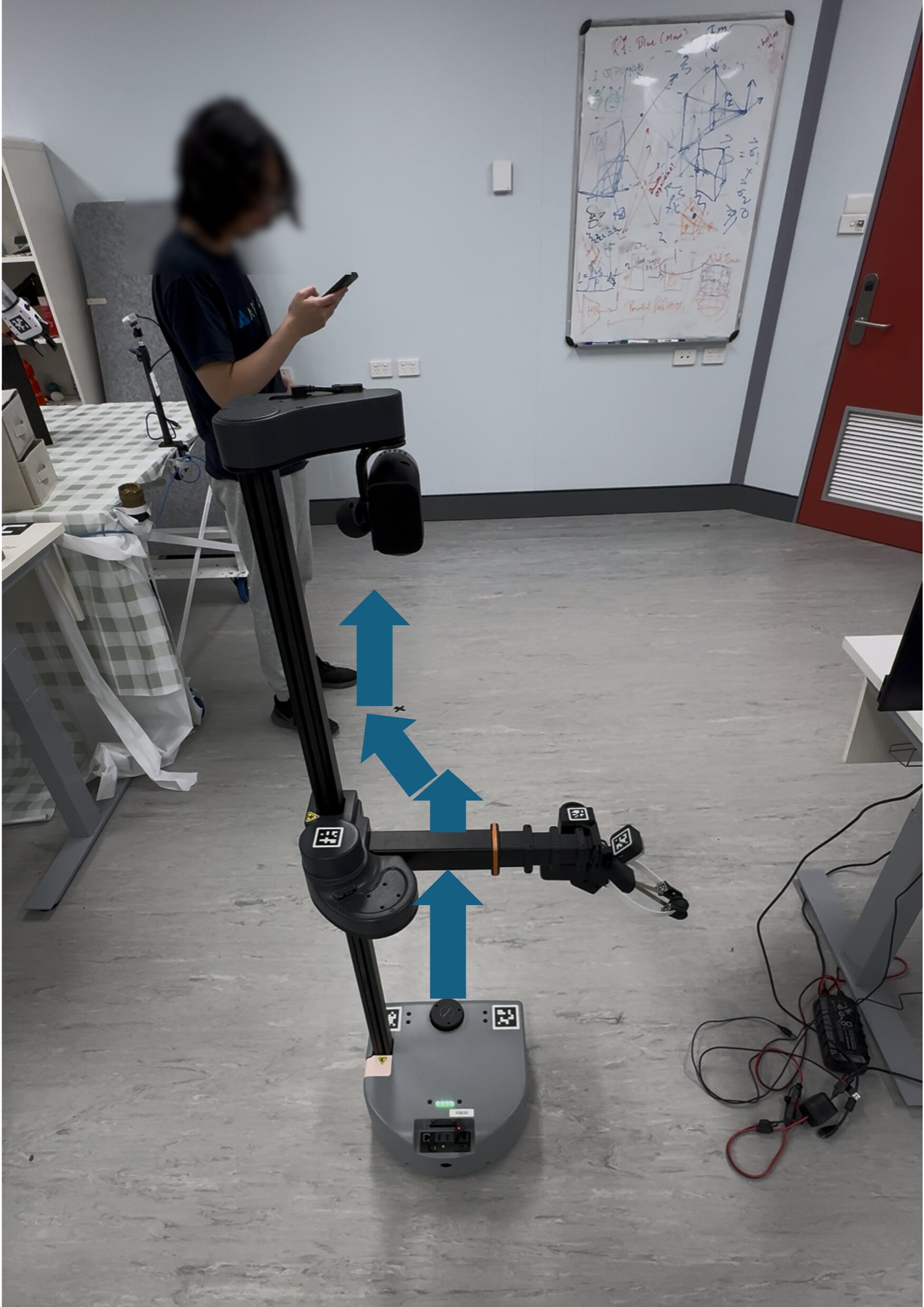}
         %\caption{Light Dark}
     \end{subfigure}
     \begin{subfigure}[b]{0.24\textwidth}
         \centering
         \includegraphics[width=\textwidth]{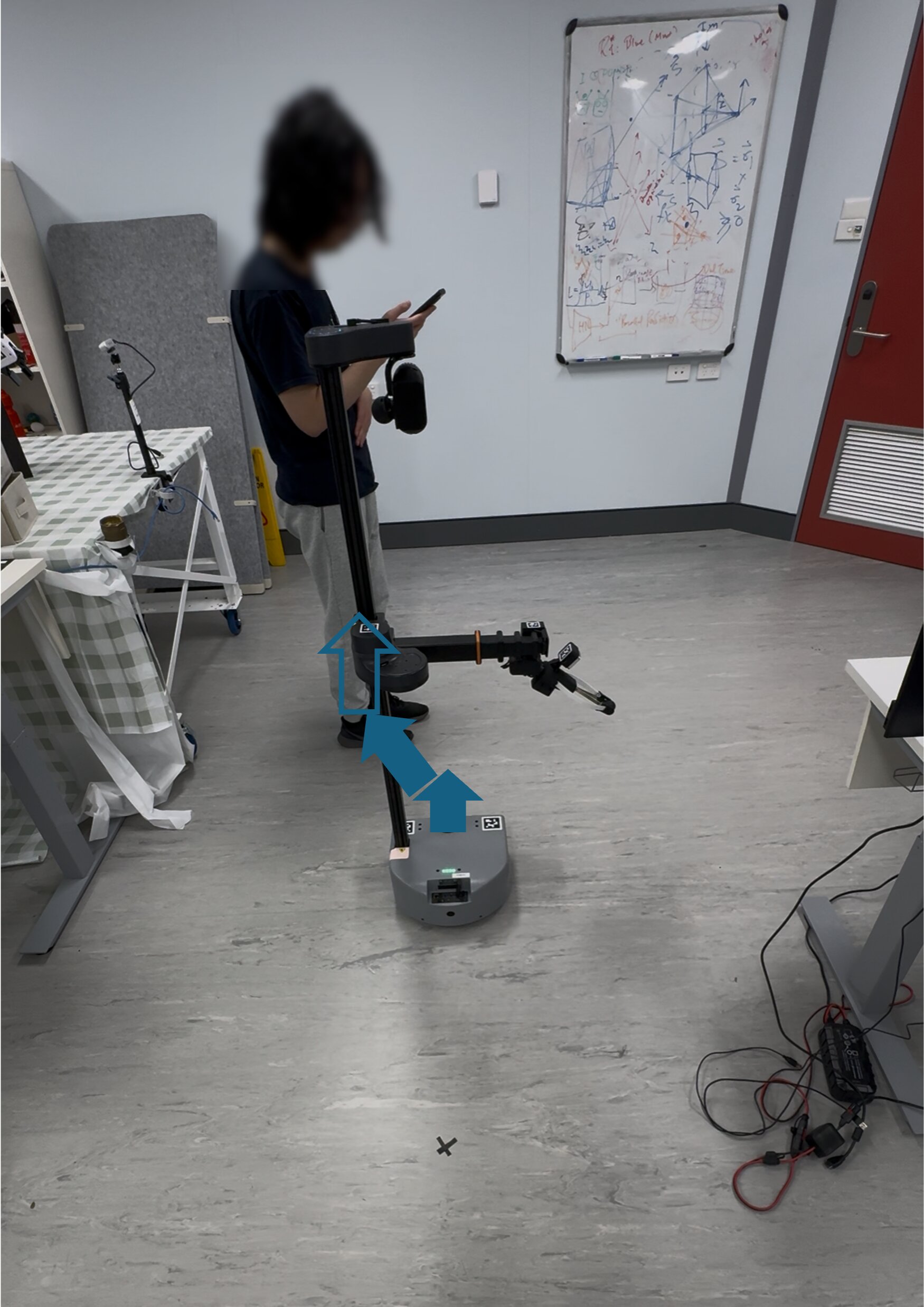}
         %\caption{Maze 2D}
     \end{subfigure}
     \begin{subfigure}[b]{0.24\textwidth}
         \centering
         \includegraphics[width=\textwidth]{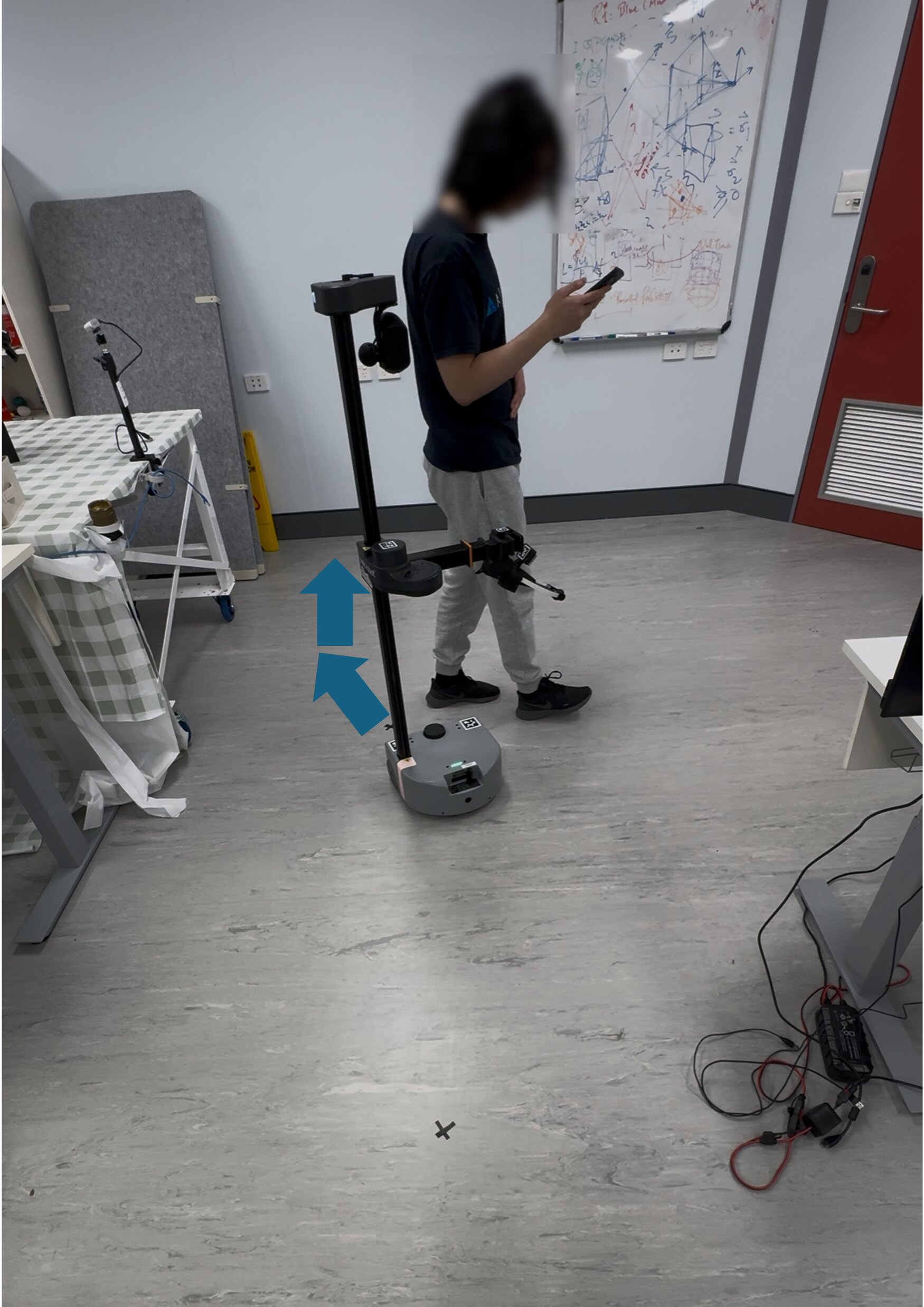}
         %\caption{Random 3D}
         %\label{fig:visual Random 3D}
     \end{subfigure}
     \begin{subfigure}[b]{0.24\textwidth}
         \centering
         \includegraphics[width=\textwidth]{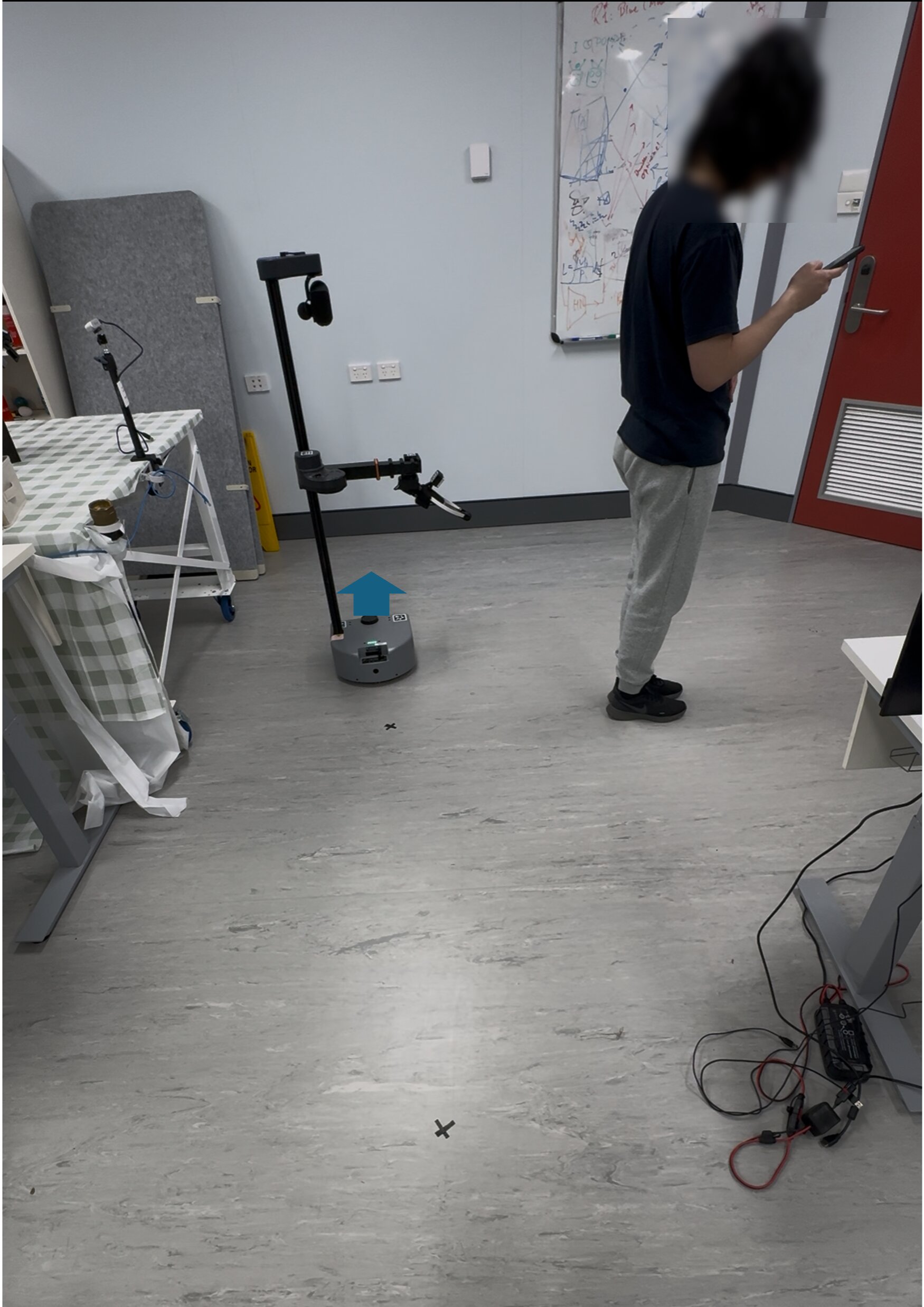}
         %\caption{Drone Capture}
         %\label{fig:visual capture}
     \end{subfigure}
     
    \caption{Stretch demonstrations using \nop. It smartly navigates around the moving pedestrian and quickly reaches the goal without waiting for the pedestrian or colliding with the environment.}
    \label{fig: stretch ropras3}
\end{figure*}

\nop is also robust to performance degradation from the underlying reference policies. We ran 4 variants of \textbf{Random3D} with an increasing number of obstacles. The results in \tref{tab:maze3dresults} indicate that \nop's performance (especially with uniform random sampling heuristics) does not degrade much even when the reference policy's (\ie{} \rrtc) failure percentage increases due to increasingly narrower passages in the maze. This is because \nop uses the results of \rrtc only to provide a set of alternative sequences of actions to perform. As long as there is sufficient diversity of macro-actions (\ie{} sufficient support of the reference policy), the reference-based \pomdp planner can converge to a reasonable policy of the \pomdp problem fast. This is especially true for the uniform sampling heuristic, as it can provide better diversity which results in more collision free path than dynamic heuristic does and leads to better performances in very cluttered environments.

For the high-dimensional navigation planning problem---\textbf{Multi-Drones Tag}, \nop performs at least 4 times better than the rest of the methods as it is the only method that exhibits strategies where drones actively discover and spread out to surround the tag. Both B-\vamp and \pomcp cannot easily adapt to the uncertainties from deterministic planning. B-\vamp does not incorporate any uncertainty in the effects of actions and \pomcp fixes a set of macro-actions from the reference policy for each belief node based only on a \emph{single} sampled particle.
Hence, the expanded set could be insufficient to fully represent the support of the belief, which is a problem that can be further compounded if the reference policy keeps failing.

In \textbf{Ray-Detect} (\tref{tab: Results on Ray-Detect and Shelf-Move}), \nop is the only method that consistently realizes that the most robust way to reach the cylinder is to use the ray to observe obstacles on the path and swing around the obstacles to avoid getting trapped by cluttered obstacles. Although the belief tree from R-\pomcp can reach the same depth as \nop, the UCB action selection strategy used by R-\pomcp ignores the benefits brought by the sampling heuristics and converges more slowly to find the optimal action. Interestingly, we found out that R-\pomcp benefits more from a uniform sampling strategy than a more complex dynamic sampling strategy, as expanding all actions at once requires the underlying sampling strategy to have built-in action diversity, whereas \nop expands action as new particles are introduced to a node, tremendously benefiting from the dynamic sampling strategy.

\textbf{Shelf-Move} is the most sophisticated problem that combines the difficulties of previous problems. Besides requiring very long horizon ($H=3000$) to deliberately put obstacle cylinders away at carefully planned locations so that the target cylinder can be retrieved and placed at the target location, the scenario itself consists of a non-trivial amount of uncertainties and the clutter of the environment implies occasional failures of finding a deterministic path. Unlike previous problems with a single winning strategy, this problem has many different choices for grasping and placing the cylinders. \nop equipped with the dynamic sampling heuristic is the only one that demonstrates smart grasping and placing of these objects (see the 3rd row of Figure~\ref{fig: ROPRAS Manipulation Behavior Visualisations}). It often identifies that  removing two obstacles away at non-target locations will clear a sufficiently large path to retrieve the target cylinder at the back, and place it at the intended location. In case of mistakenly placing an obstacle at the target location, \nop can often find a way to arrange the obstacles correctly for the placement of the target cylinder. Other methods do not exhibit this kind of long horizon thinking and their successful runs were often lucky runs.

\subsection{Physical Robot Demonstrations}
We deploy \nop on a Hello-Robot Stretch 3 in an uncertain and dynamic environment. In a lab environment, a pedestrian whose position is not exactly known to Stretch moves across the lab with a roughly constant speed. Stretch is tasked to navigate to a goal placed on the other side of the pedestrian. The pedestrian's speed is set such that if going straight, Stretch would most likely run into the pedestrian. The state space has 16 dimensions, of which 13 comes from Stretch, and the remaining 3 are the pedestrian's location as we model the pedestrian as a sphere. %in motion planning. 
We discretize the action space into 45 primitive open loop twist controls that correspond to straight-line and curved motions of the mobile base. The mobile base of Stretch is controlled by twist controllers. We use a simple Euler integrator as the transition function and add noises to it to compensate for the errors between the realized Stretch motion and integrated solution. Stretch can query its IMU sensors to provide actual observations for belief updates. Such an update serves as a filtering step to better localize Stretch and reduce odometry uncertainties. The problem horizon is set to 75, with the planning horizon set to 25. A reward of 800 is provided to the agent for reaching the goal. A -800 penalty is provided for colliding with the pedestrian. A {-20} penalty is given for being too close (within 0.3m) to the pedestrian. Otherwise, a -1 penalty is given for each primitive step taken. We compare \nop with B-\vamp and R-\pomcp in a few trials. As shown in \fref{fig: stretch ropras3} and \fref{fig: stretch-bvamp-rpomcp}, \nop is the only method that demonstrates a consistent smart and robust strategy of taking an efficient detour to go behind the moving pedestrian without colliding with other parts of the environment. B-\vamp in ~\fref{fig: stretch-bvamp-rpomcp} steers straight and bumps the pedestrian as it does not incorporate \pomdp planning, whereas R-\pomcp took a long detour and ran into the tables. See Appendix \ref{sec: appendix stretch details} for more details on the Stretch implementations.

\subsection{Ablation Study}\label{sec: ablations}
\subsubsection{Explorative Reference Policy.}
One might be concerned that limiting state sampling to only hand-picked information states in the reference policy (see Section \ref{sec: sampling heuristics}) is too restrictive. Therefore, we also evaluate \nop when these states are sampled in an $\epsilon$-greedy fashion, where with probability $\epsilon$, the sampling heuristic samples from the entire state space, and it samples those information states with $(1-\epsilon)$ probability. When $\epsilon$ is 0, it corresponds to the unmodified sampling strategy, and when $\epsilon$ is 1, the hand-crafted sampling strategy is purely replaced with uniform random sampling of the state space. 

For this ablation study, we test the above sampling strategy on the Maze2D and Multi-Drone Tag scenarios. We also increase the planning time per step from 1s to 3s to allow proper belief space coverage due to the increase of sampling possibilities. For a fair comparison, we also ran the strong baseline, R-\pomcp, on the two scenarios for 3s planning time per step.

The results are displayed in \fref{fig: ablation}. As expected, when the reference policy gradually approaches uniform, \nop's performance decreases. The reference policy becomes farther away from the deterministic optimal policy with respect to the KL-divergence, as a result the found policy can only perturb the reference policy so much to collect more rewards. However, the performance drop is not linear with respect to $\epsilon$. For instance, in Maze2D, significant amount of drop is seen when $\epsilon$ is increased from 0.6 to 1.

Despite the decreasing performance of \nop, it generally still performs better than the strong baseline R-\pomcp (without any $\epsilon-$explorations when sampling actions). In Maze2D, \nop with pure explorative reference policy (ie., $\epsilon = 1$) still outperforms R-\pomcp, with 13\% success rate when the policy is generated using \nop and 0\% success rate when using R-\pomcp. In Multi-Drone scenario, with $\epsilon <= 0.8$, \nop outperforms R-\pomcp scenario, indicating the importance of our novel tree search and backup steps.

Finally, this ablation study indicates \nop is less sensitive to the sampling heuristics in Multi-Drone Tag, as its worst success rate is retained at 40\%. This is due to the fact that Maze2D is more geometrically complex than Multi-Drone Tag, that is, the chance of sampling the right states to form a robust path is much less in Maze2D than that in Multi-Drone Tag (there are many different ways to surround and capture the target).

\begin{figure}[!h]
    \centering
    \includegraphics[width=0.8\linewidth]{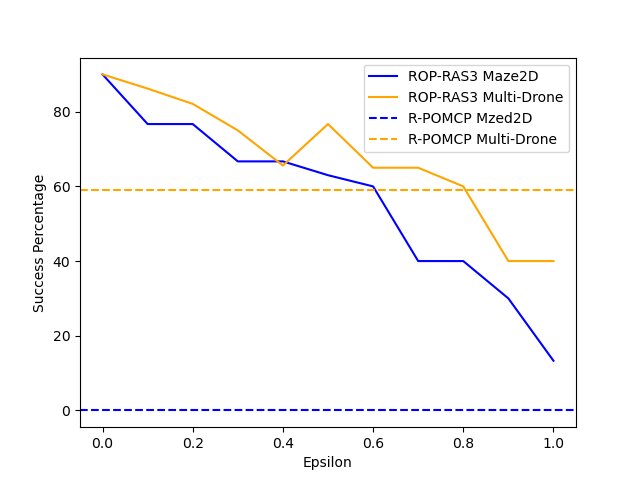}
    \caption{Ablation Study of \nop's performances with $\epsilon$-Exploration in Maze2D and Multi-Drone.}
    \label{fig: ablation}
\end{figure}

\subsubsection{Effect of Tree Search Depth.}
Hyperparameters such as the tree search depth are important to the performance of \nop. Under tight computational budgets (e.g., one second of planning), the right tree search depth needs to balance collecting long horizon information and Monte Carlo estimation accuracies. This ablation study perturbs the tree search depth across two experiments for \nop with the \textsc{Uniform} sampling heuristic. Results are shown in Figure \ref{fig: depth-ablation}. We see that performance drops as the tree depth increases or decreases; a good rule-of-thumb we found is to set the tree depth to be roughly the number of steps needed to solve the problem under deterministic settings.

\begin{figure}[!h]
    \centering
    \includegraphics[width=0.8\linewidth]{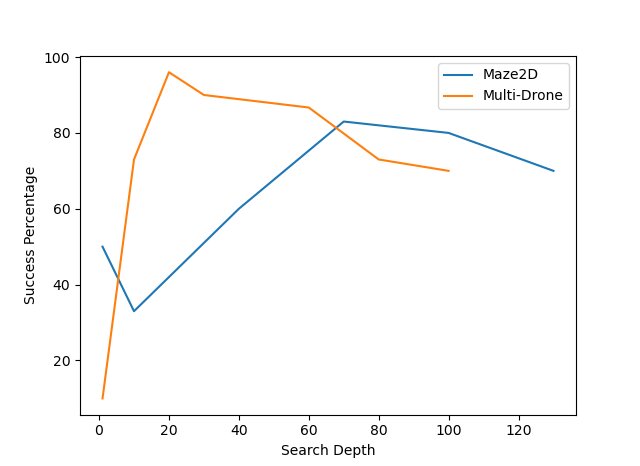}
    \caption{Ablation study of \nop's performances with different tree search depths in Maze2D and Multi-Drone.}
    \label{fig: depth-ablation}
\end{figure}

%% file: sections/discussion.tex
This paper presents a continuous reference-based \pomdp framework to handle large scale robotic problems. The objective of the Reference-based \pomdp can be partially solved analytically, resulting in a backup equation that can be estimated with sampling without online numerical optimizations. We then present a new online approximate Reference-based \pomdp solver, \nopLong (\nop), which uses \vamp to sample the state space and rapidly generate a large number of macro-actions online. These macro-actions reduce the effective planning horizon required and are adaptive to geometric features of the robot's free space. Since Reference-based \pomdp planners use macro-actions only to bias belief-space sampling and do not require exhaustive enumeration of macro-actions, \nop can efficiently exploit many diverse macro-actions to compute good \pomdp policies fast. It is shown that the number of belief particles and actions sampled controls the convergence of \nop{}s online planning. Evaluations on various long horizon \pomdps indicate that \nop outperforms state-of-the-art methods by multiple factors. Avenues of future works abounds. Hinted in the ablation study, a poorly selected reference policy could lead to inferior performance, and one may ask what are the characteristics of a reference policy that can guarantee a certain level of performance. Overall, the substantial increase in the capability of online approximate \pomdp solvers in long horizon problems brought by \nop enables improved robustness in wide ranges of robotics applications.

%% file: sections/appendix_theory.tex
\subsubsection{Analytical Solution}\label{sec: appendix analytical solution}
\begin{theorem}\label{thm: analytical solution copy}[Analytical Solution of Referenced-based \pomdp{}].
The exact solution of \eqref{eq.ref.bellman} is given by,
\begin{equation}
    \refVal(\bel) = \frac{1}{\temp} \log \Big[\int_\Actions \refPol(\act \, | \, \bel) \exp \Big\{ \temp \refQVal(\bel, \act) \Big\} \textup{d}\act\Big].
\end{equation}
Moreover, the \emph{exact} solution of the Reference-Based \pomdp is given by,
\begin{equation}
\optRefPol(\act \, | \, \bel)
\propto \refPol(\act \mid \bel)\exp(\temp \refQVal(\bel, \act))
\end{equation}
\end{theorem}

\begin{proof}
     At a given belief $\bel$, with the constraint that the probability needs to sum to one and positive everywhere, we seek to maximize the following objective,
    \begin{multline}\label{eq: lagrangian}
        \lagrangian_\bel(\pol, \lambda, \alpha) = \int_\Actions \pi(\act\mid\bel)\refQVal(\bel, \act)\;\textup{d}\act - \frac{1}{\eta}\KL(\pol\|\refPol)  \\ + \lambda\Big(1-\int_\Actions\pol(\act\mid\bel)\;\textup{d}\act\Big) - \int_\Actions \alpha(\act)\pol(\cdot\mid \bel)\;\textup{d}\act
    \end{multline}
    Since $\pol > 0$, the objective is maximized when $\alpha(a) = 0$ for all $\act\in\Actions$. The first order condition applied to the objective implies that $\almosteverywhere \forall \act\in\Actions$,
    \begin{equation}
        \refQVal(\bel, \act) - \frac{1}{\temp}(\log \pol(\act\mid\bel) - \log \refPol(\act\mid\bel) + 1) - \lambda = 0
    \end{equation}
    Rearrange we obtain that,
    \begin{equation}\label{eq: pol raw sol}
        \pol(\act\mid\bel) = \refPol(\act\mid\bel)\exp(\temp\refQVal(\bel, \act) - \temp\lambda - 1)
    \end{equation}
    Impose the normalization condition on $\pol$ and solve for $\lambda$,
    \begin{equation}\label{eq: lambda sol}
        \temp\lambda + 1 = \log \int_\Actions \refPol(\act\mid \bel)\exp(\temp \refQVal(\bel, \act))\;\textup{d}\act
    \end{equation}
    Substitute \eqref{eq: lambda sol} back to \eqref{eq: pol raw sol}, we obtain
    \begin{equation}
        \optRefPol(\act \mid \bel) = \frac{\refPol(\act \mid \bel)\exp(\temp \refQVal(\bel, \act))}{\int_\Actions \refPol(\act \mid \bel)\exp(\temp \refQVal(\bel, \act))\;\textup{d}\Tilde{a}}
    \end{equation}
    Substitute the optimal policy back to \eqref{eq: lagrangian}, we have the desired result,
    \begin{equation}
    \refVal(\bel) = \frac{1}{\temp} \log \Big[\int_\Actions \refPol(\act \, | \, \bel) \exp \Big\{ \temp \refQVal(\bel, \act) \Big\}\;\textup{d}\act\Big].
    \end{equation}
    \qed
\end{proof}

\subsubsection{Convergence Analysis}\label{sec: appendix convergence} 
We propose \rbss (Algorithm~\ref{alg: rbss}) as a theoretical version of \nop to analyze its online planning convergence. \rbss aims to use action weights and observation weights to compute $\hat{\refVal}_d$ and $\hat{\refQVal}_d$, the estimators for the referenced-based $\refVal$ and $\refQVal$ values at depth $d \in \{0, \dots, D-1\}$ of the tree. Given an initial weighted particle set $\particleBeliefs_0 = \{(\st_i, \omega_i)\}_{i=1}^{\particlesNum}$, where states are drawn from the initial belief $b_0$ and weights $\omega = 1/\particlesNum$, {\rbss} proceeds to compute the reference-based $\refQVal$-value after sampling $\actionsNum$ many actions from the reference policy. In \textsc{Estimate}-$\refQVal$, {\rbss} has access to a generative model $\GenModel$ to simulate the state transitions, observations and rewards based on the input $\act$. The particle weights can be updated by using the observation weights. The current reward and the next belief is used to compute the $\refQVal$ by recursively calling for the \textsc{Estimate}-$\refVal$ to compute the next step expected reward. Similar to \textsc{Estimate}-$\refQVal$, the action weight $\nu_i$ is used as SN weights for estimating the $\refVal$ expectations, which again depends on the $\refQVal$-values for the newly sampled actions. The recursion terminates at the specified tree depth $D$. At tree depth 0, the estimated $\hat{\refQVal}_0$ values can be used to approximate the optimal reference policy denoted as $\hat{\pol}$ according to \eqref{eq.rbpomdp.opt.pol}. The following assumptions are needed for the analysis of {\rbss}, 
\begin{enumerate}
     \item $\States$, $\Actions$ and $\Observations$ are continuous spaces.
     \item The R\'enyi divergence of any target distribution and sampling distribution is bounded above by $d^{\max}_\infty < \infty$. 
     \item The reward function is Borel and maps to $[0,1]$.
     \item We have access to a black box generator $\GenModel$ to sample next states, rewards and observations given the current state and action. We can also evaluate the observation probability $\OP$.
     \item The execution horizon terminates after $D < \infty$ steps.
\end{enumerate}
We relax the discrete action space assumption from \cite{Lim20:IJCAI} to continuous spaces in assumption 1. The second assumption is needed to invoke Theorem \ref{thm: SN concentration}. The third assumption is standard, see \cite{jin2020provably, du2021bilinear, uehara2021representation}, and it implies that the maximum $\refVal$ value is $\refVal_{\max} = 1/(1-\gamma)$. And the last two assumptions are common for online \pomdp planners. We emphasize that these assumptions are for theoretical analysis only and are not needed in the experiments. The following theorem shows that the number of particles in a tree node and the number of actions we sample control the convergence of \rbss. Note  that the convergence is with respect to the reference-based \pomdp objective \eqref{eq.maximized}, not the \pomdp objective.

\begin{theorem}\label{thm: q value estimate bound copy}[Accuracy of \rbss $\refQVal$-Value Estimate].
    Under the aforementioned setup. For any $\epsilon > 0$, choosing constants $\actionsNum$, $\particlesNum$, $\lambda$, $\delta$ that satisfies, 
    $$\lambda = \epsilon(1-\discount)^2/5,$$
    $$\delta = \lambda/(\refVal_{\max}D(1-\gamma)^2),$$
    $$\delta \geq 3\actionsNum(3\actionsNum\particlesNum)^D\exp(-\min\{\actionsNum, \particlesNum\}t^2_{\max}),$$
    \begin{equation}
        t_{\max} = \frac{\lambda}{3\refVal_{\max}d^{\max}_\infty} - \frac{1}{\sqrt{\min\{\actionsNum, \particlesNum}\}} > 0,
    \end{equation}
    Then, the $\refQVal$-values estimates obtained for all depth D and sampled actions $a$ are near-optimal with probability at least $1-\delta$,
    \begin{equation}
        \big|\refQVal^*_d(\bel_d, \act) - \hat{\refQVal}^*_d(\particleBeliefs_d, a) \big| \leq \frac{2\lambda}{1-\discount}.
    \end{equation}
\end{theorem}

We begin by stating the notation and lemmas used for the proof of the theorem. \\
\textbf{Notation}. We denote the transition density of state sequence $i$ from the root node to depth $d$ as,
\begin{equation}
    \TP^i_{1:d} \equiv \prod^d_{n=1} \TP(\st_{n, i}\mid \st_{n-1, i}, \act_n)
\end{equation}
And the observation density of state sequence $i$, observation sequence $j$ from the root node to depth $d$ as,
\begin{equation}
    \OP^{i,j}_{1:d} \equiv \prod^d_{n=1}\OP(\obs_{n, j} \mid \act_n, \st_{n, i})
\end{equation}
Any absence of indices $i, j$ means that $\{\st_n\}$ or $\{\obs_n\}$ appear as regular variables. When both densities appear together, we use $(\TP\OP)_{1:d}$ to denote $ \TP_{1:d}\OP_{1:d}$. We use $\bel^i_d$ to denote $\bel_d(\st_{d,i})$, $r_{d,i}$ as the reward $\Reward(\st_{d, i}, \act_{d})$ and $\weight_{d,i}$ as the weight of $\st_{d,i}$. The following lemma bounds the leaf node $\refQVal$-value estimations,
\begin{lemma}\label{lemma: depth D-1 Q value}[Depth $D-1$ $\refQVal$-value convergence.]
    For any sampled action $a$ at depth $D-1$, $\hat{\refQVal}_{D-1}^*(\particleBeliefs_{D-1}, a)$ is an SN estimator of $\refQVal^*_D(\bel_D, \act)$. And the following holds with probability at least $1-3\exp(-\particlesNum t^2_{\max})$,
    \begin{equation}
        \big| \refQVal^*_{D-1}(\bel_{D-1}, \act) - \hat{\refQVal}_{D-1}^*(\particleBeliefs_{D-1}, \act)\big| \leq \lambda
    \end{equation}
\end{lemma}

\begin{proof}
    By our finite horizon assumption, at depth $D-1$, the $\refQVal$-value is simply the expectation of final reward,
    \begin{equation}
        \refQVal^*_{D-1}(\bel_{D-1}, \act) = \int_\States \Reward(\st_{D-1}, \act)\bel_{D-1}\;\textup{d}\st_{D-1}
    \end{equation}
    which is equivalent to the \pomdp leaf node case, the rest of the proof follows from Lemma 1 of \cite{Lim20:IJCAI}.\qed
\end{proof}

The standard \pomdp backup implies maximizations are required to compute the $V$-values. Since this step has been done analytically for our reference-based \pomdp backup \eqref{eq.maximized}, the $\refVal$-value becomes another round of expectation estimation, which leads to the key contribution of the proof,
\begin{lemma}\label{lemma: depth D-1 V value}[Depth $D-1$ $\refVal$-value convergence].
    At depth $D-1$, $\hat{\refVal}_{D-1}^*(\particleBeliefs_{D-1})$ is an SN estimator of $\refVal^*_{D-1}(\bel_{D-1})$. Conditioned on the event in Lemma \ref{lemma: depth D-1 Q value}, the following holds with probability $1-3\actionsNum\particlesNum\exp(-\min\{\actionsNum, \particlesNum\}t^2_{\max})$,
    \begin{equation}
        \big| \refVal^*_{D-1}(\bel_{D-1}) - \hat{\refVal}_{D-1}^*(\particleBeliefs_{D-1})\big| \leq 2\lambda
    \end{equation}
\end{lemma}

\begin{proof}
    At depth $D-1$, consider the absolute difference between the $\refVal$-value and its estimation,
    \begin{align}
        & \big|\refVal^*_{D-1}(\bel_{D-1}) - \hat{\refVal}_{D-1}^*(\particleBeliefs_{D-1})\big| \\
        & \leq \big|\Exp_{\refPol(\cdot\mid\bel)} \big [\exp(\refQVal^*_{D-1}(\bel_{D-1}, \act))\big] \\ & - \frac{\sum_{k=1}^{\actionsNum} \nu_k \exp(\hat{\refQVal}_{D-1}^*(\particleBeliefs_{D-1}, \act_k))}{\sum_{k=1}^{\actionsNum} \nu_k} \big| \equiv (A)
    \end{align}
    where we have used the definition of $\refVal$ in \eqref{eq.maximized}, the return of \textsc{Estimate-}$\refVal$ in Algorithm \ref{alg: rbss}, and basic logarithmic properties for the inequality. We then use triangle inequalities to split $(A)$ into two terms,
    \begin{align}
        & (A) \leq \Big|\Exp_{\refPol(\cdot\mid\bel)} \big [\exp(\refQVal^*_{D-1}(\bel_{D-1}, \act))\big] \\ & - \frac{\sum_{k=1}^{\actionsNum} \nu_k \exp(\refQVal_{D-1}^*(\particleBeliefs_{D-1}, \act_k))}{\sum_{k=1}^{\actionsNum} \nu_k} \Big| \\ & + \Big | \frac{\sum_{k=1}^{\actionsNum} \nu_k \exp(\refQVal_{D-1}^*(\particleBeliefs_{D-1}, \act_k))}{\sum_{k=1}^{\actionsNum} \nu_k} \\ & - \frac{\sum_{k=1}^{\actionsNum} \nu_k \exp(\hat{\refQVal}_{D-1}^*(\particleBeliefs_{D-1}, \act_k))}{\sum_{k=1}^{\actionsNum} \nu_k} \Big | \equiv (B) + (C)
    \end{align}
    The first difference term, $(B)$, can be seen as importance sampling error and the second term, $(C)$, can be seen as the function estimation error. Define,
    \begin{equation}
        \Theta_{D-1}(\bel_{D-1}) = \Exp_{\refPol(\cdot\mid\bel)} \big [\exp(\refQVal^*_{D-1}(\bel_{D-1}, \act))\big]
    \end{equation}
    \begin{equation}
        \Lambda_{D-1}(\particleBeliefs_{D-1}) = \frac{\sum_{k=1}^{\actionsNum} \nu_k \exp(\refQVal_{D-1}^*(\particleBeliefs_{D-1}, \act_k))}{\sum_{k=1}^{\actionsNum} \nu_k}
    \end{equation}
    To bound $(B)$, we firstly show that the $\Lambda_{D-1}(\particleBeliefs)$ is indeed the SN estimator of $\Theta_{D-1}(\bel)$.  By following the recursive belief update \eqref{eq: belief expand}, the belief term can be expanded as,
    \begin{equation}\label{eq: belief expand}
        \bel_{D-1}(\st_{D-1}) = \frac{\int_{\States^{D-1}} (\TP\OP)_{1:D-1}\bel_0\;\textup{d}\st_{0:D-2}}{\int_{\States^{D}} (\TP\OP)_{1:D-1}\bel_0\;\textup{d}\st_{0:D-1}}
    \end{equation}
    
    Expanding $\Theta_{D-1}(\bel)$ using \eqref{eq: belief expand} and our structural assumption on $\refPol$ in \eqref{eq: ref pol struct}, 
    {\small
    \begin{multline}
        \Theta(\bel_{D-1}) = \int_\Actions \refPol(\act \mid \bel_{D-1})\exp(\refQVal(\bel_{D-1}, \act))\;\textup{d}\act \\ 
        = \frac{\int_\Actions\int_{\States^{D-1}}\exp(\refQVal) \delta_{\fullyObsPol(\st_{D-1})}(\act) (\TP\OP)_{1:D-1}\bel_0\;\textup{d}\st_{0:D-2}\;\textup{d}\act}{\int_{\States^{D}}(\act)(\TP\OP)_{1:D-1}\bel_0\;\textup{d}\st_{0:D-1}\;\textup{d}\act}
    \end{multline}}

    The density we aim to estimate is $\refPol(\cdot\mid \bel_{D-1})$, which is a pushforward measure of $\bel_{D-1}$, hence if we have an estimate of $\bel_{D-1}$, we can estimate $\refPol$ easily. The $\bel_{D-1}$ estimation is done in \textsc{Estimate}-$\refQVal$ of Algorithm \ref{alg: rbss}. Importantly, recall from \cite{Lim20:IJCAI}, let $\mathcal{P}^{D-1}$ denote the normalized measure incorporating the observation sequence $j$ on top of the state sequence $i$ and $\mathcal{F}^{D-1}$ is the probability of the state sequence. 
    \begin{align}
        \mathcal{P}^{D-1} &= \mathcal{P}^{D-1}_{\{\obs_n\}_j}(\{\st_n\}_i) = \frac{\TP^i_{1:D-1}\OP^{i,j}_{1:D-1}\bel_0^i}{\int_{\States^{D}}\TP_{1:D-1}\OP^{j}_{1:D-1}\bel_0\;\textup{d}\st_{0:D-1}} \\
        \mathcal{F}^{D-1} & = \mathcal{Q}^{D-1}(\{s_n\}_i) = \TP^i_{1:d}\bel^i_0
    \end{align}
    Then, the belief importance weight is given by,
    \begin{equation}\label{eq: w importance weight}
        \weight_{\mathcal{P}^{D-1}/\mathcal{F}^{D-1}}(\{s_n\}_i) = \frac{\OP^{i,j}_{1:D-1}}{\int_{\States^{D-1}}\TP_{1:D-1}\OP^{j}_{1:D-1}\bel_0\;\textup{d}\st_{0:D-1}}
    \end{equation}
    The action importance weight is given by,
    \begin{align}\label{eq: nu importance weight}
        \nu_{\mathcal{P}^{D-1}/\mathcal{F}^{D-1}}(\act_k) = \sum_{i=1}^{\particlesNum} \mathbb{I}_{\fullyObsPol(\st_{n},i)}(\act_k)\weight_{\mathcal{P}^{D-1}/\mathcal{F}^{D-1}}(\{s_{n}\}_i)
    \end{align}
    
    Fixing an observation sequence $\{\obs_j\}$, the weight of the actions sampled at depth $D-1$ is given by the observation weight reweighted by the fully observable policy (see Algorithm \ref{alg: rbss}), 
    \begin{align}
        \nu_{D-1,k} &= \sum_{i = 1}^{\particlesNum} \mathbb{I}_{\fullyObsPol(\st_{D-1}, i)}(\act_k)\omega_{D-1, i} 
        \\ &\propto \sum_{i = 1}^{\particlesNum} \mathbb{I}_{\fullyObsPol(\st_{D-1}, i)}(\act_k) \OP_{1:D-1}^{i,j}\label{eq: nu expand}
    \end{align}
    The last equation come from the recursive update of the belief weights in Algorithm \ref{alg: rbss}.
    The estimator $\Lambda_{D-1}(\particleBeliefs)$ can be re-written as,
    \begin{multline}
        \Lambda_{D-1}(\particleBeliefs) = \\ \frac{\sum_{k=1}^{\actionsNum}\sum_{i = 1}^{\particlesNum} \mathbb{I}_{\fullyObsPol(\st_{D-1}, i)}(\act_k) \OP_{1:D-1}^{i,j} \exp(\eta\hat{\refQVal}_{D-1}(\particleBeliefs, \act_k))}{\sum_{k=1}^{\actionsNum}\sum_{i = 1}^{\particlesNum} \mathbb{I}_{\fullyObsPol(\st_{D-1}, i)}(\act_k) \OP_{1:D-1}^{i,j}} 
    \end{multline}
    Using \eqref{eq: w importance weight}, \eqref{eq: nu importance weight} and \eqref{eq: nu expand}, we get
    \begin{align}
        \Lambda_{D-1}(\particleBeliefs) &= \frac{\sum_{k=1}^{\actionsNum} \nu_{\mathcal{P}^{D-1}/\mathcal{F}^{D-1}}(\act_k) \exp(\eta\hat{\refQVal}_{D-1}(\particleBeliefs_{D-1}, \act_k))}{\sum_{k=1}^{\actionsNum}\nu_{\mathcal{P}^{D-1}/\mathcal{F}^{D-1}}(\act_k)}
        \\ &= \sum_{k=1}^{\actionsNum} \tilde{\nu}_{\mathcal{P}^{D-1}/\mathcal{F}^{D-1}}\exp(\eta\hat{\refQVal}_d(\particleBeliefs_{D-1}, \act_k))
    \end{align}

    Our assumption that $\Reward$ is a bounded Borel function implies that $\exp(\refQVal_{D-1})$ is also a bounded Borel function. Since the states are i.i.d sequences drawn from $\mathcal{F}^{D-1}$, we can apply the SN concentration bound from Theorem \ref{thm: SN concentration} to bound the importance sampling error by $\lambda$ with probability at least $1-3\exp(-\actionsNum t^2_{\max})$.

    Next the $\refQVal$ function estimation error $(C)$ is bounded by our conditioning on Lemma \ref{lemma: depth D-1 Q value}. With a union bound on the two events, we obtain the final result. \qed
\end{proof}

Recursively, we can then bound the $\refQVal$-value estimate at arbitrary depth $d$. 
\begin{lemma}\label{lemma: depth d Q value} [Any depth $d$ $\refQVal$-value convergence.]
    At any depth $d = \{0, \dots, D-1\}$, $\hat{\refQVal}^*_d(\particleBeliefs_d, \act)$ is an SN estimator of $\refQVal^*_d(\bel_d, \act)$ for any action $\act\in\Actions$. And the following holds with probability at least $1-3\actionsNum(3\actionsNum\particlesNum)^D\exp(-\min\{\actionsNum, \particlesNum\}t^2_{\max})$,
    \begin{equation} \label{eq: depth d bound}
        \Big| \refQVal^*_d(\bel_d, \act) - \hat{\refQVal}^*_d(\particleBeliefs, \act)\Big| \leq \alpha_d
    \end{equation}
    \begin{equation}\label{eq: depth d gaps}
        \alpha_d \equiv 2\lambda + \gamma\alpha_{d+1}; \alpha_{D-1} = \lambda
    \end{equation}
\end{lemma}
\begin{proof}
     We pick $\actionsNum$ and $\particlesNum$ such that $\min\{\actionsNum, \particlesNum\} > (3\refVal_{\max} d^{\max}_\infty / \lambda)^2$ to satisfy $t_{\max} > 0$ which ensures the probability holds at any depth $d$ and $\act$. The $\refQVal$-value estimate defined in \textsc{Estimate}-$\refQVal$ is given as,
     \begin{equation}
         \hat{\refQVal}^*_d(\particleBeliefs_d, \act) = \frac{\sum_{i=1}^{\particlesNum} \weight_{d,i}\Big(r_{d,i}  + \gamma \hat{\refVal}^*_{d+1}(\overline{\bel_d\act\obs_i})\Big)}{\sum_{i=1}^{\particlesNum} \weight_{d,i}}
     \end{equation}
     We are interested in bounding the gap between the estimator and the actual $\refQVal^*$. At the leaf node, the gap is bounded by $\lambda$ given by Lemma \ref{lemma: depth D-1 Q value}. Hence $\alpha_{D-1} = \lambda$. Next, the $\refQVal$ estimation error for any depth $d$ can be separated into two components using the triangle inequality,
     \begin{multline}
         \Big|\refQVal^*_d(\bel_d, \act) - \hat{\refQVal}^*_d(\particleBeliefs_d, a)\Big| \\
         \leq \underbrace{\Big |\Exp\big[\Reward(\st_d, \act) \mid \bel_d\big] - \frac{\sum^{\particlesNum}_{i=1}\weight_{d,i}r_{d,i}}{\sum^{\particlesNum}_{i=1}\weight_{d,i}}\Big |}_{(D)} \\
         + \discount\underbrace{\Big|\exp\big[\refVal^*_{d+1}(\bel\act\obs)\mid \bel_d\big] - \frac{\sum^{\particlesNum}_{i=1}\weight_{d,i}\hat{\refVal}^*_{d+1}(\overline{\bel_d\act\obs_i})}{\sum^{\particlesNum}_{i=1}\weight_{d,i}} \Big|}_{(E)}
     \end{multline}
     Since the form of $\refQVal$-value is the same as \pomdp $Q$-values, using results from \cite{Lim20:IJCAI}, we can bound $(D)$ above by $\frac{\Reward_{\max}}{3\refVal_{\max}}$. And by noticing that the proof of Lemma \ref{lemma: depth D-1 V value} can be recursively applied to all depths $d$ when the $\refQVal_d$ at the same depth is bounded by results from \cite{Lim20:IJCAI}, we can then bound $(E)$ from above by $\lambda + \frac{1}{3}\lambda + \frac{2}{3\discount}\lambda + \alpha_{d+1}$. Putting everything together, we have $(D)+(E)\leq 2\lambda + \gamma\alpha_{d+1}$, which is $\alpha_d$ stated in this Lemma. To compute the worst case probability accounting for all events happened in the past, we multiply the base probability by $(4\actionsNum\particlesNum)^D$, as we want the function estimations to be within their thresholds at all depths from $d$, where 4 is the number of times we applied the SN concentration bound at a given depth (1 time in $\refVal$ estimation and 3 times in $\refQVal$ estimations). We then multiply the factors by another $3\actionsNum$ to account for the root node $\refQVal$-value estimations. Therefore, equation \eqref{eq: depth d bound} holds at all depth $d=1,\dots, D-1$ with probability at least $1-3\actionsNum(4\actionsNum\particlesNum)^D\exp(-\min{\{\actionsNum, \particlesNum\}}t^2_{\max})$. \qed
\end{proof}

\begin{proof} (Theorem \ref{thm: q value estimate bound copy}). We pick constants $\lambda, \actionsNum, \particlesNum, \delta$ and densities $\TP, \OP, \fullyObsPol, \bel_0$ that satisfies the requirements in Theorem \ref{thm: q value estimate bound copy}. Then with probability $1-\delta$
through Lemmas \ref{lemma: depth D-1 Q value}, \ref{lemma: depth D-1 V value} and \ref{lemma: depth d Q value}, we have
\begin{equation}
    \Big|\refQVal^*_d(\bel_d, \act) - \hat{\refQVal}^*_{d}(\particleBeliefs_d, \act)\Big| \leq \alpha_0 \leq 2\sum_{d=0}^{D-1}\discount^d\lambda \leq \frac{2\lambda}{1-\discount}
\end{equation} 
\qed 
\end{proof}

%% file: sections/appendix_exp.tex
\subsection{Experiment Details}
\subsubsection{Simulation Details.}\label{sec: appendix simulation details}
The \pomdp definition of each benchmark scenarios is summarized into Table \ref{tab: pomdp benchmark space summary} and Table \ref{tab: pomdp benchmark func summary}. In \tref{tab: pomdp benchmark space summary}, we use $S_p$ to denote the joint space of the Franka Panda arm and $H$ to denote the maximum horizons we set for each problem. In \tref{tab: pomdp benchmark func summary}, the probability distributions are noises that i.i.d added to each dimension of the state or observations. In sphere search, $(0.5, 0.76, 0.4)$ and $(0.5, -0.76, 0.4)$ denotes the two possible target locations. For Light-Dark, Random3D, Ray-Detect and Shelf-Move, the initial belief is the true state with Gaussian noises. In Maze2D, the drone can spawn at either entry point, hence the initial belief is a bimodal Dirac delta distribution. In Sphere-Search, the initial belief includes both adding Gaussian noises to the initial manipulator configuration and a bimodal Dirac delta distribution for the possible goal locations. In Multi-Drone, the drone locations are fully observable at all times, but the target location is initially uniformly distributed across the entire workspace.

We also show critical time stamps of \nop in each experiment conducted, demonstrating its long horizon thinking and careful planning capabilities under uncertainty. See~\fref{fig: ROPRAS Light-Dark Behavior Visualisations} to ~\fref{fig: ROPRAS Manipulation Behavior Visualisations}.

\subsubsection{Implementation Details.}\label{sec: appendix implementation details}
When implementing \nop, it is important to keep a good balance between the number of simulations \nop computes and the amount of time it takes to compute. In complex scenarios, we found at least 30 to 50 simulations were needed to find good solutions. Since the number of simulations is low, we also do not keep too many belief particles, thus reducing belief update time (see \eqref{eq: belief expand}). However, small number of particles often lead to particle deprivation, hence particle reinvigoration was needed to inject a small amount of noise to the beliefs, this is done after the belief update. Another important parameters to tune is the search depth $D$. Although \nop can theoretically handle longer horizons, it is not desired to set $D$ to be too large as otherwise it compromises the number of simulations or the time it takes to complete one planning. We found a good heuristic is to set $D$ to be roughly the number of steps it takes to complete the task in the fully observable deterministic case. In all problem scenarios considered in this work, it is easy to define a good metric between observations (e.g. whether same objects observed are close to each other in terms of Euclidean distance), but harder to find a good metric for the actions (especially in manipulation tasks), we therefore only used action progressive widening in our implementation and replace observation widening with binning.

\subsubsection{Stretch Details.}\label{sec: appendix stretch details}
Since VAMP is a kinematic planner but the Stretch 3 has a non-holonomic base, we have to convert the VAMP path to a trajectory the Stretch 3 can follow. We do this within the reference policy by converting a path to a sequence of actions (a macro-action). Each action in the macro-action sequence is a twist (linear and angular velocity) for the base held over 0.7 seconds. We set the amount of time for each action to be held heuristically based on the task and environment. We can therefore send the result from the \nop planner (a sequence of twists) directly to the robot for open-loop control, alternating between planning and execution. To obtain the macro-action from a VAMP path, we choose the next action in our macro-action to be the action (from 45 possible) that minimizes some error from a lookahead point. The error in our experiments is a weighted sum of 1.0 for the distance to the lookahead point and 0.5 for the heading difference from following the straight line path. This cost was also set heuristically based on how well the trajectory tracked the path. The actions were also chosen heuristically based on obtaining a diversity of possible speeds for natural behavior. The actions consist of three linear speeds ($0.1m/s$, $0.2m/s$, $0.3m/s$) forward and backward along with angular velocities spaced by $0.1 rad/s$ between $-0.5 rad/s$ and $0.5 rad/s$. Note that the Stretch 3 controller for tracking linear and angular velocities follows trapezoidal motion profiles with a set (linear) acceleration of $0.12m/s^2$ for each wheel and a maximum absolute velocity of $0.3m/s$. Therefore, we must integrate over each action to determine how the robot transitions based on the current wheel velocities. Although values and models are chosen in a crude manner, part of the benefit of \pomdp planning is to take these uncertainties and errors into account and find a robust plan to execute.

\subsubsection{Connections to First Work.}
Improvements have been made to \nop compared to our first work \cite{Liang2024Scaling}. With our theoretical insights and double progressive widening techniques, \nop can now handle fully continuous \pomdps. Meanwhile, the original backup equation adjusts the estimations of the $\refVal$-values by subtracting away the old estimations for $\refQVal$-values and plug the new estimated $\refQVal$-values in. Our theory now indicates the backup can be done by simply applying Monte-Carlo estimators at each nested level. Further information on nested Monte-Carlo estimations can be found from \cite{rainforth2018nesting}.

\begin{table*}[!htbp]
\centering
\caption{Summary of \pomdp Spaces and Discounts of All Benchmark Scenarios}
\label{tab: pomdp benchmark space summary}
\resizebox{\textwidth}{!}{
\begin{tabular}{l|l|l|l|l|}
\cline{2-5}
                                             & State Space $\States$                                                                                   & Action Space $\Actions$                                                                                                                                                                                                                                         & Observation Space $\Observations$                                                                      & $\discount$ \\ \hline
\multicolumn{1}{|l|}{\textbf{Light-Dark}}    & $[-4, 4]^2$m                                                                                & \begin{tabular}[c]{@{}l@{}}$\{(0.5, 0), (-0.5, 0)$ \\ $(0, 0.5), (0, -0.5)\}$m\end{tabular}                                                                                                                                                           & $[-4, 4]^2 \text{m} \times\{\textup{None}\}$                                                  & 0.99        \\ \hline
\multicolumn{1}{|l|}{\textbf{Maze2D}}        & $[-25, 25]^2\text{m}$                                                                       & \begin{tabular}[c]{@{}l@{}}$\{(0.5, 0), (-0.5, 0)$ \\ $(0, 0.5), (0, -0.5)\}\text{m}$\end{tabular}                                                                                                                                                           & $[-25, 25]^2 \text{m} \times\{\textup{None}\}$                                                & 0.999       \\ \hline
\multicolumn{1}{|l|}{\textbf{Random3D}}      & $[-25, 25]^3\text{m}$                                                                       & \begin{tabular}[c]{@{}l@{}}$\{(0.5, 0, 0), (-0.5, 0, 0)$\\ $(0, 0.5, 0), (0, -0.5, 0), $\\ $ (0, 0, 0.5), (0, 0, -0,5)\}\text{m}$\end{tabular}                                                                                                                  & $[-25, 25]^3 \text{m} \times\{\textup{None}\}$                                                & 0.999       \\ \hline
\multicolumn{1}{|l|}{\textbf{Multi-Drone}}   & $([-15, 15]^2 \times [-2, 2] \text{m})^5$                                                   & \begin{tabular}[c]{@{}l@{}}Same as Random3D but\\ applied to all drones\end{tabular}                                                                                                                                                               & $([-15, 15]^2 \text{m} \times [-2, 2] \text{m}) \times \{\textup{None}\}$                     & 0.99        \\ \hline
\multicolumn{1}{|l|}{\textbf{Sphere-Search}} & \begin{tabular}[c]{@{}l@{}}$S_p^7\text{rad} \times $ \\ $\{(0.5, 0.76, 0.4), $\\$(0.5, -0.76, 0.4)\}$m\end{tabular}                             & $S_p^7\text{rad}$                                                                                                                                                                                                                                    & \begin{tabular}[c]{@{}l@{}}$S_p^7\text{rad} \times $\\$ \{(0.5, 0.76, 0.4), (0.5, -0.76, 0.4)\}\text{m}\times$\\$\{\textup{None}\}$ \end{tabular}& 0.99        \\ \hline
\multicolumn{1}{|l|}{\textbf{Ray-Detect}}    &\begin{tabular}[c]{@{}l@{}} $S_p^7\text{rad} \times $\\ $ ([-1, 1]^3 m \times S^3 \text{rad})^4$\end{tabular}                                & $S_p^7\text{rad}$                                                                                                                                                                                                                                    & \begin{tabular}[c]{@{}l@{}}$S_p^7\text{rad} \times[-1, 1]^3m \times S^3  \times \{\textup{None}\}$\end{tabular}                         & 0.999       \\ \hline
\multicolumn{1}{|l|}{\textbf{Shelf-Move}}    & $S_p^7\text{rad} \times ([-1, 1]^3m)^4$                                                       & $S_p^7\text{rad}$                                                                                                                                                                                                                                    & $S_p^7\text{rad} \times ([-1, 1]^3m)^4 \times \{\textup{None}\}$                                & 0.9999      \\ \hline
\end{tabular}}
\end{table*}

\begin{table*}[!htbp]
\centering
\caption{Summary of \pomdp Functions and Horizons of All Benchmarks}
\label{tab: pomdp benchmark func summary}
\resizebox{\textwidth}{!}{
\begin{tabular}{l|l|l|l|l|}
\cline{2-5}
                                             & $\TP$                                                                                       & $\OP$                                                                                                                                                                                                                                             & $\Reward$                                                                    & $H$  \\ \hline
\multicolumn{1}{|l|}{\textbf{Light-Dark}}    & Deterministic                                                                               & \begin{tabular}[c]{@{}l@{}}In Light: $\mathcal{N}(0, 0.1)$m\\ Else: None\end{tabular}                                                                                                                                                             & \begin{tabular}[c]{@{}l@{}}Step: -0.1\\ Goal: 100\end{tabular}               & 60   \\ \hline
\multicolumn{1}{|l|}{\textbf{Maze2D}}        & \begin{tabular}[c]{@{}l@{}}20\% execute a wrong\\ action\end{tabular}                       & \begin{tabular}[c]{@{}l@{}}In Light: $\mathcal{N}(0, 0.5)$m\\ Else: None\end{tabular}                                                                                                                                                             & \begin{tabular}[c]{@{}l@{}}Step: -0.1\\ Goal: 800\\ Fail: -2000\end{tabular} & 800  \\ \hline
\multicolumn{1}{|l|}{\textbf{Random3D}}      & \begin{tabular}[c]{@{}l@{}}20\% execute a wrong\\ action\end{tabular}                       & \begin{tabular}[c]{@{}l@{}}In Light: $\mathcal{N}(0, 0.5)$m\\ Else: None\end{tabular}                                                                                                                                                             & \begin{tabular}[c]{@{}l@{}}Step: -0.1\\ Goal: 800\\ Fail: -2000\end{tabular} & 800  \\ \hline
\multicolumn{1}{|l|}{\textbf{Multi-Drone}}   & \begin{tabular}[c]{@{}l@{}}Deterministic\\ Target can teleport\\ Drones cannot\end{tabular} & \begin{tabular}[c]{@{}l@{}}In range: $\mathcal{N}(0, 0.5)$m\\ Else: None\end{tabular}                                                                                                                                                             & \begin{tabular}[c]{@{}l@{}}Step: -0.1\\ Goal: 600\end{tabular}               & 400  \\ \hline
\multicolumn{1}{|l|}{\textbf{Sphere-Search}} & $\mathcal{N}(0, 0.02) \text{rad}$                                                           & \begin{tabular}[c]{@{}l@{}}Joints are observable with\\  $\mathcal{N}(0,0.02)$ rad noises.\\ \\ In light: target fully observable\\ Else: None.\end{tabular}                                                                                      & \begin{tabular}[c]{@{}l@{}}Step: -0.1\\ Goal: -800\\ Fail: 800\end{tabular}  & 150  \\ \hline
\multicolumn{1}{|l|}{\textbf{Ray-Detect}}    & $\mathcal{N}(0, 0.02) \text{rad}$                                                           & \begin{tabular}[c]{@{}l@{}}Joints are observable with\\  $\mathcal{N}(0,0.02)$ rad noises.\\ \\ If detect: \\ object position + $\mathcal{U}(-0.02, 0.02)$m\\ object orientation + $\mathcal{U}(-0.04, 0.04)\text{rad}$\\ Else: None\end{tabular} & \begin{tabular}[c]{@{}l@{}}Step: -0.2\\ Goal: 300\end{tabular}               & 800  \\ \hline
\multicolumn{1}{|l|}{\textbf{Shelf-Move}}    & $\mathcal{N}(0, 0.001) \text{rad}$                                                          & \begin{tabular}[c]{@{}l@{}}Joints are observable with\\  $\mathcal{N}(0,0.02)$ rad noises.\\ \\ If detect: object position with $\mathcal{U}(-0.04, 0.04)$ m\\ Else: None\end{tabular}                                                            & \begin{tabular}[c]{@{}l@{}}Step: -0.1\\ Goal: 800\end{tabular}               & 3000 \\ \hline
\end{tabular}}
\end{table*}

\begin{figure*}
    \centering
    \begin{subfigure}[b]{0.24\textwidth}
        \centering
        \includegraphics[width=\textwidth]{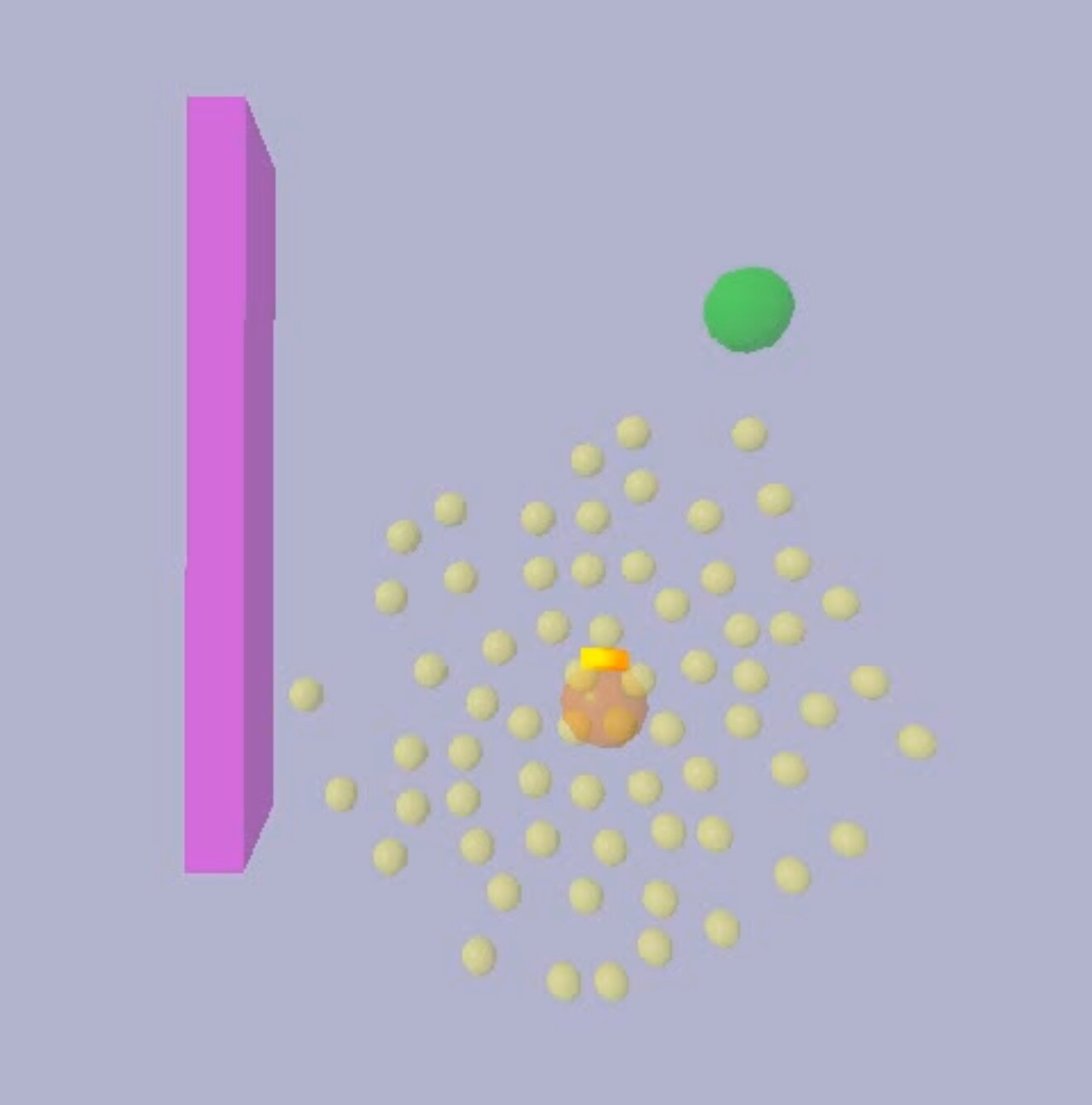}
    \end{subfigure}
    \begin{subfigure}[b]{0.24\textwidth}
        \centering
        \includegraphics[width=\textwidth]{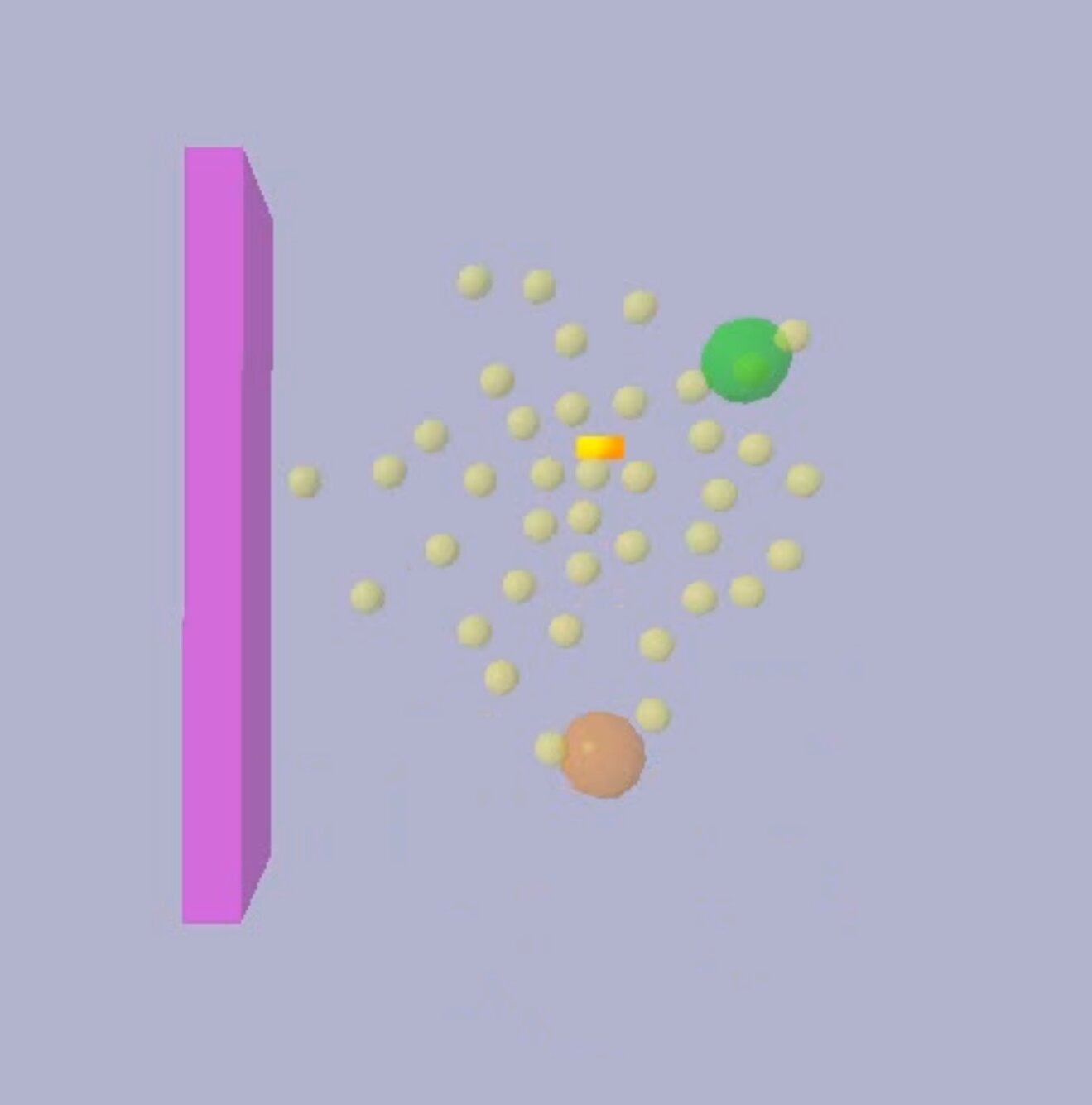}
    \end{subfigure}
    \begin{subfigure}[b]{0.24\textwidth}
        \centering
        \includegraphics[width=\textwidth]{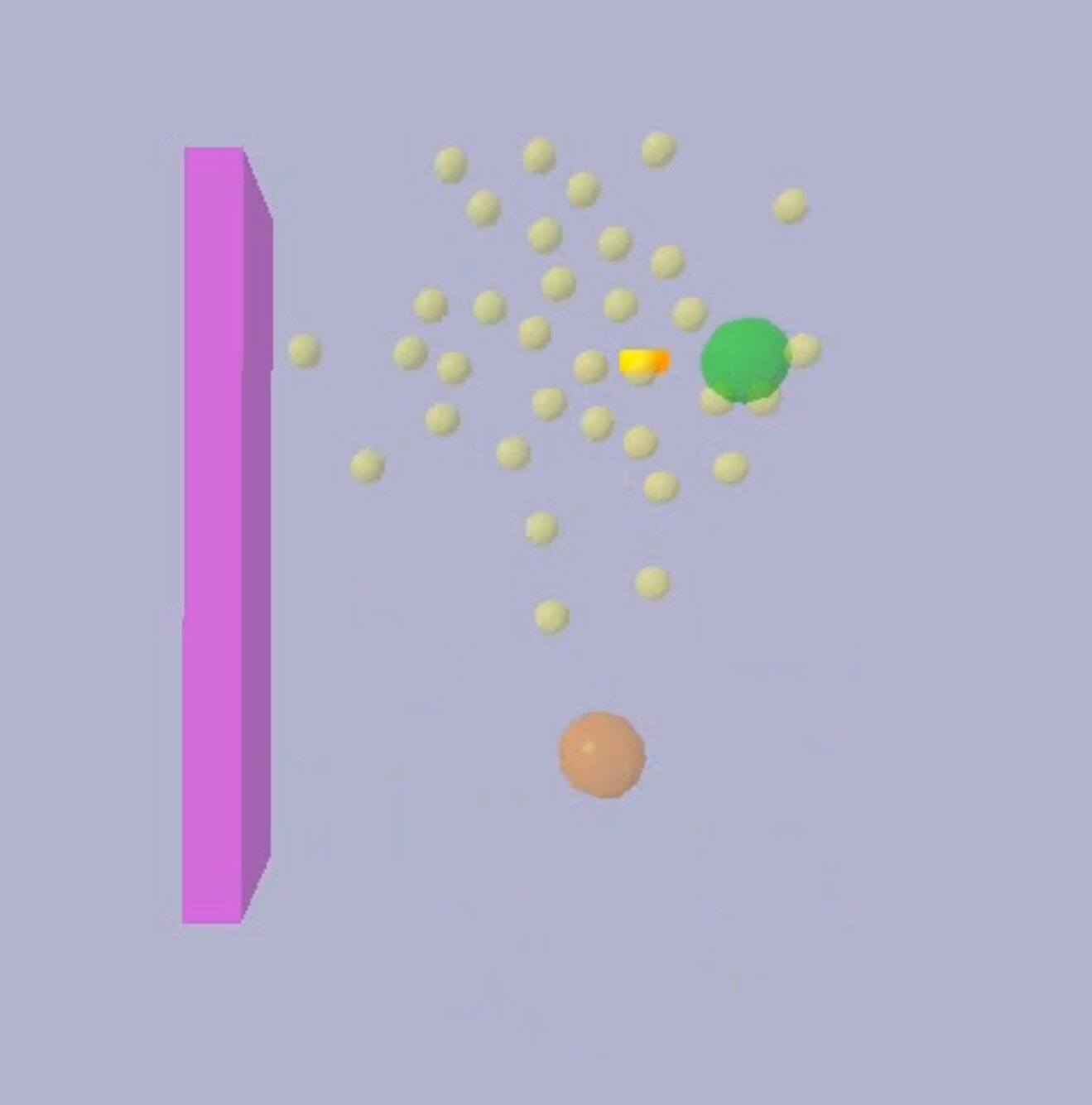}
    \end{subfigure}
    \begin{subfigure}[b]{0.24\textwidth}
        \centering
        \includegraphics[width=\textwidth]{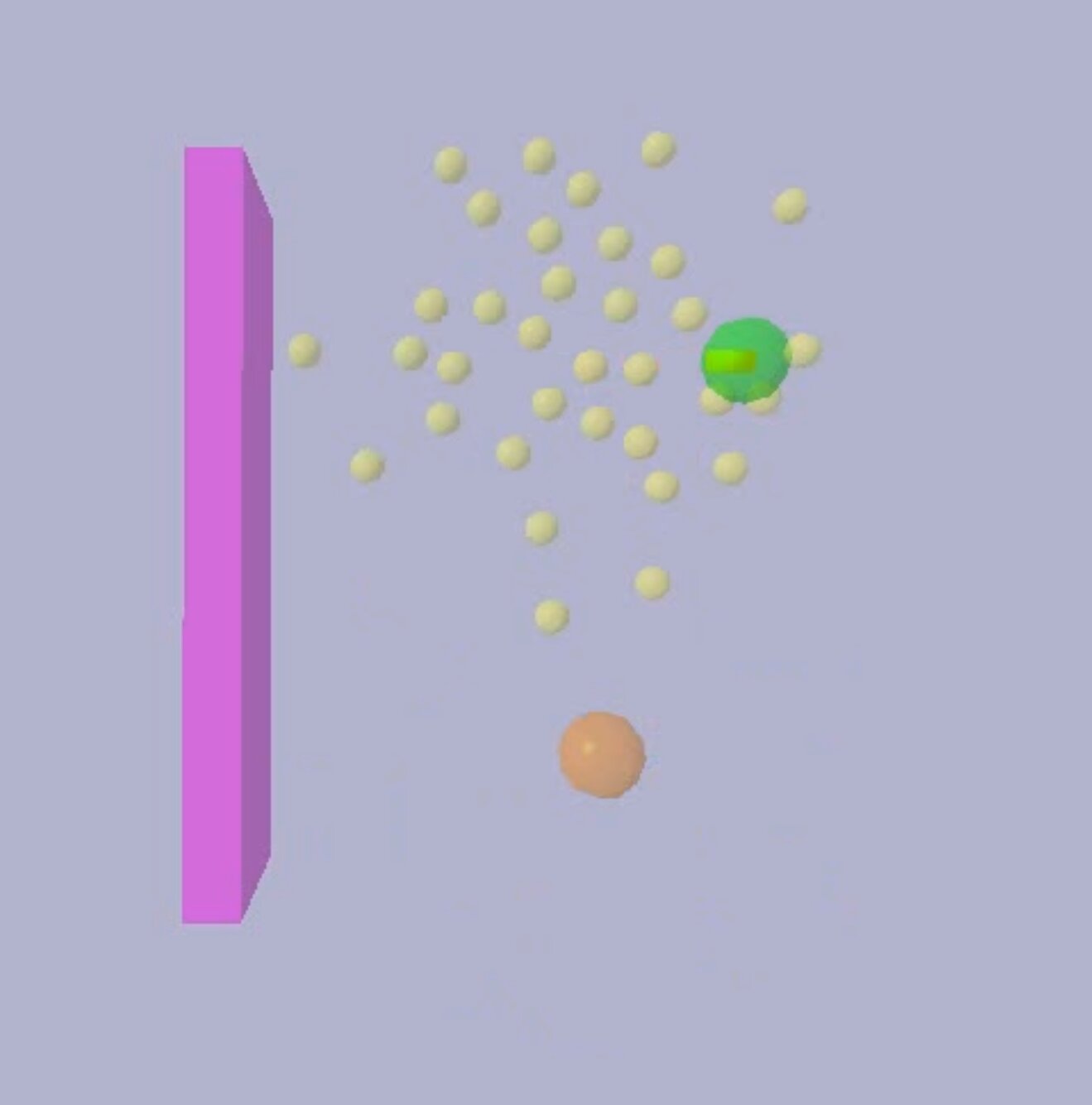}
    \end{subfigure}

    \caption{Light Dark is the easiest problem and the agent can quickly navigate to the goal. }
    \label{fig: ROPRAS Light-Dark Behavior Visualisations}
\end{figure*}

\begin{figure*}
    \centering
    \begin{subfigure}[b]{0.45\textwidth}
        \centering
        \includegraphics[width=\textwidth]{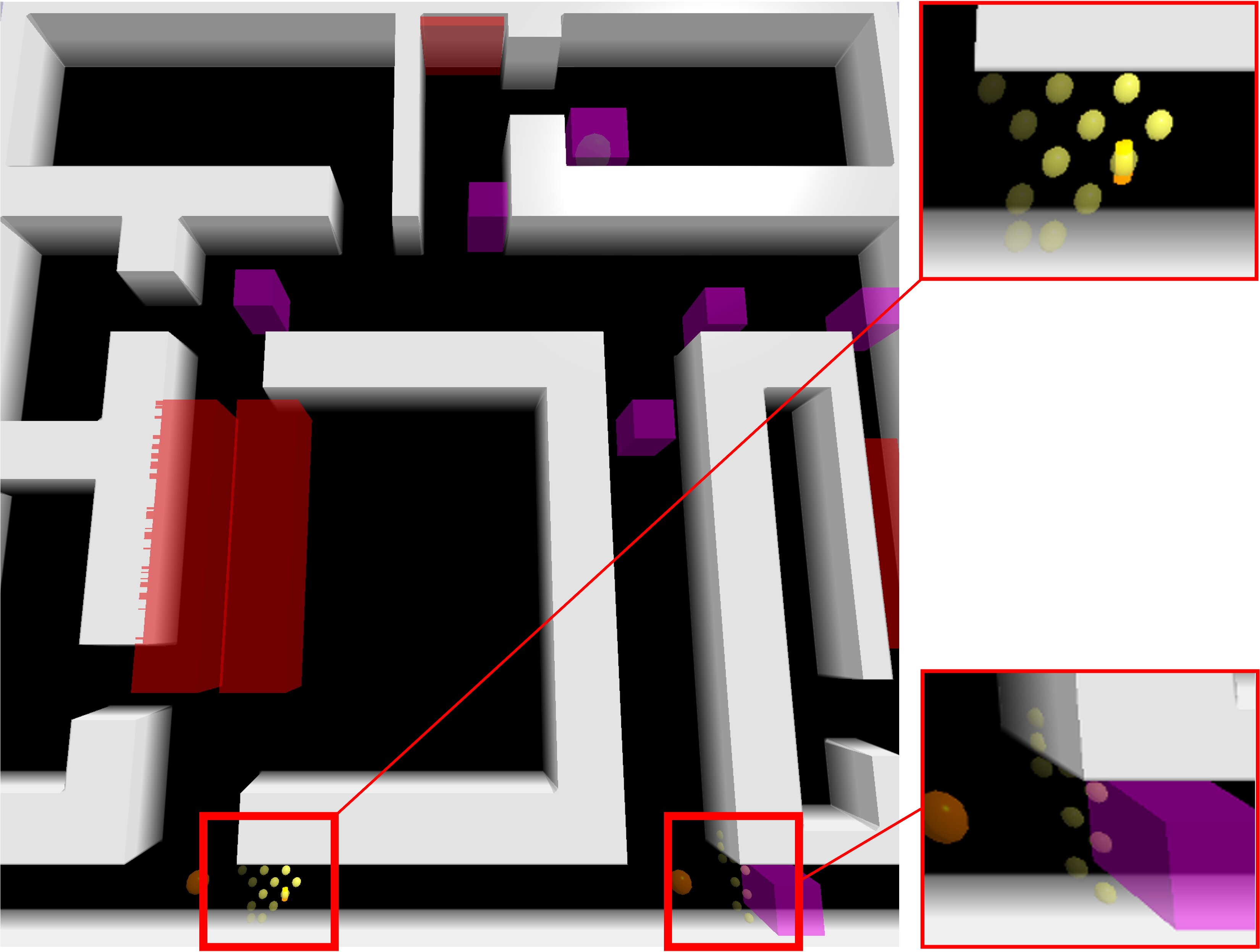}
    \end{subfigure}
    \begin{subfigure}[b]{0.45\textwidth}
        \centering
        \includegraphics[width=\textwidth]{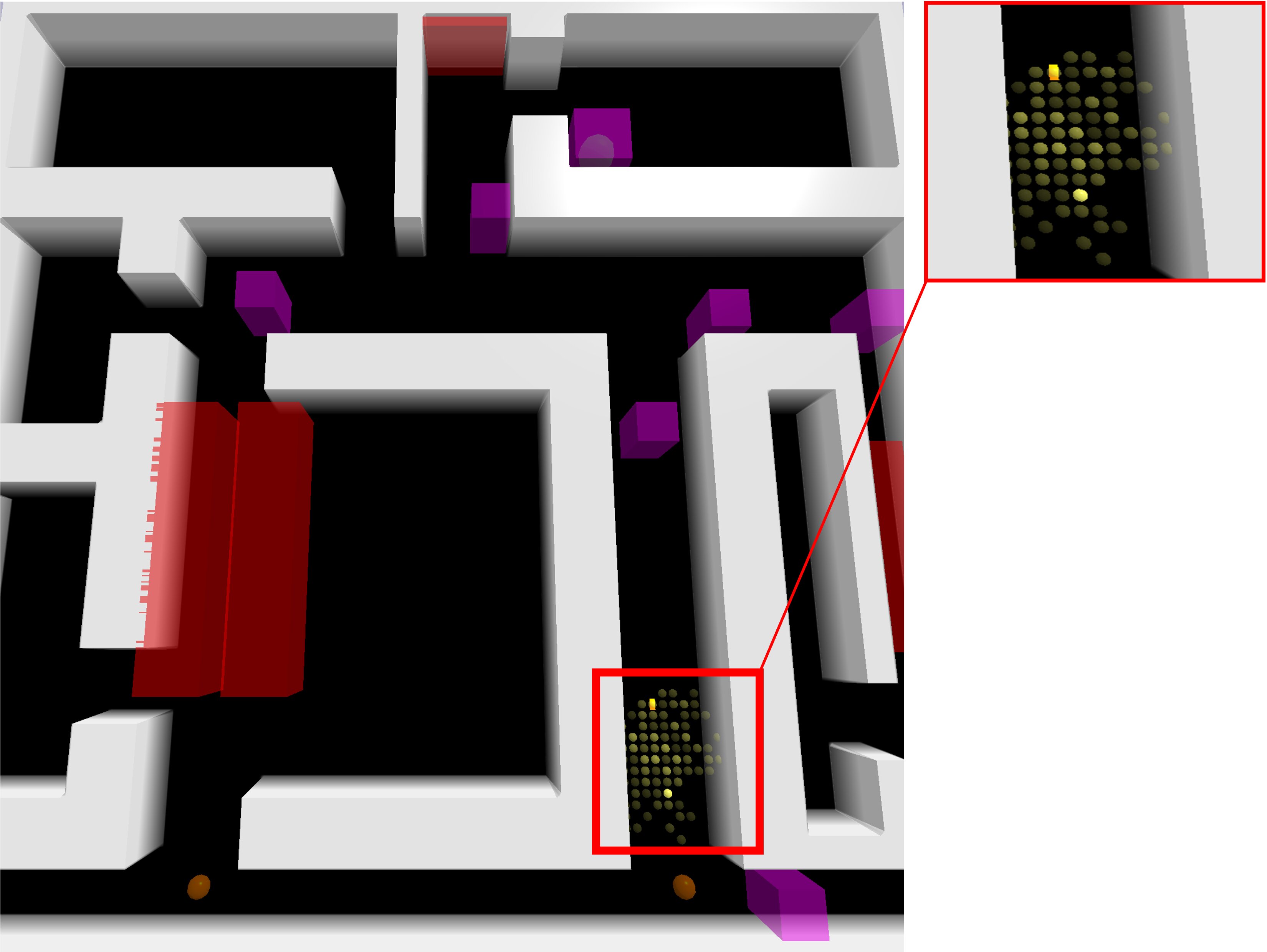}
    \end{subfigure}
    \begin{subfigure}[b]{0.45\textwidth}
        \centering
        \includegraphics[width=\textwidth]{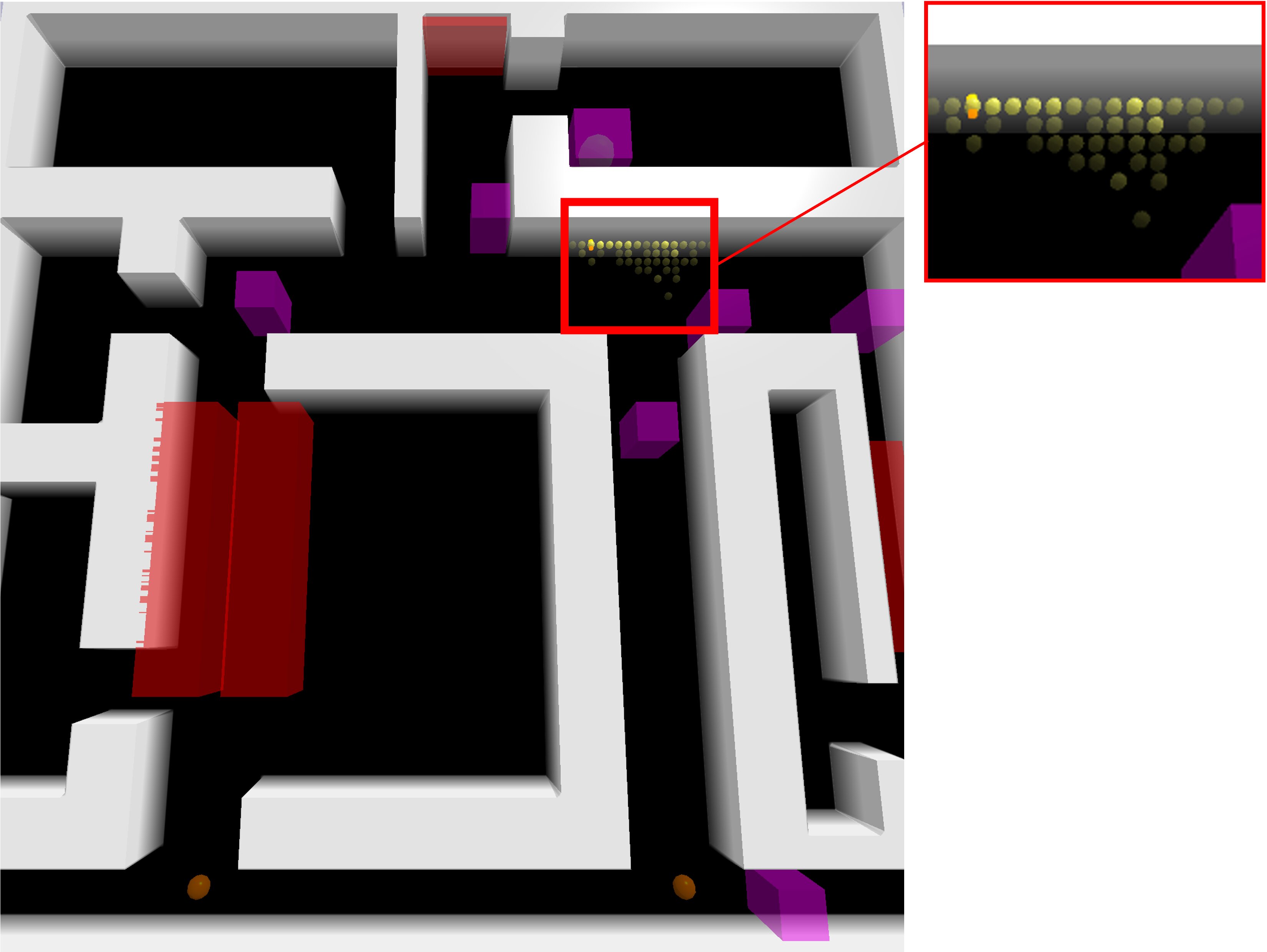}
    \end{subfigure}
    \begin{subfigure}[b]{0.45\textwidth}
        \centering
        \includegraphics[width=\textwidth]{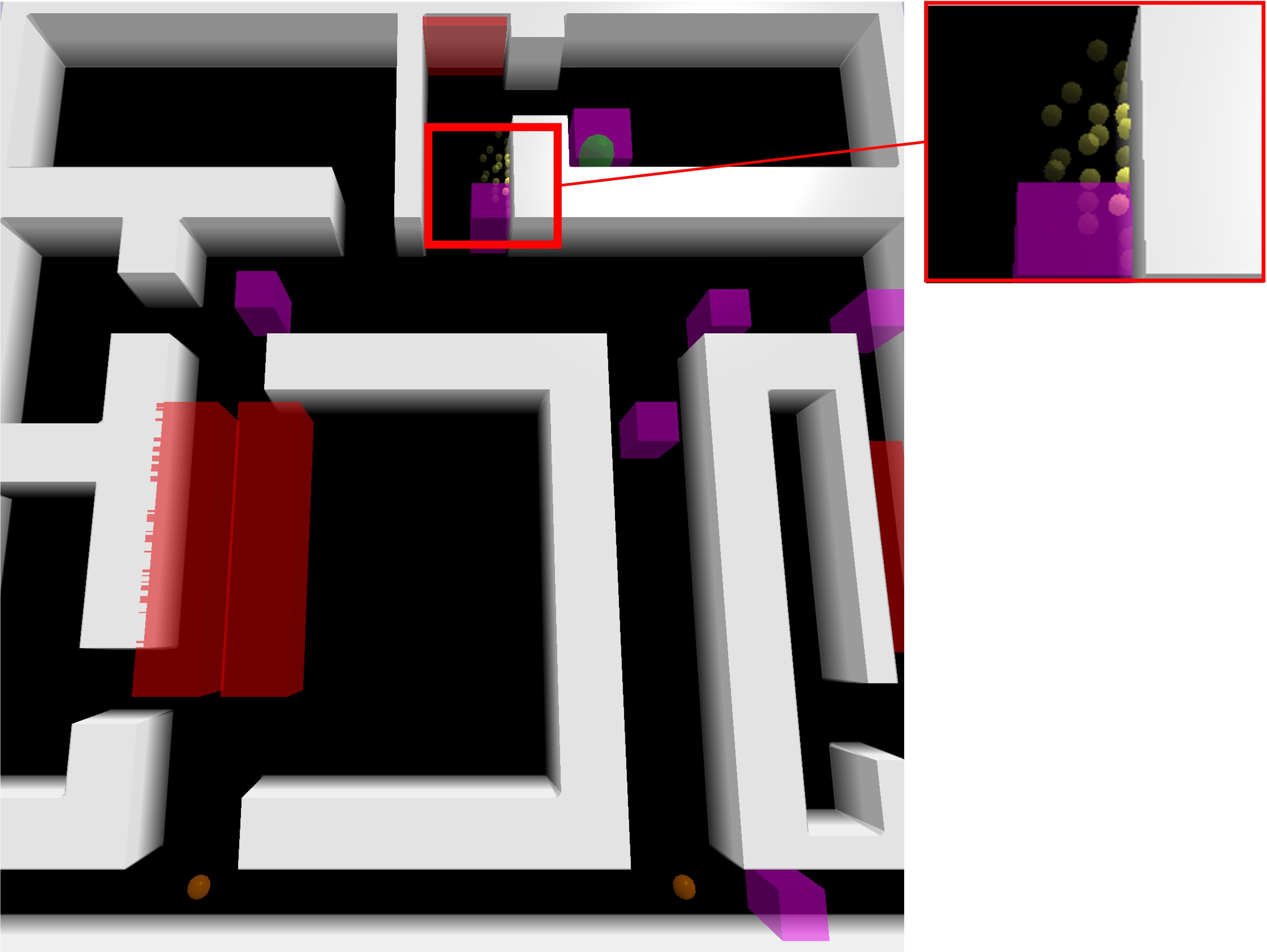}
    \end{subfigure}
    \caption{In Maze, the agent avoids the corridors with danger zones and starts going towards right first to localize. As it navigates to the middle corridor, it needs to carefully pick landmarks to receive observations of itself. Sometimes, the agent uses  walls to align all the beliefs on one side as a localization mechanism, and eventually managed to reach the goal.}
    \label{fig: ROPRAS Maze Behavior Visualisations}
\end{figure*}

\begin{figure*}
    \centering
    \begin{subfigure}[b]{0.45\textwidth}
        \centering
        \includegraphics[width=\textwidth]{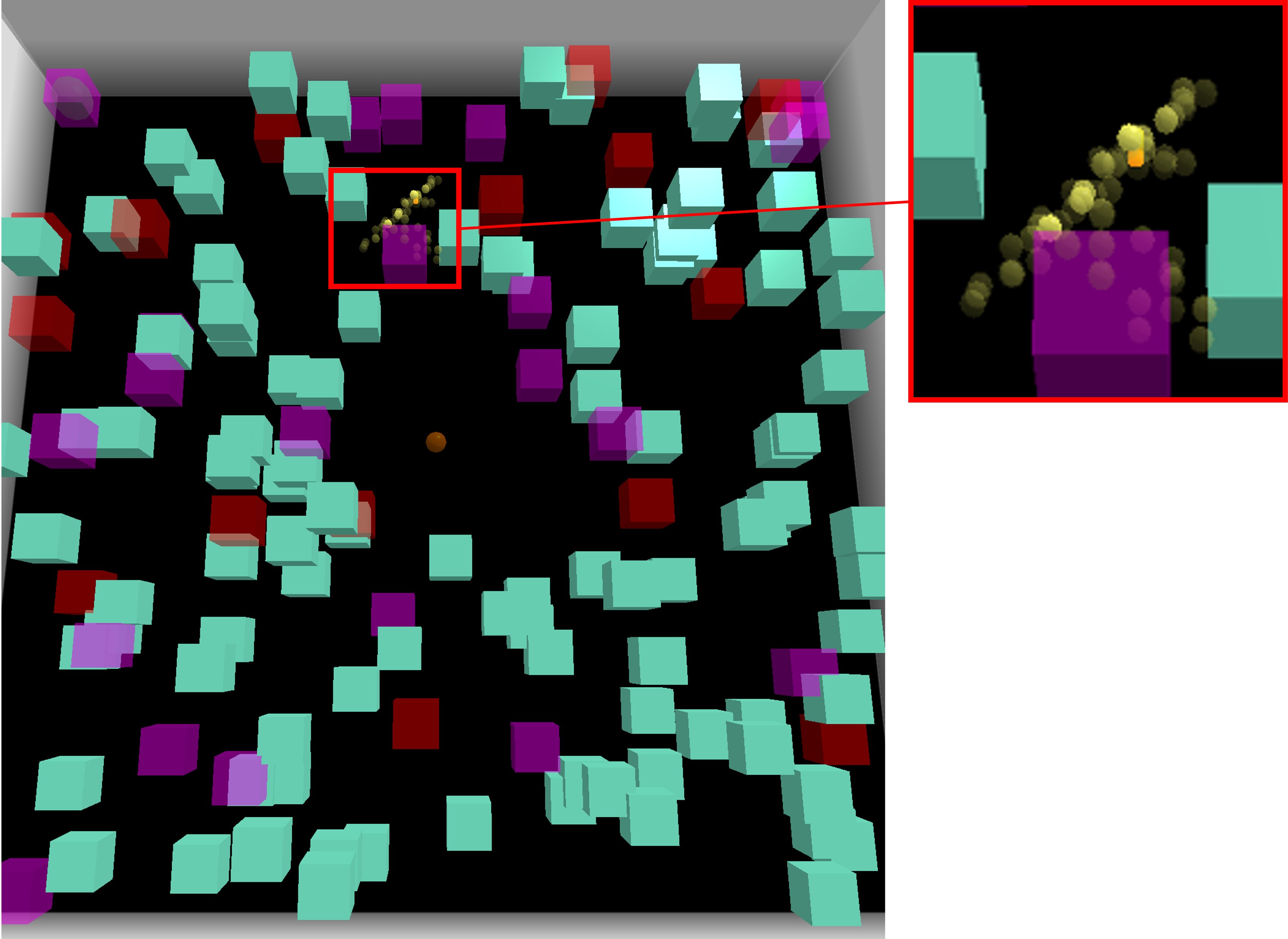}
    \end{subfigure}
    \begin{subfigure}[b]{0.45\textwidth}
        \centering
        \includegraphics[width=\textwidth]{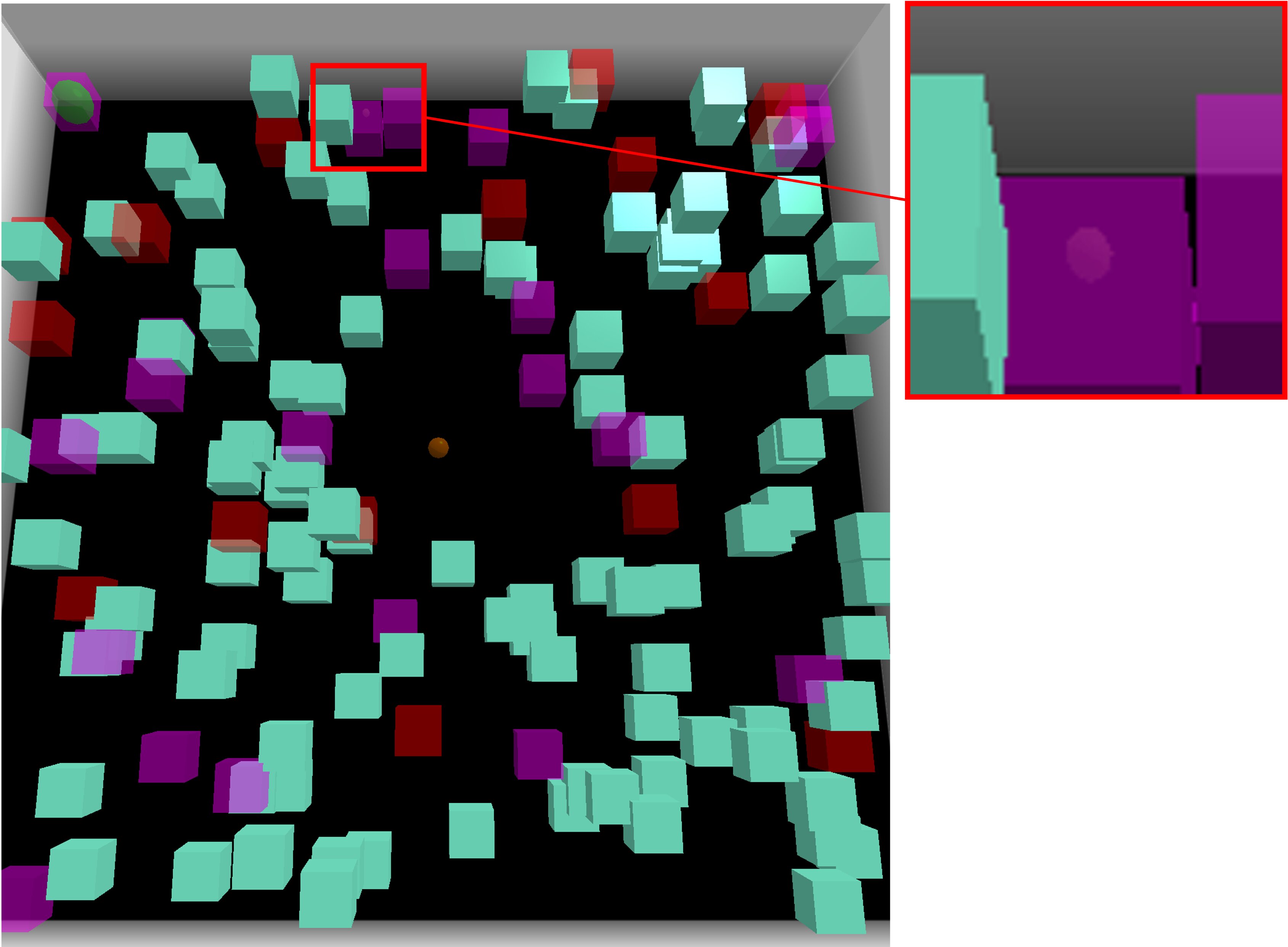}
    \end{subfigure}
    \begin{subfigure}[b]{0.45\textwidth}
        \centering
        \includegraphics[width=\textwidth]{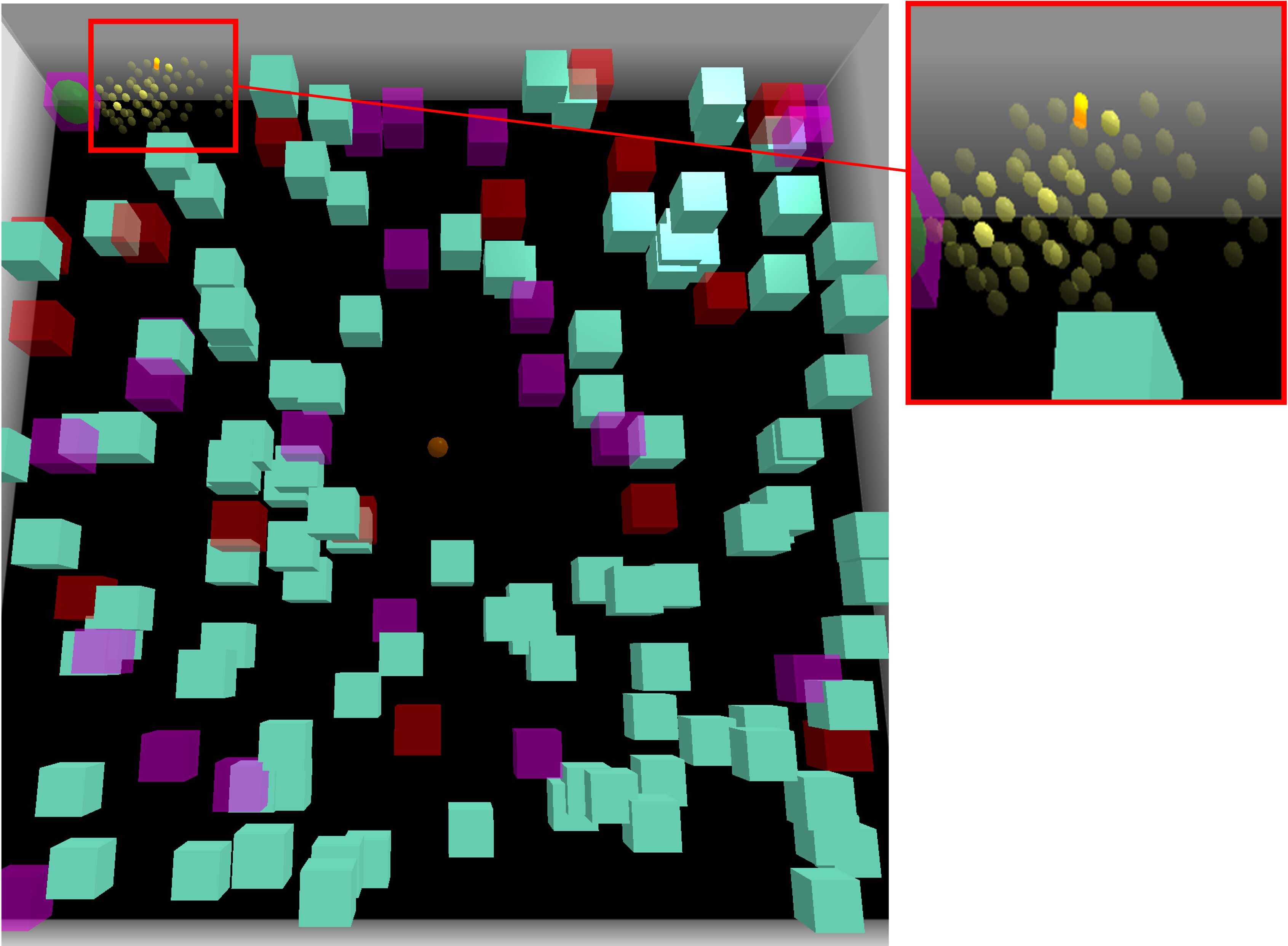}
    \end{subfigure}
    \begin{subfigure}[b]{0.45\textwidth}
        \centering
        \includegraphics[width=\textwidth]{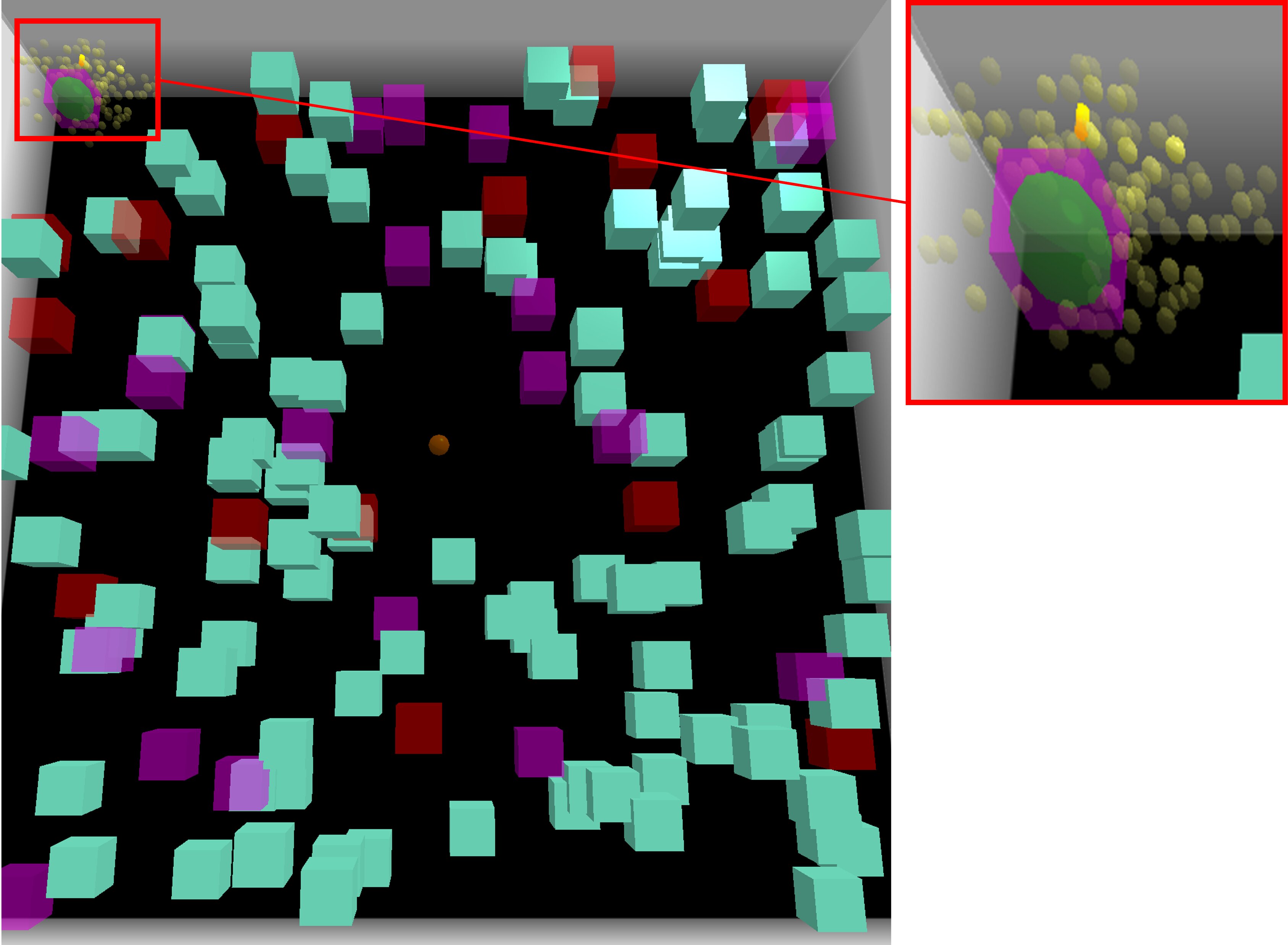}
    \end{subfigure}
    \caption{In Random3D, the agent firstly went to the closest landmark on the way to the goal. It then sticks to the wall as it's a robust path, belief particles can all aligned to the wall, avoiding the danger zones and efficiently navigate to the goal.}
    \label{fig: ROPRAS Random3D Behavior Visualisations}
\end{figure*}

\begin{figure*}
\centering
    \begin{subfigure}[b]{0.24\textwidth}
        \centering
        \includegraphics[width=\textwidth]{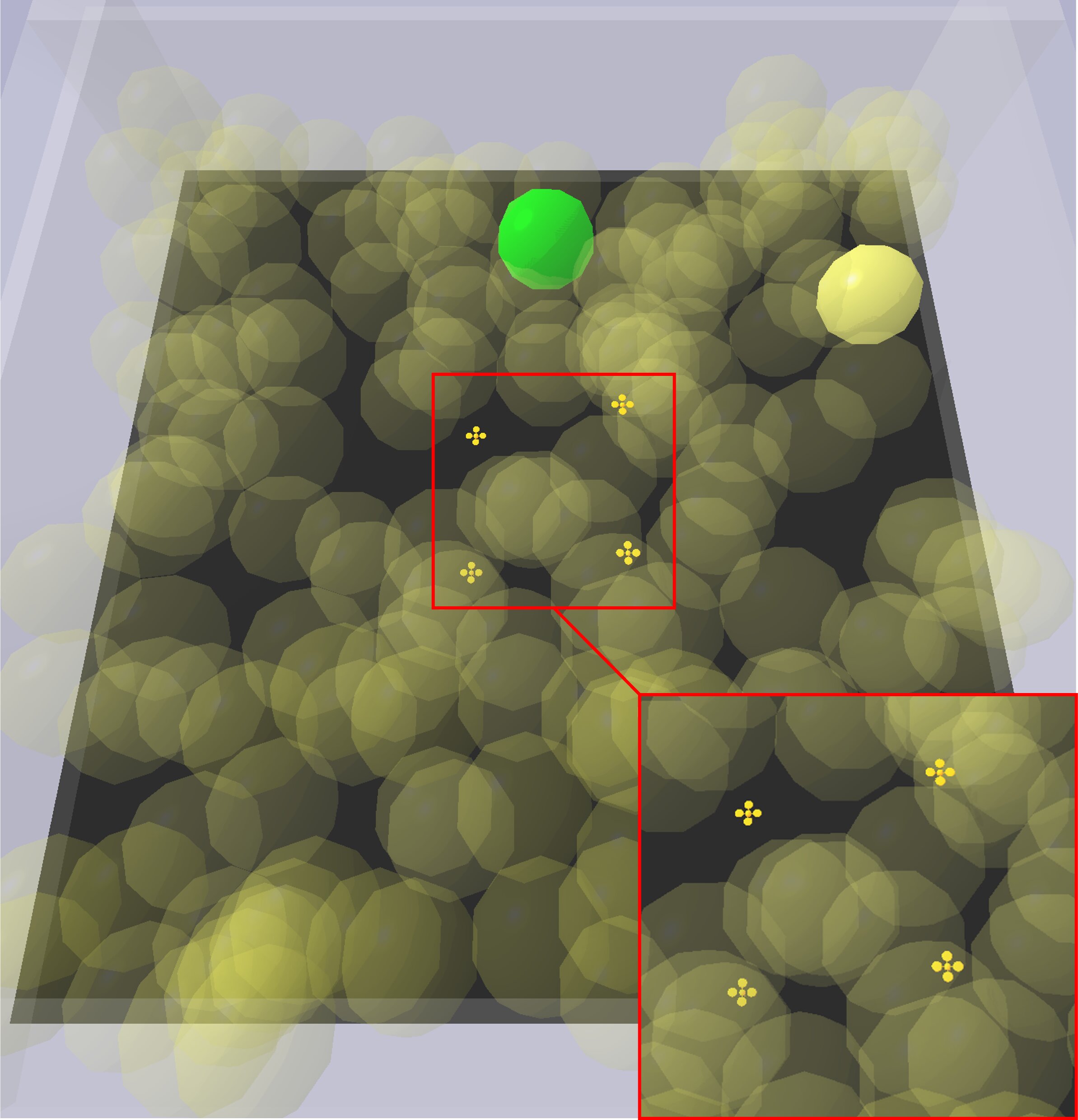}
    \end{subfigure}
    \begin{subfigure}[b]{0.24\textwidth}
        \centering
        \includegraphics[width=\textwidth]{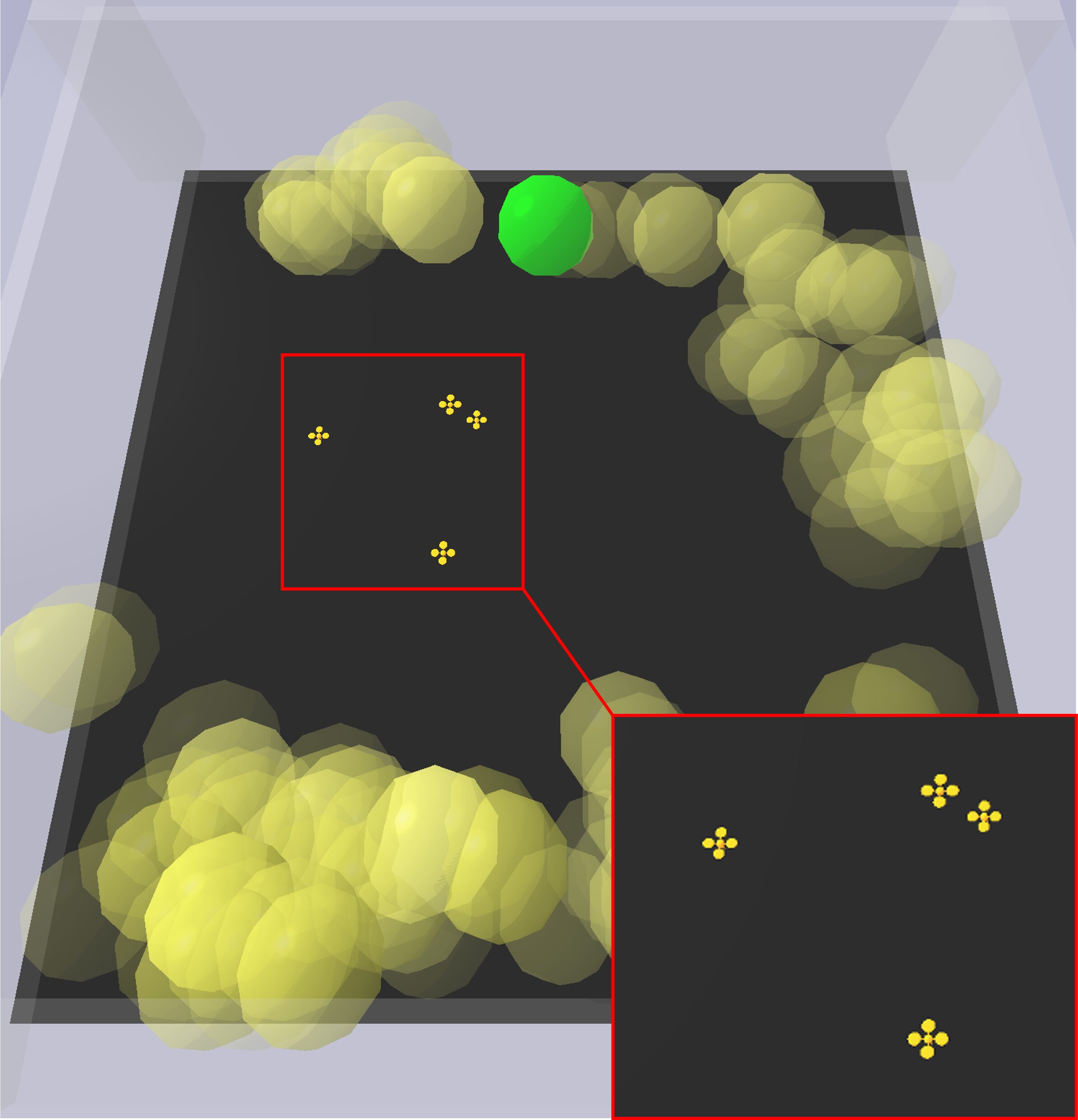}
    \end{subfigure}
    \begin{subfigure}[b]{0.24\textwidth}
        \centering
        \includegraphics[width=\textwidth]{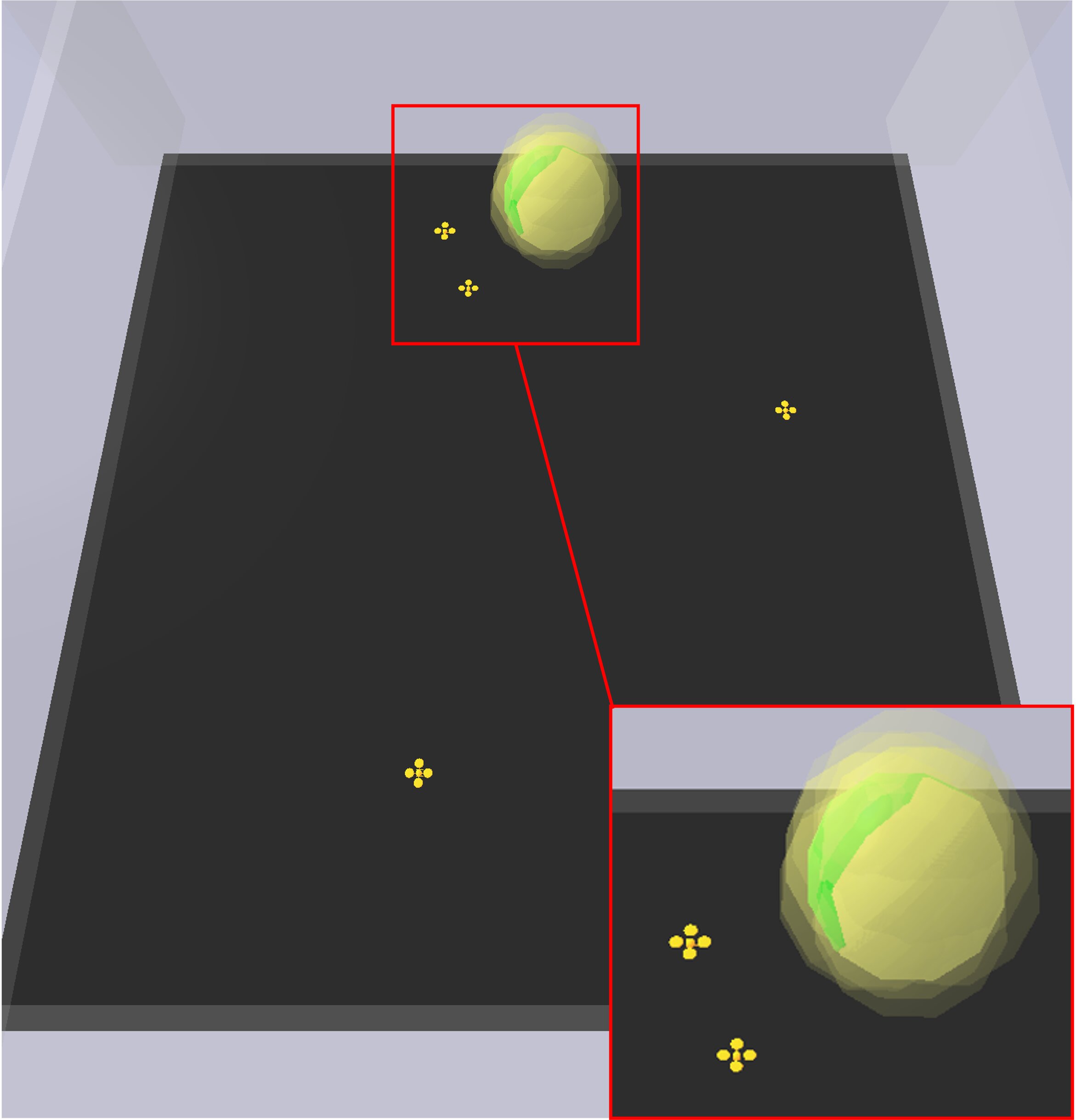}
    \end{subfigure}
    \begin{subfigure}[b]{0.24\textwidth}
        \centering
        \includegraphics[width=\textwidth]{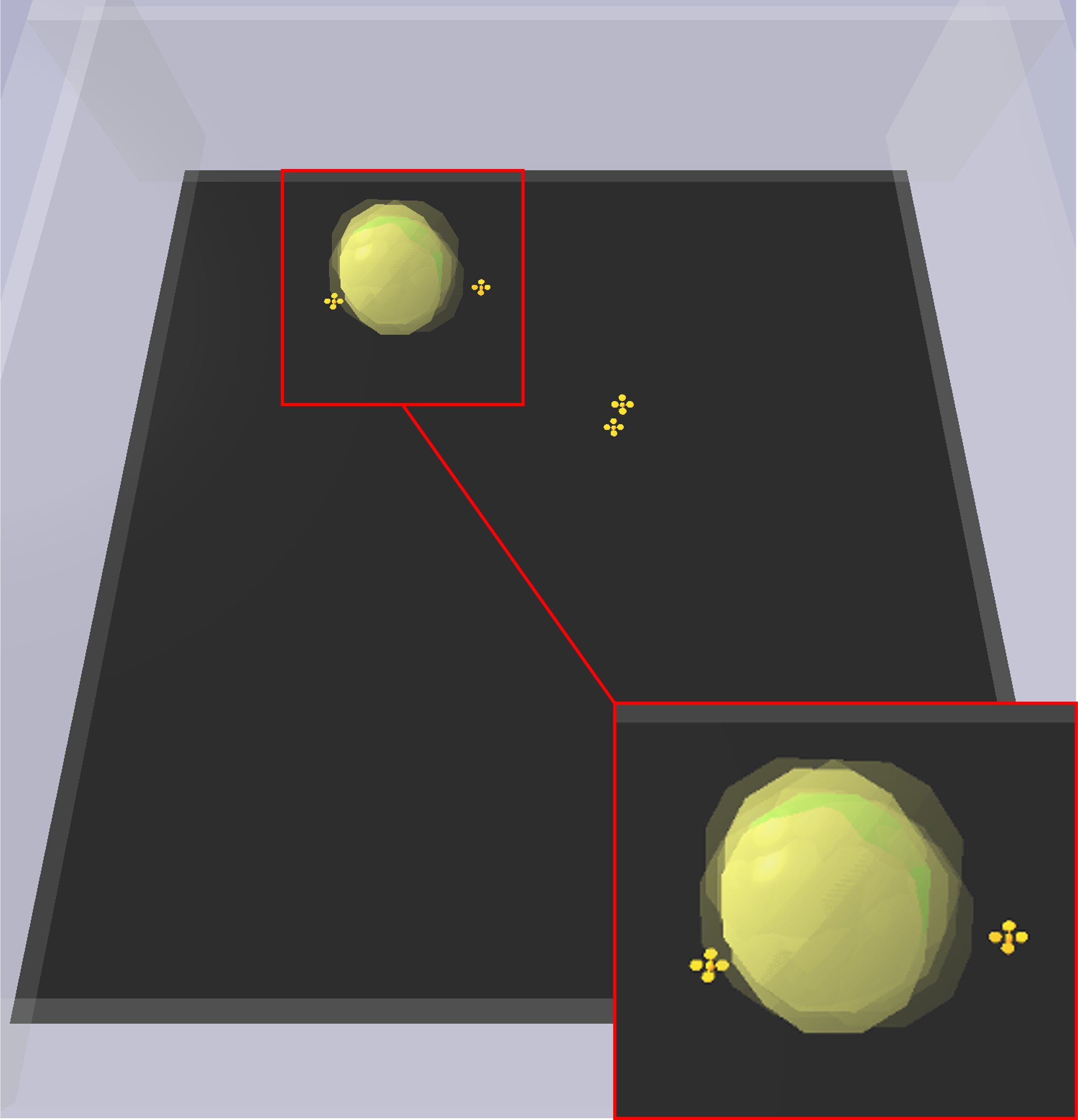}
    \end{subfigure}
    \caption{In Multi-Drone, \nop commands the drones to spread out to detect the target. Then one drone aims to chase a downward moving target to force it to teleport to the other side where another agent has been waiting to capture the target.}
    \label{fig: ROPRAS Multi-Drone Behavior Visualisations}
\end{figure*}

\begin{figure*}
     \centering
     \begin{subfigure}[b]{0.24\textwidth}
         \centering
         \includegraphics[width=\textwidth]{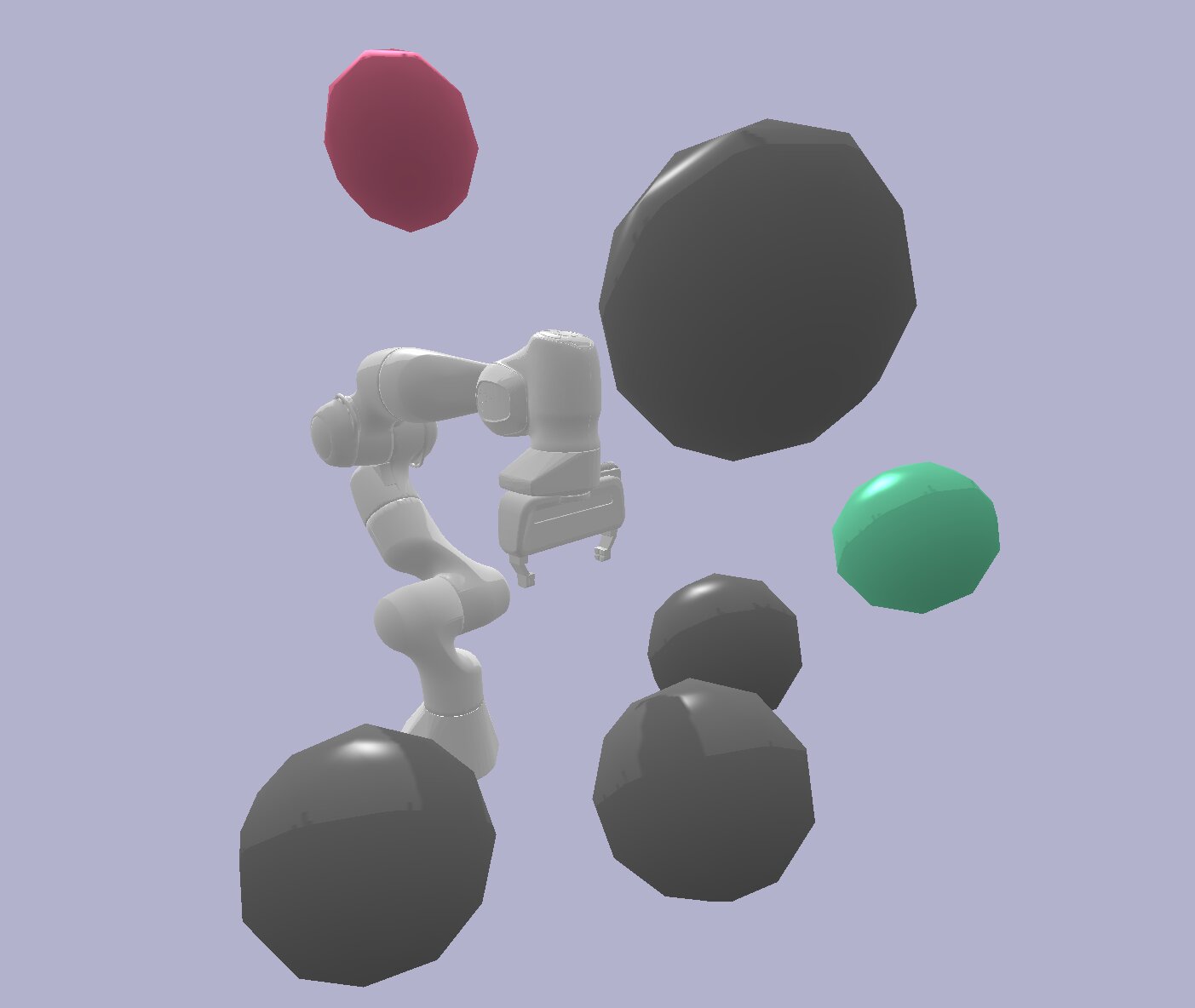}
     \end{subfigure}
     \begin{subfigure}[b]{0.24\textwidth}
         \centering
         \includegraphics[width=\textwidth]{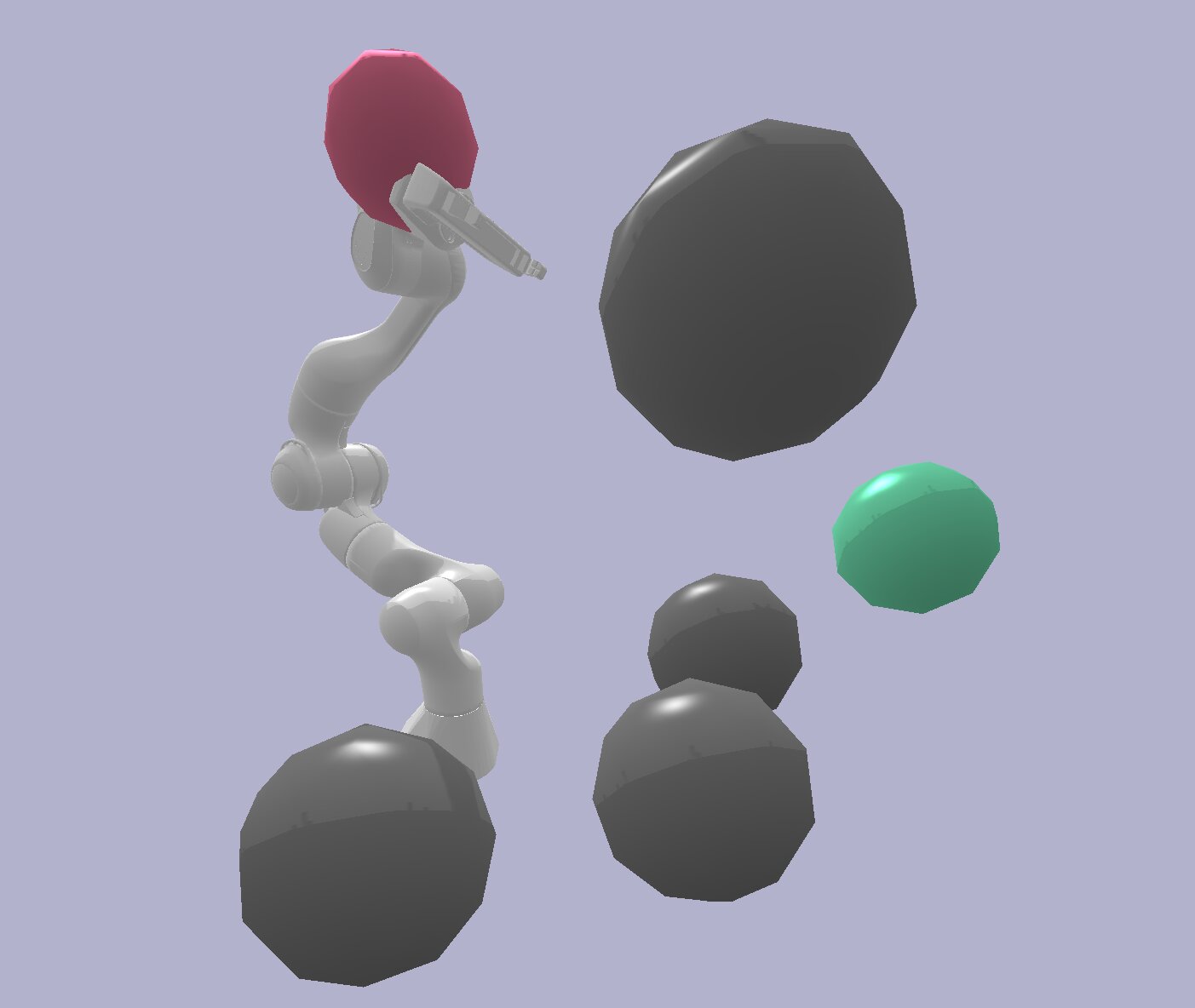}
     \end{subfigure}
     \begin{subfigure}[b]{0.24\textwidth}
         \centering
         \includegraphics[width=\textwidth]{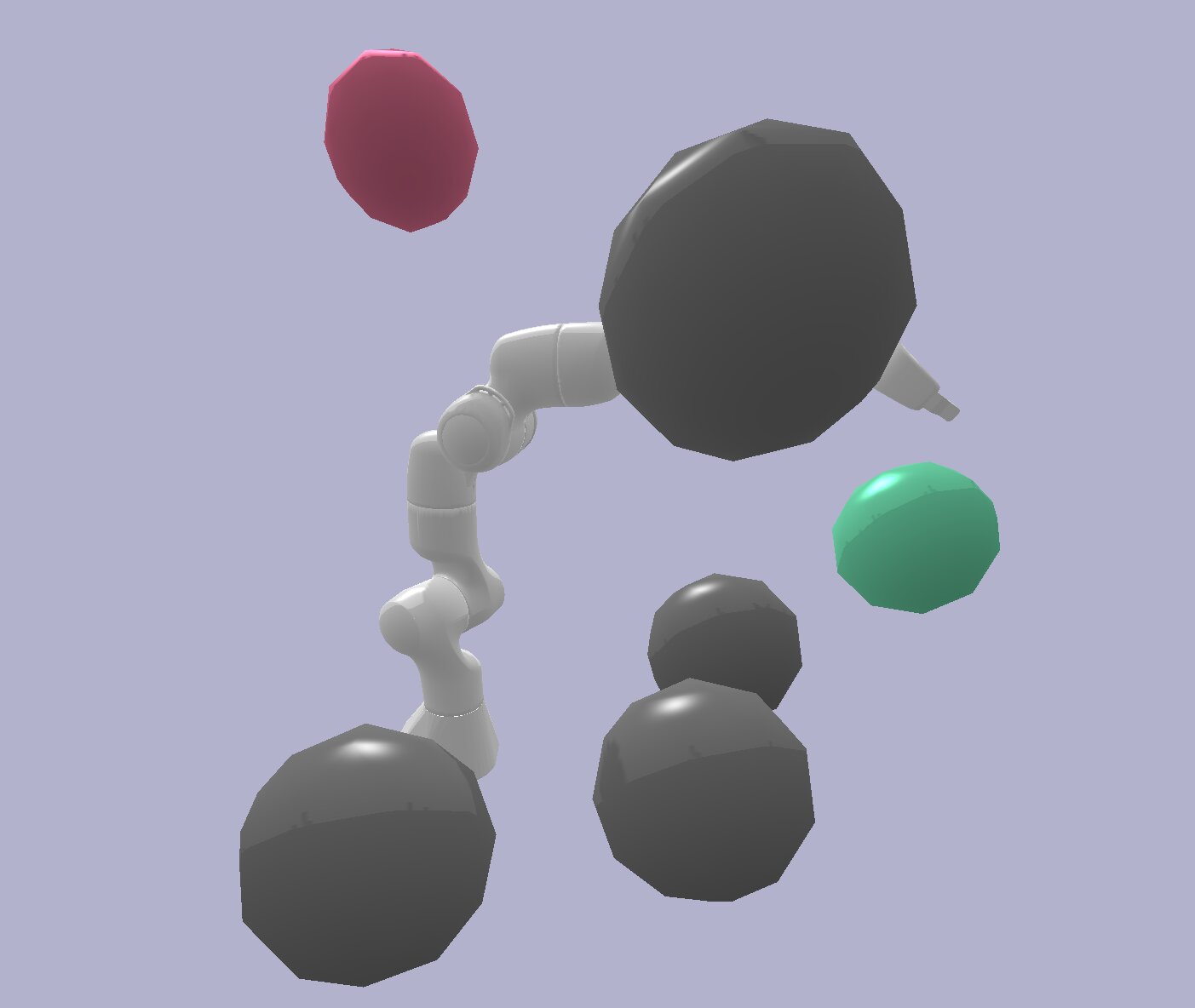}
     \end{subfigure}
     \begin{subfigure}[b]{0.24\textwidth}
         \centering
         \includegraphics[width=\textwidth]{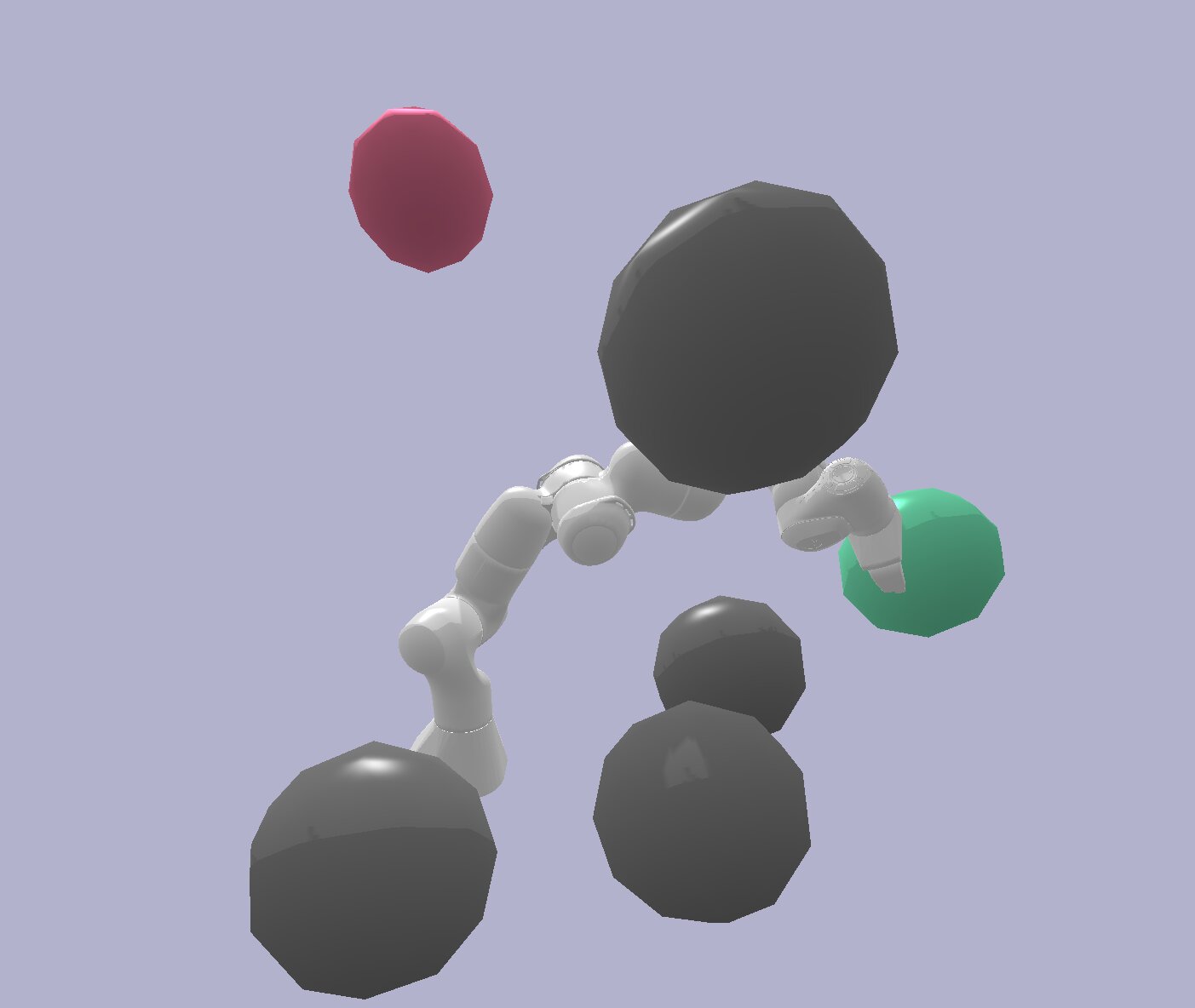}
     \end{subfigure}
     
     \begin{subfigure}[b]{0.24\textwidth}
         \centering
         \includegraphics[width=\textwidth]{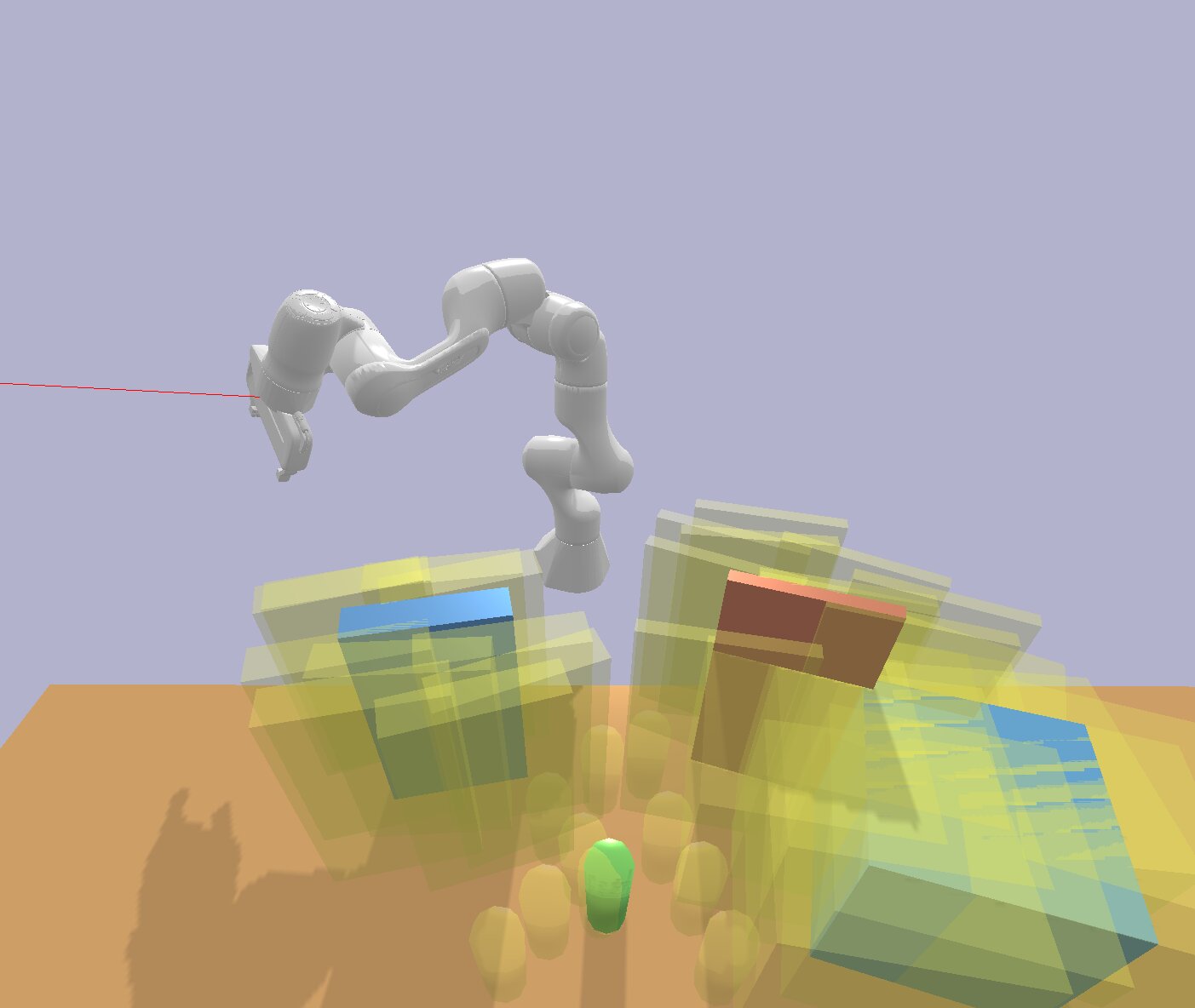}
     \end{subfigure}
     \begin{subfigure}[b]{0.24\textwidth}
         \centering
         \includegraphics[width=\textwidth]{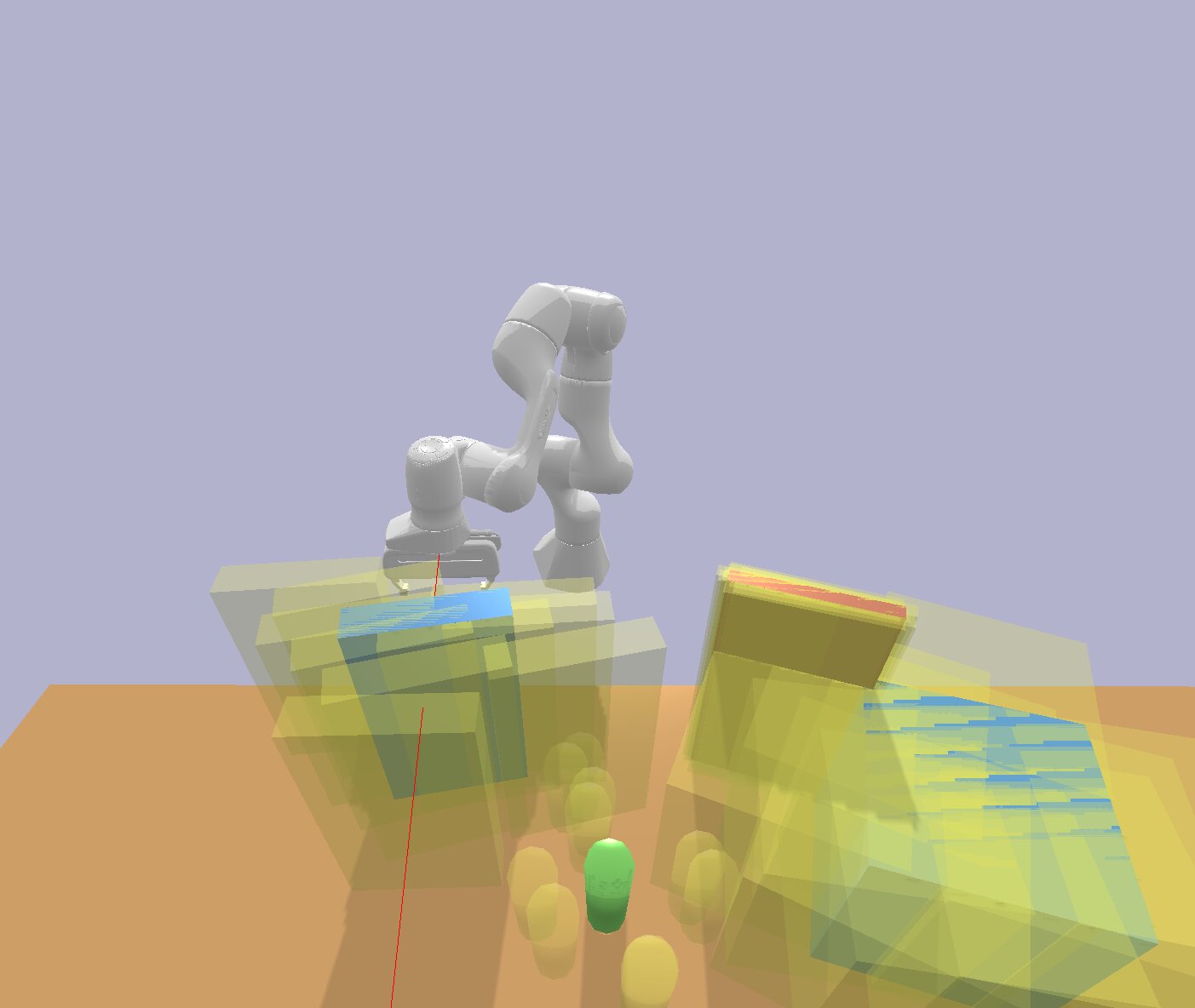}
     \end{subfigure}
     \begin{subfigure}[b]{0.24\textwidth}
         \centering
         \includegraphics[width=\textwidth]{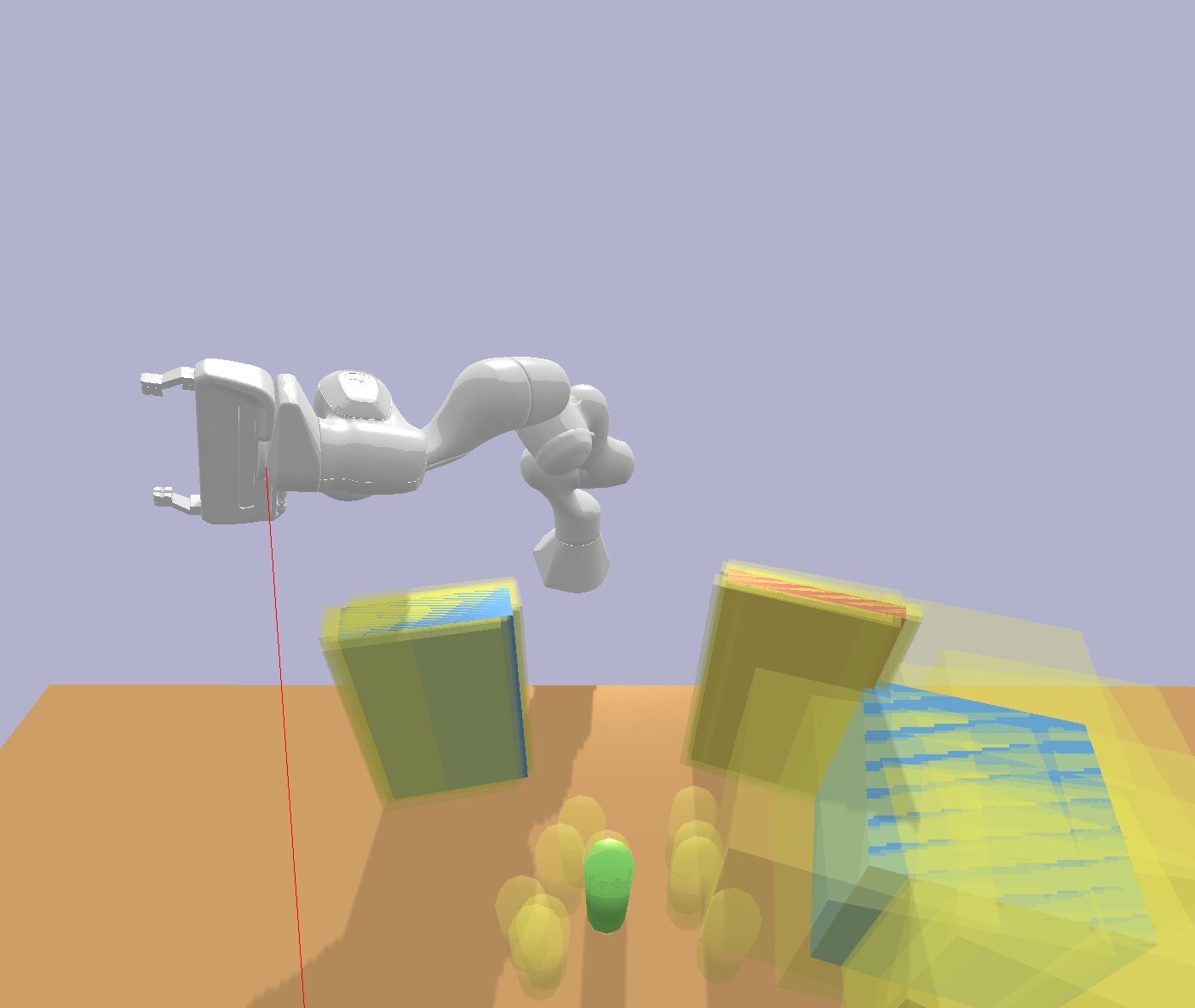}
     \end{subfigure}
     \begin{subfigure}[b]{0.24\textwidth}
         \centering
         \includegraphics[width=\textwidth]{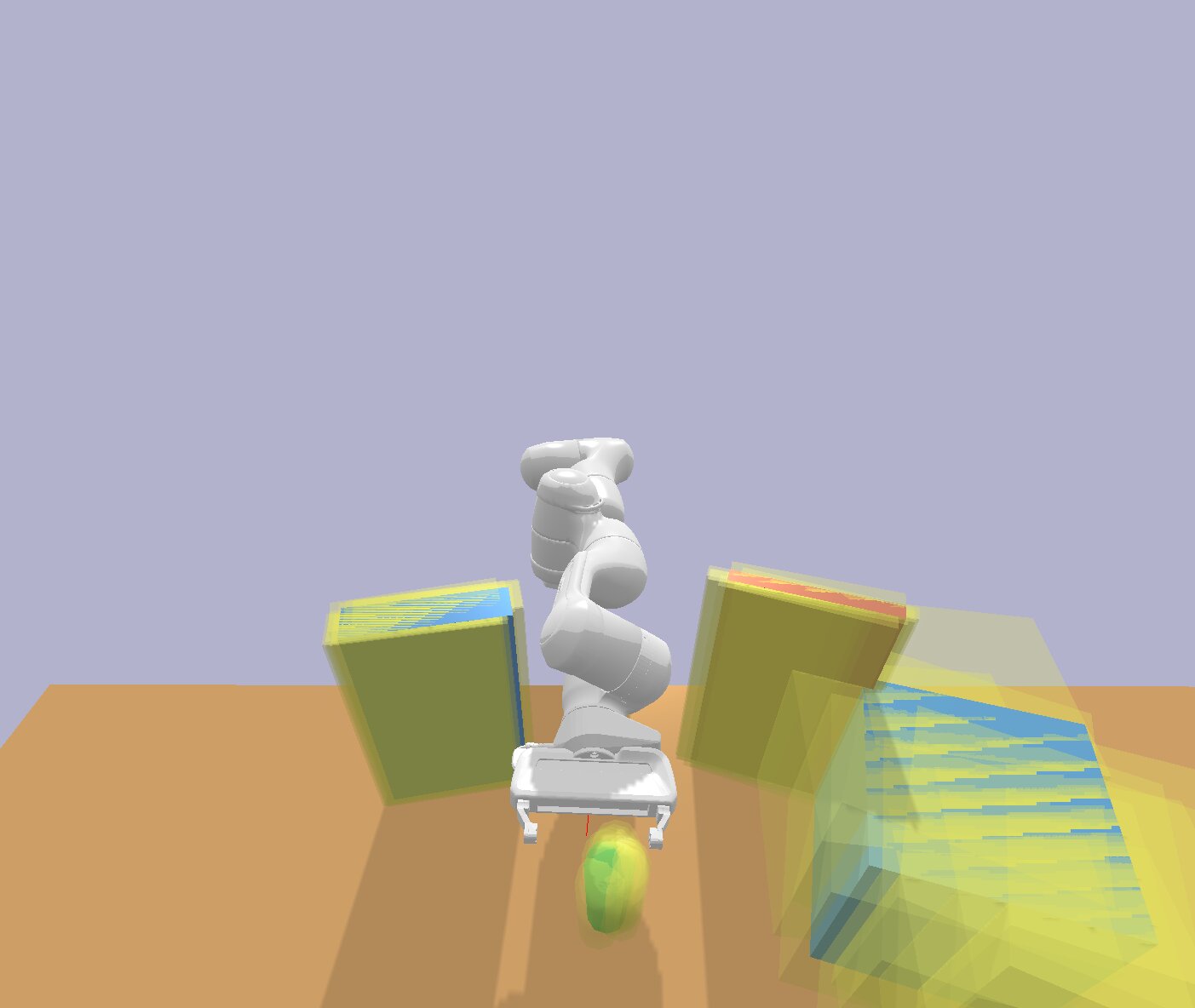}
     \end{subfigure}

     \begin{subfigure}[b]{0.24\textwidth}
         \centering
         \includegraphics[width=\textwidth]{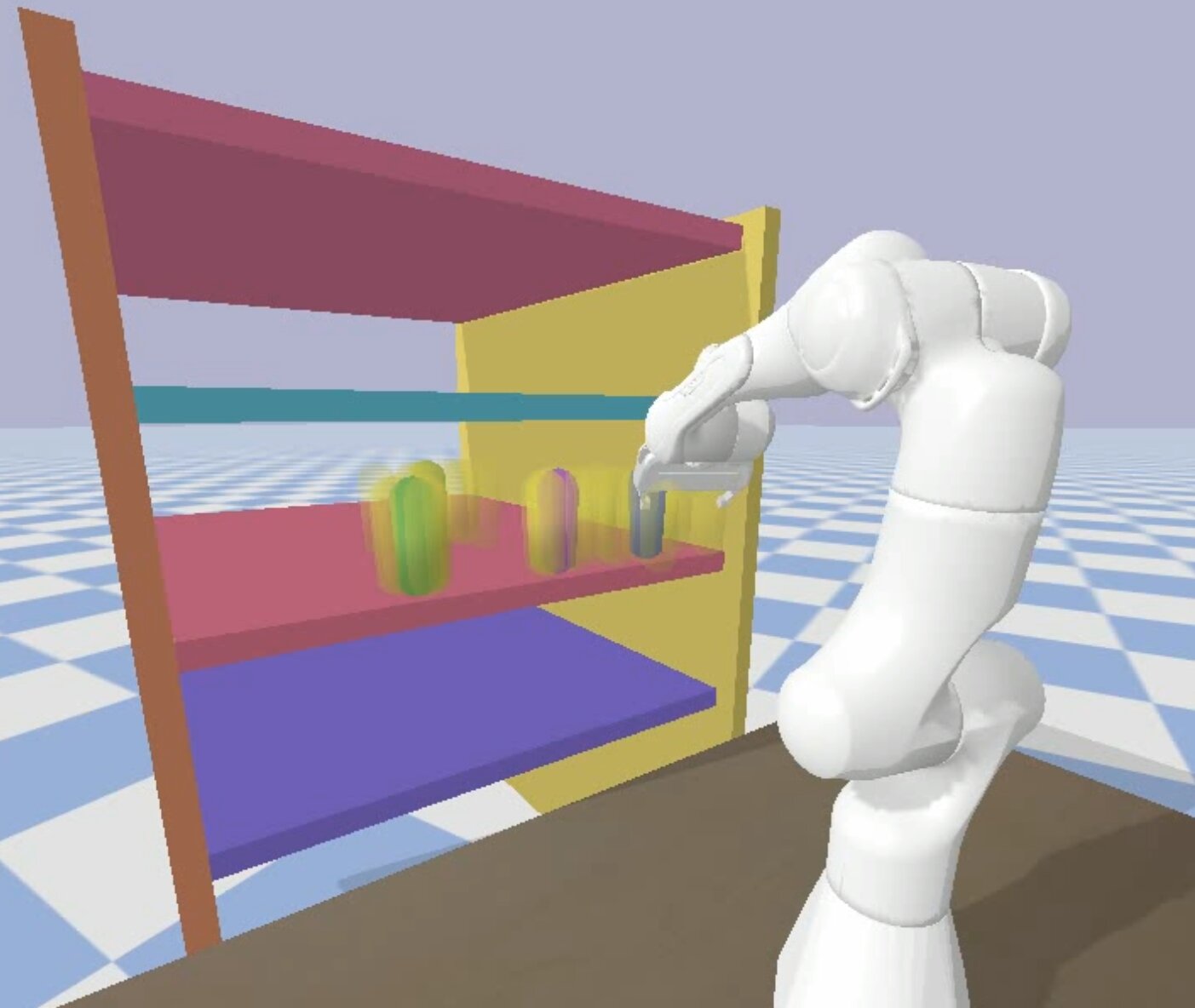}
     \end{subfigure}
     \begin{subfigure}[b]{0.24\textwidth}
         \centering
         \includegraphics[width=\textwidth]{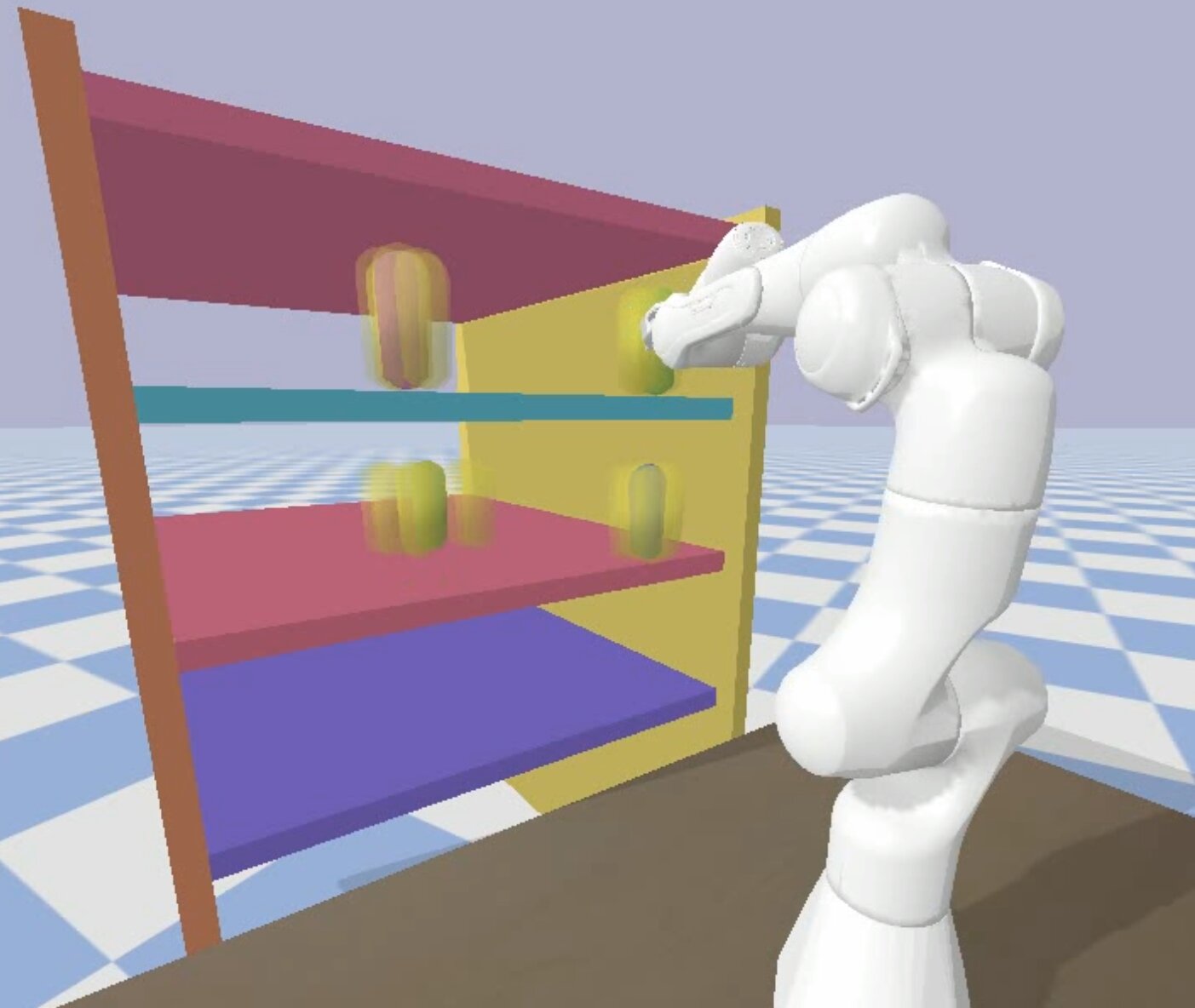}
     \end{subfigure}
     \begin{subfigure}[b]{0.24\textwidth}
         \centering
         \includegraphics[width=\textwidth]{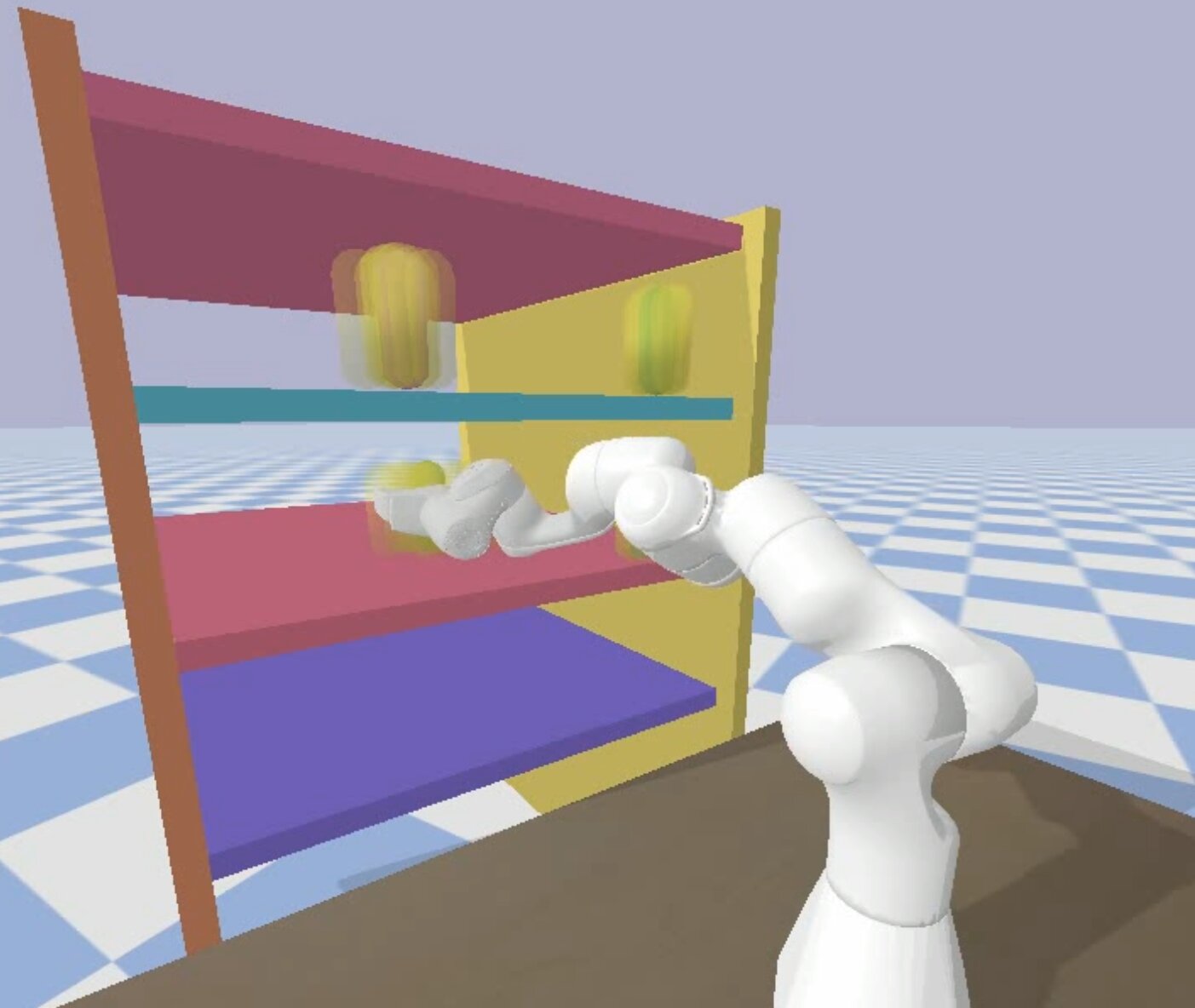}
     \end{subfigure}
     \begin{subfigure}[b]{0.24\textwidth}
         \centering
         \includegraphics[width=\textwidth]{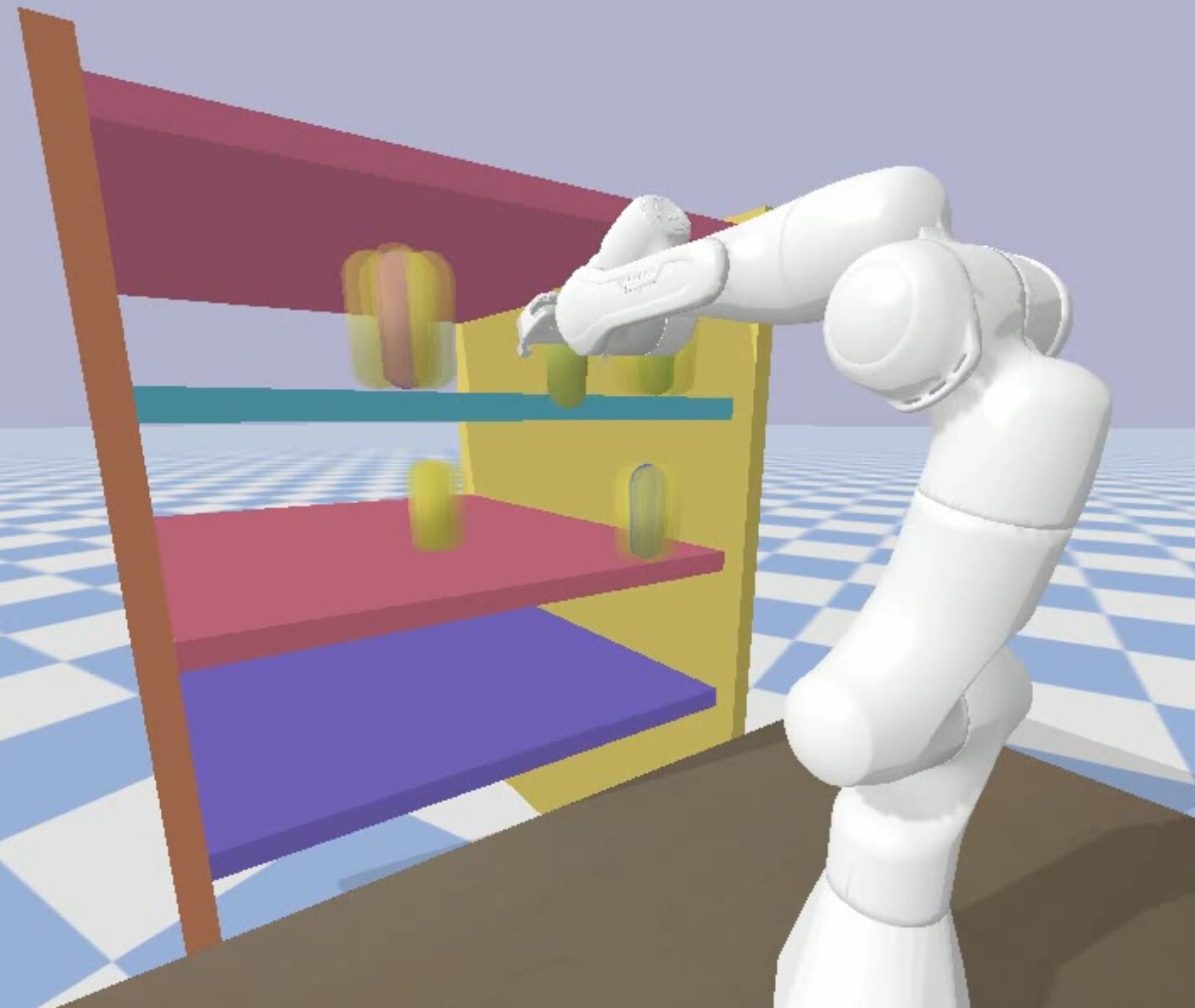}
     \end{subfigure}
    
    \caption{Behaviors of \nop across all manipulation tasks. In Sphere Search (1st row), the manipulator firstly reaches the light ({\color{purple}{purple}}) to receive an observation on where the sphere ({\color{green} green}) is spawned. Then it navigates directly towards the sphere, avoiding the obstacles ({\color{gray} gray}). In Ray Detect (2nd row), the manipulator spends extra steps to detect two important obstacles colored in {\color{blue} blue} and {\color{brown} brown}. As the uncertainties are reduced (sampled particles are displayed in {\color{yellow} yellow}), the manipulator is able to navigate around the obstacles to reach the target. In Shelf Move (3rd row), the manipulator firstly moves its end effectors around to sense the world, then it place two important cylinders away to the top shelf, leaving the middle position at the top shelf empty for the target. It then grasps the target and placed it at the intended location.}
    \label{fig: ROPRAS Manipulation Behavior Visualisations}
\end{figure*}

\begin{figure*}
    \centering
    \begin{subfigure}[b]{0.24\textwidth}
         \centering
         \includegraphics[width=\textwidth]{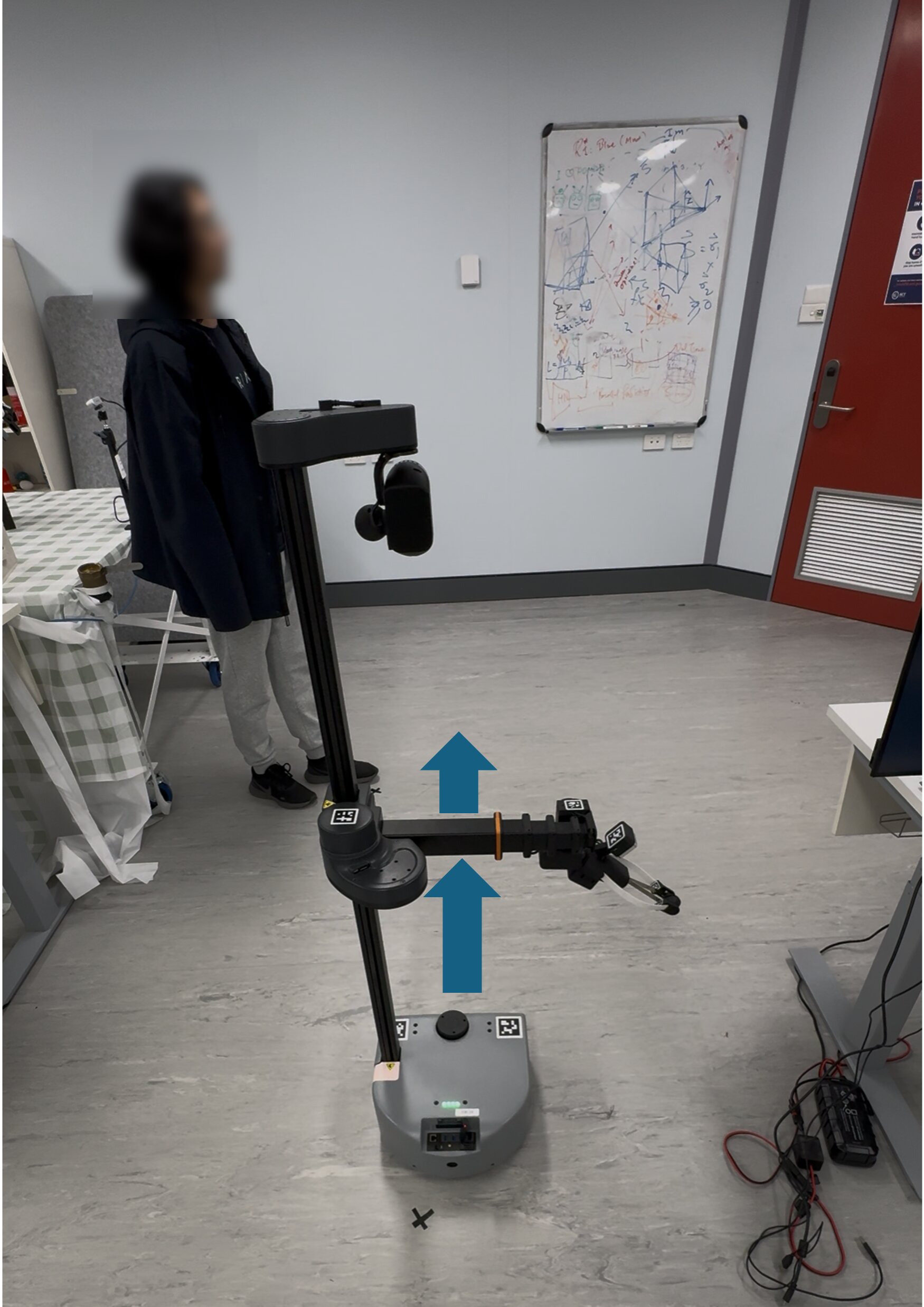}
     \end{subfigure}
     \begin{subfigure}[b]{0.24\textwidth}
         \centering
         \includegraphics[width=\textwidth]{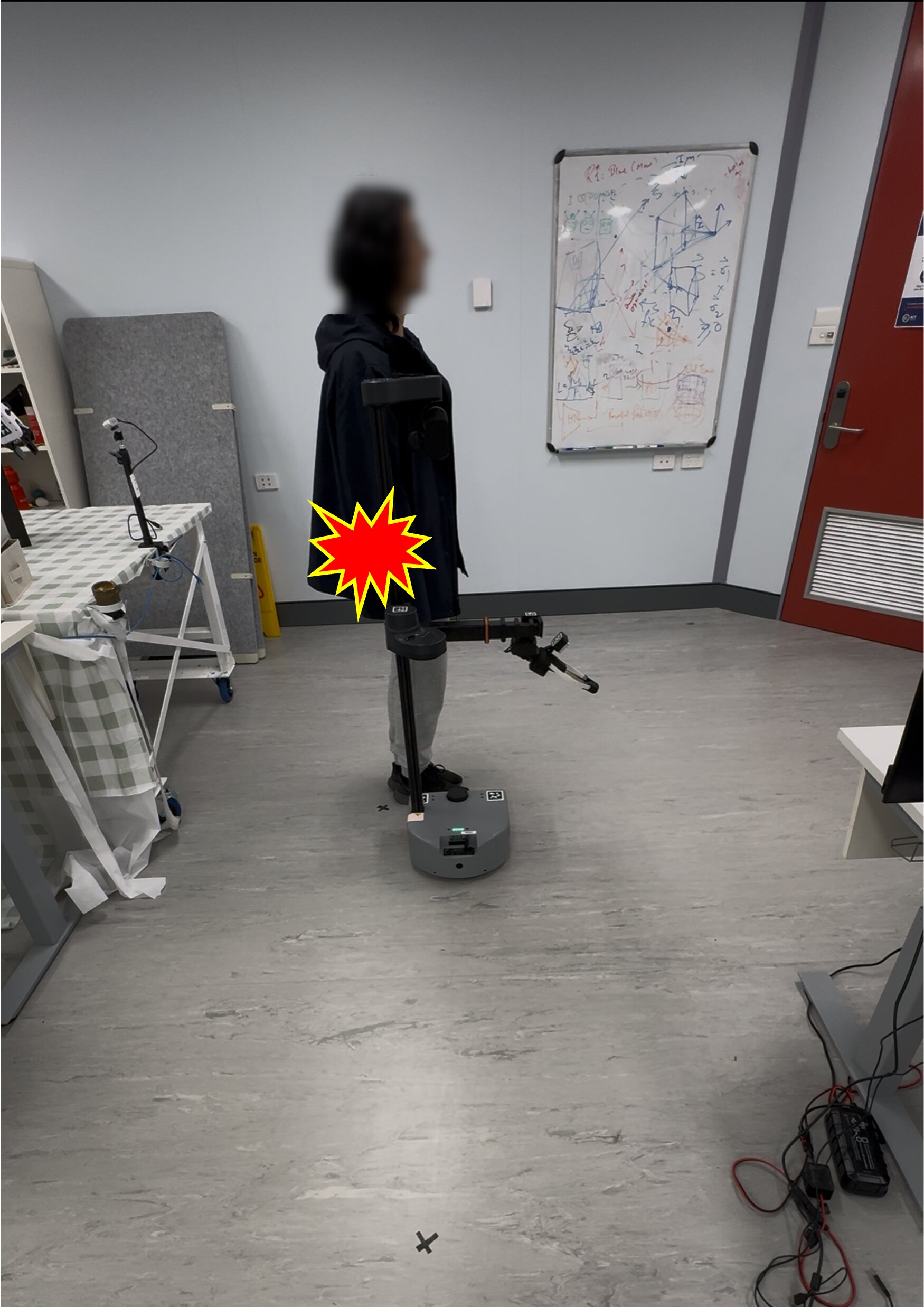}
     \end{subfigure}
     \begin{subfigure}[b]{0.24\textwidth}
         \centering
         \includegraphics[width=\textwidth]{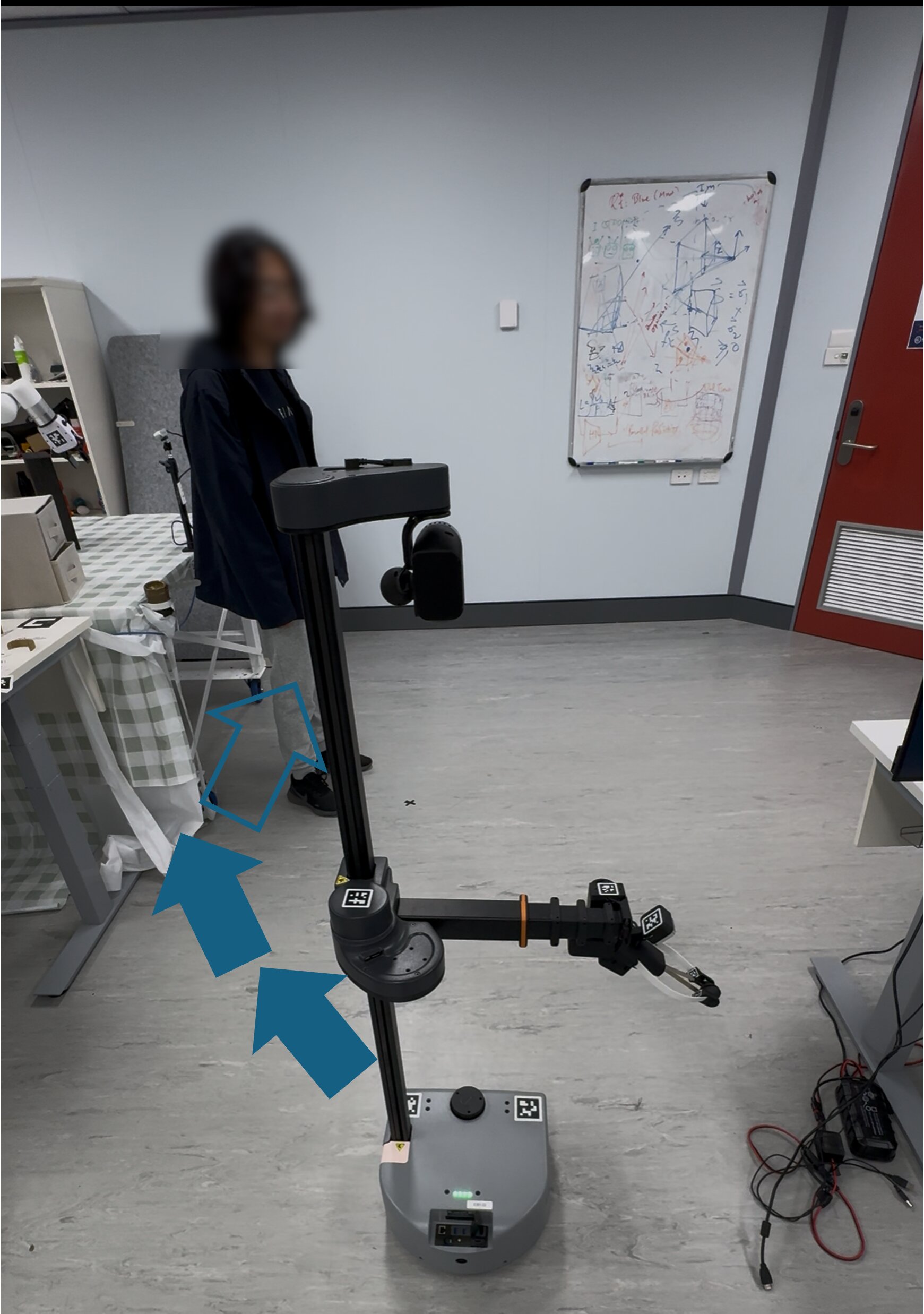}
     \end{subfigure}
     \begin{subfigure}[b]{0.24\textwidth}
         \centering
         \includegraphics[width=\textwidth]{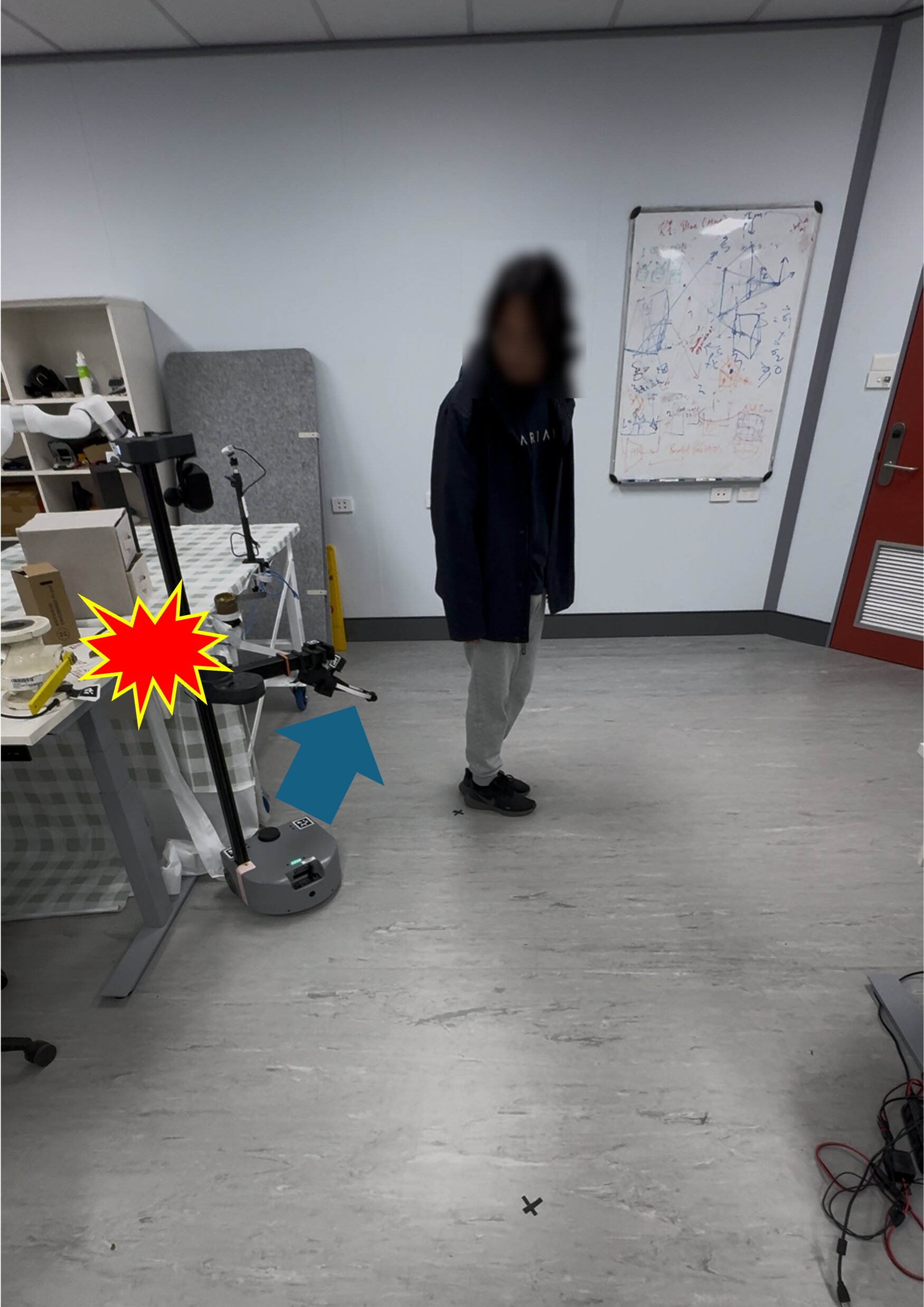}
     \end{subfigure}
     \caption{Stretch demonstrations for B-\vamp (first two pictures) and R-\pomcp (last two pictures). B-\vamp moves forward and collide with the moving pedestrian due to lack of \pomdp planning and R-\pomcp tries to avoid the person but went too far and collided with the environment.}
     \label{fig: stretch-bvamp-rpomcp}
\end{figure*}